\documentclass[twoside,11pt]{article}
\usepackage{jmlr2e}
\pdfoutput=1
\RequirePackage{amsmath,amsfonts,xr}
\RequirePackage[OT1]{fontenc}
\RequirePackage{hypernat}
\RequirePackage{hyperref}
\RequirePackage{enumerate}
\RequirePackage{bbold}
\RequirePackage{algorithmic}
\RequirePackage[section]{algorithm}
\RequirePackage{subfigure}
\usepackage{multirow}

\def\ci{\perp\hspace{-6pt}\perp}

\usepackage{tikz}
\usetikzlibrary{arrows}
\tikzstyle{block}=[draw opacity=0.7,line width=1.4cm]

\usepackage[all]{xy}

\ShortHeadings{Order-independent causal structure learning}{Colombo and Maathuis}
\firstpageno{1}

\begin{document}

\title{Order-independent constraint-based causal structure learning}

\author{\name Diego Colombo \email colombo@stat.math.ethz.ch \\
       \name Marloes H. Maathuis \email maathuis@stat.math.ethz.ch \\
       \addr Seminar for Statistics,\\
       ETH Zurich,\\
       8092 Zurich, Switzerland}

\editor{}

\maketitle

\begin{abstract}
  We consider constraint-based methods for causal structure learning, such
  as the PC-, \mbox{FCI-,} RFCI- and CCD- algorithms
  (\citet{SpirtesEtAl93,SpirtesEtAl00}, \cite{Richardson96},
  \cite{CoMaKaRi2012}, \cite{ClaassenEtAl13}). The first step of all these
  algorithms consists of the PC-algorithm. This algorithm is known to be
  order-dependent, in the sense that the output can depend on the order in
  which the variables are given. This order-dependence is a minor issue in
  low-dimensional settings. We show, however, that it can be very
  pronounced in high-dimensional settings, where it can lead to highly
  variable results. We propose several modifications of the PC-algorithm
  (and hence also of the other algorithms) that remove part or all of this
  order-dependence. All proposed modifications are consistent in
  high-dimensional settings under the same conditions as their original
  counterparts. We compare the PC-, FCI-, and RFCI-algorithms and their
  modifications in simulation studies and on a yeast gene expression data
  set. We show that our modifications yield similar performance in
  low-dimensional settings and improved performance in high-dimensional
  settings. All software is implemented in the R-package
  \texttt{pcalg}.
\end{abstract}

\begin{keywords}
  directed acyclic graph, PC-algorithm, FCI-algorithm, CCD-algorithm,
  order-dependence, consistency, high-dimensional data
\end{keywords}


\section{Introduction}\label{sec.introduction}

Constraint-based methods for causal structure learning use conditional
independence tests to obtain information about the underlying causal
structure. We start by discussing several prominent examples of such
algorithms, designed for different settings.


The PC-algorithm (\cite{SpirtesEtAl93,SpirtesEtAl00}) was designed for learning directed \emph{acyclic} graphs (DAGs) under the assumption of \emph{causal sufficiency}, i.e., no unmeasured common
causes and no selection variables. It learns a Markov equivalence class of DAGs that can be uniquely described by  a
so-called completed partially directed acyclic graph (CPDAG) (see Section
\ref{sec.graph.def} for a precise definition).
The PC-algorithm is widely
used
in high-dimensional settings (e.g.,
\cite{KalischEtAl10, NagarajanEtAl10, StekhovenEtAll12, ZhangEtAl11}), since it is computationally feasible
for sparse graphs with up to thousands of variables, and open-source
software is available (e.g., \texttt{pcalg} \citep{KalischEtAl12} and
TETRAD IV \citep{SpirtesEtAl00}). Moreover, the PC-algorithm has been shown
to be consistent for high-dimensional sparse graphs
\citep{KalischBuehlmann07a, HarrisDrton12}.

The FCI- and RFCI-algorithms and their modifications
(\cite{SpirtesEtAl93,SpirtesEtAl00,SpirtesMeekRichardson99,Spirtes01-anytime,
  Zhang08-orientation-rules}, \cite{CoMaKaRi2012}, \cite{ClaassenEtAl13})
were designed for learning directed \emph{acyclic} graphs when
\emph{allowing for latent and selection variables}. Thus, these algorithms
learn a Markov equivalence class of DAGs with latent and selection
variables, which can be uniquely represented by a partial ancestral graph
(PAG). These algorithms first employ the PC-algorithm, and then perform
additional conditional independence tests because of the latent variables.

Finally, the CCD-algorithm \citep{Richardson96} was designed for learning
Markov equivalence classes of directed (\emph{not necessarily acyclic})
graphs under the assumption of \emph{causal sufficiency}. Again, the first
step of this algorithm consists of the PC-algorithm.

Hence, all these algorithms share the PC-algorithm as a common first step. We will therefore focus our analysis on this algorithm, since any improvements to the PC-algorithm can be directly carried over to the other algorithms.
When the PC-algorithm is applied to data, it is generally order-dependent,
in the sense that its output depends on the order in which the variables
are given. \cite{DashDruzdzel99} exploit the
order-dependence to obtain candidate graphs
for a score-based approach. \cite{CanoEtAl08} resolve the order-dependence
via a rather involved method based on measuring edge
strengths. \cite{SpirtesEtAl00} (Section 5.4.2.4) propose a method that
removes the ``weakest" edges as early as possible. Overall,
however, the order-dependence of the PC-algorithm has received relatively little
attention in the literature, suggesting that it seems to  be regarded as a minor
issue. We found, however, that the order-dependence can become very
problematic for high-dimensional data, leading to highly variable results
and conclusions for different variable orderings.

In particular, we analyzed a yeast
gene expression data set (\cite{Hughes00}; see Section
\ref{sec.simulation.realdata} for more detailed information) containing
gene expression levels of 5361 genes for 63 wild-type yeast
organisms. First, we considered estimating the skeleton of the CPDAG, that
is, the undirected graph obtained by discarding all arrowheads in the
CPDAG. Figure \ref{fig.levelplot.1} shows the large variability in the
estimated skeletons for 25 random orderings of the
variables. Each estimated skeleton consists of roughly 5000 edges which
can be divided into three groups: about 1500 are highly
stable and occur in all orderings, about 1500 are moderately stable and
occur in at least $50\%$ of the orderings, and about 2000 are unstable
and occur in at most $50\%$ of the orderings. Since the FCI- and CCD-algorithms
employ the PC-algorithm as a first step, their resulting skeletons for these data
are also highly order-dependent.

An important motivation for learning DAGs lies in their
causal interpretation. We therefore also investigated the effect of
different variable orderings on causal inference that is based on the PC-algorithm. In particular, we applied the IDA algorithm
\citep{MaathuisColomboKalischBuhlmann10, MaathuisKalischBuehlmann09} to the
yeast gene expression data discussed above. The IDA algorithm conceptually
consists of two-steps: one first estimates the Markov equivalence class of
DAGs using the PC-algorithm, and one then applies Pearl's do-calculus
\citep{Pearl00} to each DAG in the Markov equivalence class. (The algorithm
uses a fast local implementation that does not require listing all DAGs in the
equivalence class.) One can then obtain estimated lower bounds on the sizes
of the causal effects between all pairs of genes. For each of the 25 random
variable orderings, we ranked the gene pairs according to these lower bounds, and
compared these rankings to a gold standard set of large causal
effects computed from gene knock-out data. Figure \ref{fig.roc.old} shows
the large variability in the resulting receiver operating
characteristic (ROC) curves. The ROC curve that was published in
\cite{MaathuisColomboKalischBuhlmann10} was significantly better than
random guessing with $p<0.001$, and is somewhere in the middle. Some of the
other curves are much better, while there are also curves that are
indistinguishable from random guessing.

\begin{figure}[!h]\centering%
  \subfigure[Edges occurring in the estimated skeletons for 25
  random variable orderings, as well as for the original ordering
  (shown as variable ordering 26). A black entry for an edge $i$ and a variable ordering $j$ means that edge $i$ occurred in the estimated skeleton for the $j$th variable ordering. The edges along the $x$-axis are ordered
  according to their frequency of occurrence in the estimated skeletons, from
  edges that occurred always to edges that occurred only once.
  Edges that did not occur for any of the variable orderings were
  omitted.  For technical reasons, only every 10th edge is actually
  plotted.]{\label{fig.levelplot.1}
     \includegraphics[scale=0.32,angle=0]{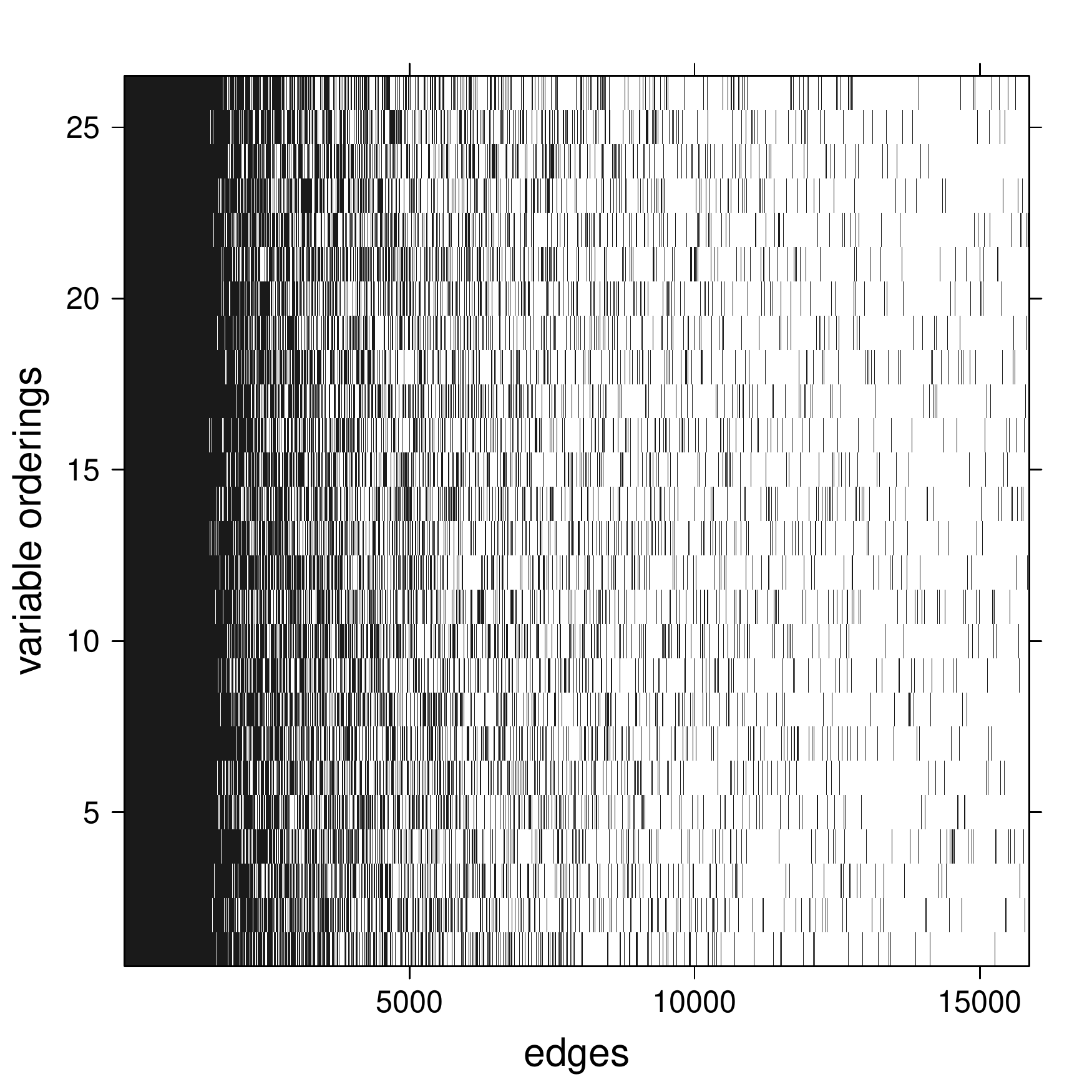}
  }\qquad
   \subfigure[ROC curves corresponding to the 25 random orderings of the
   variables (solid black), where the curves are generated exactly as in
   \cite{MaathuisColomboKalischBuhlmann10}. The ROC curve for the original
   ordering of the variables (dashed blue) was published in
   \cite{MaathuisColomboKalischBuhlmann10}. The dashed-dotted red curve
   represents random guessing.]{\label{fig.roc.old}
     \includegraphics[scale=0.34,angle=0]{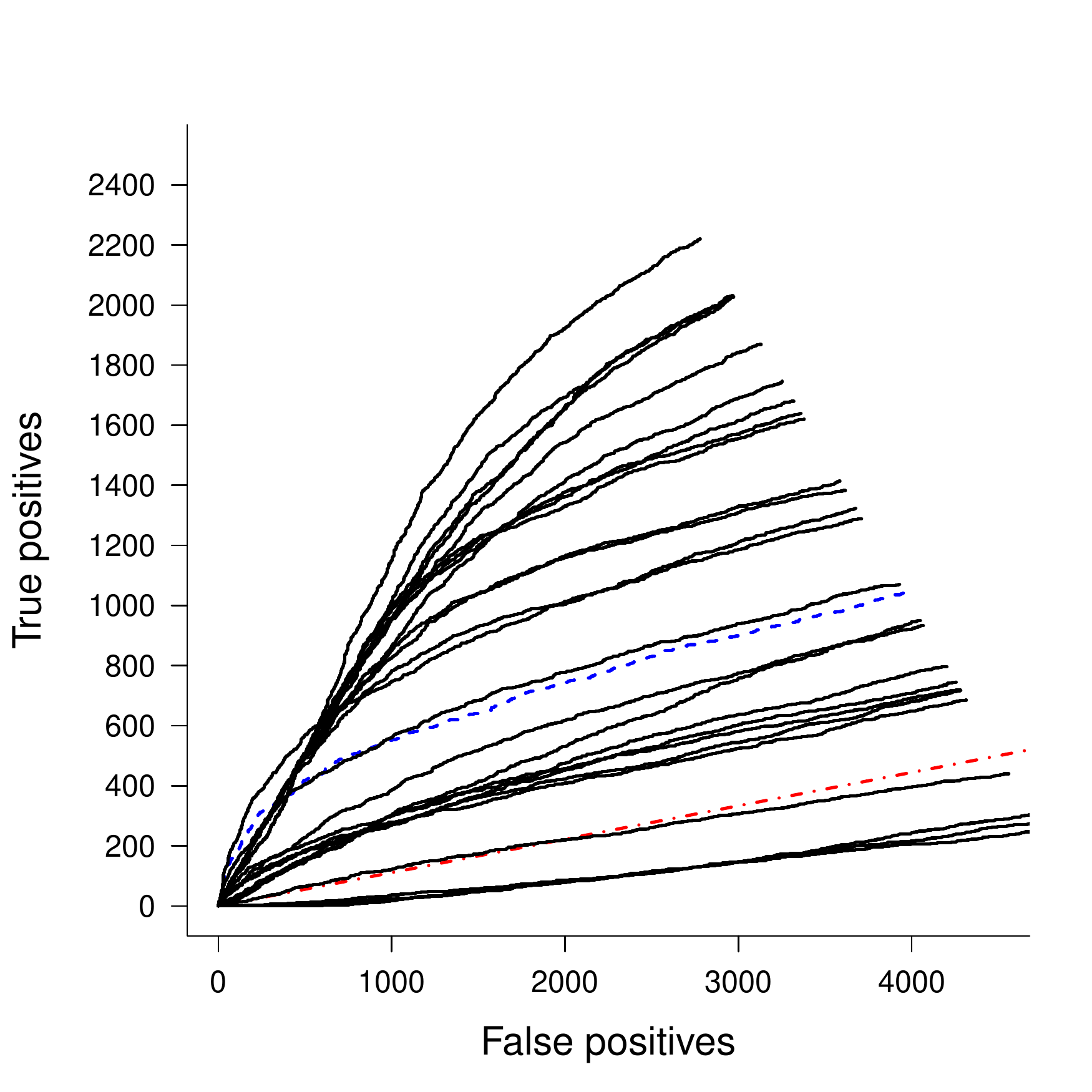}
  }
  \label{fig.hughesetal.1}
  \caption{Analysis of the yeast gene expression data \citep{Hughes00} for 25
    random orderings of the variables, using tuning parameter
    $\alpha=0.01$. The estimated graphs and resulting causal rankings are
    highly order-dependent.}
\end{figure}


The remainder of the paper is organized as follows. In Section
\ref{sec.graph.def} we discuss some background and terminology.
Section \ref{sec.original.pc} explains the original PC-algorithm. Section
\ref{sec.modified.pc} introduces modifications of the PC-algorithm (and
hence also of the (R)FCI- and CCD-algorithms) that remove part or all of the
order-dependence. These modifications are identical to their original
counterparts when perfect conditional independence information is
used. When applied to data, the modified algorithms are partly or fully
order-independent. Moreover, they are consistent in high-dimensional
settings under the same conditions as the original algorithms. Section
\ref{sec.simulation.results} compares all algorithms in simulations, and
Section \ref{sec.simulation.realdata} compares them on the yeast gene
expression data discussed above. We close with a discussion in Section
\ref{sec.discussion}.

\section{Preliminaries}\label{sec.graph.def}

\subsection{Graph terminology}

A graph $\mathcal{G}=(\mathbf{V},\mathbf{E})$ consists of a vertex set
$\mathbf{V}=\{X_1,\dots,X_p\}$ and an edge set $\mathbf{E}$.
The vertices represent random variables and the edges
represent relationships between pairs of variables.

A graph
containing only directed edges ($\rightarrow$) is \emph{directed},
one containing only undirected edges ($-$) is \emph{undirected}, and
one containing directed and/or undirected edges is \emph{partially
directed}. The \emph{skeleton} of a partially directed graph is the
undirected graph that results when all directed edges are replaced by
undirected edges.

All graphs we consider are \textit{simple}, meaning that there is at most one
edge between any pair of vertices. If an edge is present, the vertices are said
to be \textit{adjacent}. If all pairs of vertices in a graph are adjacent,
the graph is called \emph{complete}. The \textit{adjacency set} of a vertex
$X_i$ in a graph $\mathcal{G}=(\mathbf{V},\mathbf{E})$, denoted by
adj$(\mathcal{G},X_i)$, is the set of all vertices in $\mathbf{V}$ that are
adjacent to $X_i$ in $\mathcal{G}$. A vertex $X_j$ in adj$(\mathcal{G},X_i)$
is called a \textit{parent} of $X_i$ if $X_j \rightarrow X_i$. The
corresponding set of parents is denoted by pa$(\mathcal{G},X_i)$.

A \textit{path} is a sequence of distinct adjacent
vertices. A \textit{directed path} is a path along directed
edges that follows the direction of the arrowheads. A \emph{directed cycle} is
formed by a directed path from $X_i$ to $X_j$ together with the edge $X_j
\rightarrow X_i$. A (partially) directed graph is called a
\emph{(partially) directed acyclic graph} if it does not contain
directed cycles.

A triple $(X_i,X_j,X_k)$ in a graph $\mathcal{G}$ is
\emph{unshielded} if $X_i$ and $X_j$ as well as $X_j$ and $X_k$ are
adjacent, but $X_i$ and $X_k$ are not adjacent in $\mathcal{G}$. A
\emph{v-structure} $(X_i,X_j,X_k)$ is an unshielded triple in a graph
$\mathcal{G}$ where the edges are oriented as $X_i \rightarrow X_j
\leftarrow X_k$.

\subsection{Probabilistic and causal interpretation of DAGs}

We use the notation $X_i \ci X_j | \mathbf{S}$ to indicate that $X_i$ is
independent of $X_j$ given $\mathbf{S}$, where $\mathbf{S}$ is a set of
variables not containing $X_i$ and $X_j$ \citep{Dawid80}. If $\mathbf{S}$ is
the empty set, we simply write $X_i \ci X_j$. If $X_i\ci X_j | \mathbf{S}$,
we refer to $\mathbf{S}$ as a \emph{separating set} for $(X_i,X_j)$. A
separating set $\mathbf{S}$ for $(X_i,X_j)$ is called \emph{minimal} if
there is no proper subset $\mathbf{S'}$ of $\mathbf{S}$ such that $X_i\ci
X_j | \mathbf{S'}$.

A distribution $Q$ is said to \emph{factorize} according to a DAG
$\mathcal{G}=(\mathbf{V},\mathbf{E})$ if the joint density of $\mathbf{V} =
(X_1,\dots,X_p)$ can be written as the product of the conditional densities
of each variable given its parents in $\mathcal G$: $q(X_1,\dots,X_p) =
\prod_{i=1}^{p} q(X_i\vert\text{pa}(\mathcal{G}, X_i))$.

A DAG entails conditional independence relationships via a graphical
criterion called \emph{d-separation} \citep{Pearl00}. If two vertices $X_i$ and
$X_j$ are not adjacent in a DAG $\mathcal{G}$, then they are 
d-separated in $\mathcal G$ by a subset $\mathbf{S}$ of
the remaining vertices. If $X_i$ and $X_j$ are d-separated by $\mathbf{S}$, then $X_i \ci X_j | \mathbf{S}$ in any distribution
$Q$ that factorizes according to $\mathcal{G}$. A distribution $Q$ is said to be
\emph{faithful} to a DAG $\mathcal{G}$ if the reverse implication also
holds, that is, if the conditional independence relationships in $Q$ are
exactly the same as those that can be inferred from $\mathcal{G}$ using
d-separation.

Several DAGs can describe exactly the same conditional independence
information. Such DAGs are called Markov equivalent and form a Markov
equivalence class. Markov equivalent DAGs have the same skeleton and the same
v-structures, and a Markov equivalence class can be described
uniquely by a completed partially directed acyclic graph (CPDAG)
\citep{AndersonEtAll97, Chickering02}. A CPDAG is a partially directed acyclic graph
with the following properties: every directed edge exists in every DAG in
the Markov equivalence class, and for every undirected edge $X_i -
X_j$ there exists a DAG with $X_i \rightarrow X_j$ and a DAG with $X_i
\leftarrow X_j$ in the Markov equivalence class. A CPDAG $\mathcal C$ is said to
\emph{represent} a DAG $\mathcal G$ if $\mathcal G$ belongs to the
Markov equivalence class described by $\mathcal C$.

A DAG can be interpreted causally in the following way
\citep{Pearl00,Pearl09,SpirtesEtAl00}: $X_1$ is a direct cause of $X_2$
only if $X_1 \rightarrow X_2$, and $X_1$ is a possibly indirect cause of
$X_2$ only if there is a directed path from $X_1$ to $X_2$.


\section{The PC-algorithm}\label{sec.original.pc}

We now describe the PC-algorithm in detail. In Section
\ref{sec.oracle.orig.pc}, we discuss the algorithm under the
assumption that we have perfect conditional independence information
between all variables in $\mathbf{V}$. We refer to this as the \emph{oracle
  version}. In Section \ref{sec.original.PC.sample} we discuss the more
realistic situation where conditional independence relationships have to be
estimated from data. We refer to this as the \emph{sample version}.

\subsection{Oracle version}\label{sec.oracle.orig.pc}

A sketch of the PC-algorithm is given in Algorithm \ref{pseudo.pc}. We see that
the algorithm consists of three steps. Step 1 finds the skeleton and
separation sets, while Steps 2 and 3 determine the orientations of the edges.

\begin{algorithm}
\caption{The PC-algorithm (oracle version)}
\label{pseudo.pc}
\begin{algorithmic}[1]
   \REQUIRE Conditional independence information among all variables in
   $\mathbf{V}$, and an ordering $\text{order}(\mathbf{V})$ on the variables
   \STATE Find the skeleton $\mathcal{C}$ and separation sets using Algorithm
   \ref{pseudo.old.pc};
   \STATE Orient unshielded triples in the skeleton $\mathcal{C}$ based on
   the separation sets;
   \STATE In $\mathcal{C}$ orient as many of the remaining undirected edges
   as possible by repeated application of rules R1-R3 (see text);
   \RETURN{Output graph $(\mathcal C)$ and separation sets (sepset)}.
\end{algorithmic}
\end{algorithm}

\begin{algorithm}
\caption{Step 1 of the PC-algorithm (oracle version)}
\label{pseudo.old.pc}
\begin{algorithmic}[1]
   \REQUIRE Conditional independence information among all variables in
   $\mathbf{V}$, and an ordering order$(\mathbf{V})$ on the variables
   \STATE Form the complete undirected graph $\mathcal{C}$ on the vertex set
   $\mathbf{V}$
   \STATE Let $\ell = -1$;
   \REPEAT
   \STATE Let $\ell = \ell + 1$;
   \REPEAT
   \STATE Select a (new) ordered pair of vertices $(X_i,X_j)$ that are
   adjacent in $\mathcal{C}$ and satisfy $|\text{adj}(\mathcal{C},X_i)
   \setminus \{X_j\}| \geq \ell$, using order$(\mathbf{V})$;
   \REPEAT
   \STATE Choose a (new) set $\mathbf{S} \subseteq
   \text{adj}(\mathcal{C},X_i) \setminus \{X_j\}$ with $|\mathbf{S}|=\ell$,
   using order$(\mathbf{V})$;
   \IF{$X_i$ and $X_j$ are conditionally independent given $\mathbf{S}$}
   \STATE Delete edge $X_i - X_j$ from $\mathcal{C}$;
   \STATE Let $\text{sepset}(X_i,X_j) = \text{sepset}(X_j,X_i) = \mathbf{S}$;
   \ENDIF
   \UNTIL{$X_i$ and $X_j$ are no longer adjacent in $\mathcal{C}$ or all
     $\mathbf{S} \subseteq \text{adj}(\mathcal{C},X_i) \setminus \{X_j\}$
     with $|\mathbf{S}|=\ell$ have been considered}
   \UNTIL{all ordered pairs of adjacent vertices $(X_i,X_j)$ in
     $\mathcal{C}$ with $|\text{adj}(\mathcal{C},X_i) \setminus \{X_j\}|
     \geq \ell$ have been considered}
   \UNTIL{all pairs of adjacent vertices $(X_i,X_j)$ in $\mathcal{C}$
     satisfy $|\text{adj}(\mathcal{C},X_i) \setminus \{X_j\}| \leq \ell$}
   \RETURN{$\mathcal{C}$, sepset}.
 \end{algorithmic}
\end{algorithm}

Step 1 is given in pseudo-code in Algorithm \ref{pseudo.old.pc}.
We start with a complete undirected graph $\mathcal{C}$. This graph is
subsequently thinned out in the loop on lines 3-15 in Algorithm
\ref{pseudo.old.pc}, where an edge $X_i - X_j$ is deleted if
 $X_i \ci X_j | \mathbf{S}$ for some subset $\mathbf{S}$ of the
remaining variables. These conditional independence queries are organized
in a way that makes the algorithm computationally efficient for
high-dimensional sparse graphs.

First, when $\ell=0$, all pairs of vertices are tested for marginal
independence. If $X_i \ci X_j$, then the edge  $X_i - X_j$ is
deleted and the empty set is saved as separation set in
$\text{sepset}(X_i,X_j)$ and $\text{sepset}(X_j,X_i)$. After all pairs of
vertices have been considered (and many edges might have been deleted), the
algorithm proceeds to the next step with $\ell=1$.

When $\ell=1$, the algorithm chooses an ordered pair of vertices
$(X_i,X_j)$ still adjacent in $\mathcal C$, and checks $X_i \ci
X_j | \mathbf{S}$ for subsets $\mathbf{S}$ of size $\ell=1$ of
$\text{adj}(\mathcal{C},X_i) \setminus \{X_j\}$. If such a conditional
independence is found, the edge $X_i-X_j$ is removed, and the corresponding
conditioning set $\mathbf{S}$ is saved in $\text{sepset}(X_i,X_j)$ and
$\text{sepset}(X_j,X_i)$. If all ordered pairs of adjacent vertices have been
considered for conditional independence given all subsets of size $\ell$ of
their adjacency sets, the algorithm again increases $\ell$ by one. This
process continues until all adjacency sets in the current graph are smaller
than $\ell$. At this point the skeleton and the separation sets have been
determined.

Step 2 determines the v-structures. In particular, it
considers all unshielded triples in $\mathcal C$, and orients an unshielded
triple $(X_i,X_j,X_k)$ as a v-structure if and only if $X_j \notin
\text{sepset}(X_i,X_k)$.

Finally, Step 3 orients as many of the
remaining undirected edges as possible by repeated application of the
following three rules:
\begin{enumerate}
  \item[R1:] orient $X_j - X_k$ into $X_j \rightarrow X_k$ whenever there is
    a directed edge $X_i \rightarrow X_j$ such that $X_i$ and $X_k$ are
    not adjacent (otherwise a new v-structure is created);
  \item[R2:] orient $X_i - X_j$ into $X_i \rightarrow X_j$ whenever there is
    a chain $X_i \rightarrow X_k \rightarrow X_j$ (otherwise a directed
    cycle is created);
  \item[R3:] orient $X_i - X_j$ into $X_i \rightarrow X_j$ whenever there
    are two chains $X_i - X_k \rightarrow X_j$ and $X_i - X_l \rightarrow
    X_j$ such that $X_k$ and $X_l$ are not adjacent (otherwise a new
    v-structure or a directed cycle is created).
\end{enumerate}

The PC-algorithm was shown to be sound and complete.
\begin{theorem}\label{th.correct.pc.original}
   (Theorem 5.1 on p.410 of \cite{SpirtesEtAl00}) Let the distribution of
   $\mathbf{V}$ be faithful to a DAG $\mathcal G=(\mathbf{V},\mathbf{E})$,
   and assume that we are given perfect conditional independence
   information about all pairs of variables $(X_i,X_j)$ in $\mathbf{V}$
   given subsets $\mathbf{S} \subseteq \mathbf{V} \setminus
   \{X_i,X_j\}$. Then the output of the PC-algorithm is the CPDAG that
   represents $\mathcal G$.
\end{theorem}

We briefly discuss the main ingredients of the proof, as these will be
useful for understanding our modifications in Section
\ref{sec.modified.pc}. The faithfulness assumption implies that conditional
independence in the distribution of $\mathbf{V}$ is equivalent to
d-separation in the graph $\mathcal G$. The skeleton of $\mathcal G$ can
then be determined as follows: $X_i$ and $X_j$ are adjacent in $\mathcal G$
if and only if they are conditionally dependent given any subset
$\mathbf{S}$ of the remaining nodes. Naively, one could therefore check all
these conditional dependencies, which is known as the SGS algorithm
\citep{SpirtesEtAl00}. The PC-algorithm obtains the same result with fewer
tests, by using the following fact about DAGs: two variables $X_i$ and
$X_j$ in a DAG $\mathcal G$ are d-separated by some subset $\mathbf{S}$ of
the remaining variables if and only if they are d-separated by
$\text{pa}(\mathcal{G}, X_i)$ or $\text{pa}(\mathcal{G}, X_j)$. The PC-algorithm is guaranteed to check these conditional independencies: at all stages of the algorithm, the graph $\mathcal C$ is a supergraph of the true
CPDAG, and the algorithm checks conditional dependencies given all subsets of the
adjacency sets, which obviously include the parent sets.

The v-structures are determined based on Lemmas 5.1.2 and 5.1.3 of
\cite{SpirtesEtAl00}. The soundness and completeness of the orientation
rules in Step 3 was shown in \cite{Meek95} and \cite{AndersonEtAll97}.

\subsection{Sample version}\label{sec.original.PC.sample}

In applications, we of course do not have perfect conditional independence
information. Instead, we assume that we have an i.i.d.\ sample of size $n$
of $\mathbf{V} = (X_1,\dots,X_p)$. A sample version of the PC-algorithm can then
be obtained by replacing all steps
where conditional independence decisions were taken by statistical
tests for conditional independence at some pre-specified level $\alpha$.
For example, if the distribution of $\mathbf{V}$ is multivariate Gaussian,
one can test for zero partial
correlation, see, e.g., \cite{KalischBuehlmann07a}.
%
%
We note that the significance level $\alpha$ is used for many tests, and
plays the role of a tuning parameter, where smaller values of $\alpha$ tend
to lead to sparser graphs.

\subsection{Order-dependence in the sample version}\label{sec.order-dep.sample.old}

Let order($\mathbf{V}$) denote an ordering on the variables in $\mathbf{V}$. We now
consider the role of order($\mathbf{V}$) in every step of the
algorithm. Throughout, we assume that all tasks are performed according to the lexicographical ordering of
order$(\mathbf{V})$, which is the standard implementation in \texttt{pcalg}
\citep{KalischEtAl12} and TETRAD IV \citep{SpirtesEtAl00}, and is
called ``PC-1'' in \cite{SpirtesEtAl00} (Section 5.4.2.4).

In Step 1, order($\mathbf{V}$) affects the estimation of the \emph{skeleton} and the
\emph{separating sets}. In particular, at each level of $\ell$, order($\mathbf{V}$)
determines the order in which pairs of adjacent vertices and subsets $\mathbf{S}$ of their adjacency sets are considered (see lines 6 and 8 in Algorithm \ref{pseudo.old.pc}).
The skeleton $\mathcal C$ is updated after each edge removal.
Hence, the adjacency sets typically change within one level of
$\ell$, and this affects which other conditional independencies are checked,
since the algorithm only conditions on subsets of the adjacency sets.
In the oracle version, we have perfect conditional independence
information, and all orderings on the variables lead to the same output. In
the sample version, however, we typically make mistakes in keeping or
removing edges. In such cases, the resulting changes in the adjacency sets
can lead to different skeletons, as illustrated in Example
\ref{ex.orderdep.old}.

Moreover, different variable orderings can lead to different separating sets in Step 1. In the
oracle version, this is not important, because any valid separating set leads to the correct v-structure decision
in Step 2. In the sample version, however,
different separating sets in Step 1 of the algorithm may yield different decisions about v-structures in Step 2.  This is illustrated in Example \ref{ex.sepset}.

Finally, we consider the role of order($\mathbf{V}$) on the
\emph{orientation rules} in Steps 2 and 3 of the sample version of the PC-algorithm. Example \ref{ex.orientation}
illustrates that different variable orderings can lead to different orientations, even if the skeleton and separating sets are order-independent.

\noindent
\begin{example}\label{ex.orderdep.old}
  (Order-dependent skeleton in the sample PC-algorithm.)
  Suppose that the distribution
  of $\mathbf{V}=\{X_1,X_2,X_3,X_4, X_5\}$ is faithful to the DAG in Figure
  \ref{fig.true.dag}. This DAG encodes the following conditional
  independencies with minimal separating sets: $X_1 \ci X_2$ and $X_2 \ci
  X_4| \{X_1,X_3\}$ .

  Suppose that we have an i.i.d.\ sample of $(X_1,X_2,X_3,X_4,
  X_5)$, and that the following conditional independencies with minimal
  separating sets are judged to hold at some level $\alpha$: $X_1 \ci X_2$,
  $X_2 \ci X_4 |\{X_1,X_3\}$, and $X_3\ci X_4|\{X_1,X_5\}$. Thus, the first
  two are correct, while the third is false.

  We now apply the PC-algorithm with two different orderings: $\text{order}_1(\mathbf{V}) =
  (X_1,X_4,X_2, \linebreak X_3,X_5)$ and $\text{order}_2(\mathbf{V}) =
  (X_1,X_3,X_4,X_2,X_5)$. The resulting skeletons are shown in Figures
  \ref{pc.old.output1} and \ref{pc.old.output2}, respectively. We see that
  the skeletons are different, and that both are incorrect as the edge $X_3
  - X_4$ is missing. The skeleton for $\text{order}_2(\mathbf{V})$ contains an
  additional error, as there is an additional edge $X_2 - X_4$.

  We now go through Algorithm \ref{pseudo.old.pc} to see what
  happened. We start with a complete undirected graph on $\mathbf{V}$. When
  $\ell=0$, variables are tested for marginal independence, and the
  algorithm correctly removes the edge between $X_1$ and $X_2$. No other
  conditional independencies are found when $\ell=0$ or $\ell=1$. When
  $\ell=2$, there are two pairs of vertices that are thought to be conditionally
  independent given a subset of size 2, namely the pairs $(X_2,X_4)$ and
  $(X_3,X_4)$.

  In $\text{order}_1(\mathbf{V})$, the pair $(X_4,X_2)$ is considered
  first. The corresponding edge is removed, as $X_4\ci X_2|\{X_1,X_3\}$ and
  $\{X_1,X_3\}$ is a subset of $\text{adj}(\mathcal{C},
  X_4)=\{X_1,X_2,X_3,X_5\}$. Next, the pair $(X_4, X_3)$ is considered and the corresponding edges is
  erroneously removed, because of the wrong decision that $X_4\ci X_3
  |\{X_1,X_5\}$ and the fact that $\{X_1,X_5\}$ is a subset of
  $\text{adj}(\mathcal{C},X_4)=\{X_1,X_3,X_5\}$.

  In $\text{order}_2(\mathbf{V})$, the pair $(X_3,X_4)$ is considered
  first, and the corresponding edge is erroneously removed. Next,
  the algorithm considers the pair $(X_4,X_2)$. The corresponding separating set $\{X_1,X_3\}$ is not a subset
  of $\text{adj}(\mathcal{C},X_4)= \{X_1,X_2,X_5\}$, so that the edge $X_2-X_4$ remains. Next, the algorithm considers the pair $(X_2,X_4)$. Again, the separating set
  $\{X_1,X_3\}$ is not a subset of  $\text{adj}(\mathcal{C},X_2)=\{X_3, X_4, X_5\}$, so that the edge
  $X_2-X_4$ again remains. In other words, since
  $(X_3,X_4)$ was considered first in $\text{order}_2(\mathbf{V})$, the
  adjacency set of $X_4$ was affected and no longer contained $X_3$, so
  that the algorithm ``forgot" to check the conditional independence $X_2
  \ci X_4|\{X_1,X_3\}$.

  \begin{figure}[!t]\centering%
  \subfigure[True DAG.]{
      \begin{tikzpicture}
        \tikzstyle{every circle node}=
        [%
        draw=black!20!black,%
        minimum size=7mm,%
        circle,%
        thick%
        ]
        \tikzstyle{every rectangle node}=
        [%
        draw=black!20!black,%
        minimum size=6mm,%
        rectangle,%
        thick%
        ]

        \node[circle] (A) at (-1.4, 0) {$X_{1}$};
        \node[circle] (B) at (1.4, 0) {$X_{2}$};
        \node[circle] (C) at (0, 1.5) {$X_{3}$};
        \node[circle] (D) at (0, -1.5) {$X_{4}$};
        \node[circle] (E) at (0, 0) {$X_{5}$};

        \path [thick,shorten >=1pt,-stealth']
        (A) edge[->] (D)
        (A) edge[->] (E)
        (A) edge[->] (C)
        (B) edge[->] (E)
        (B) edge[->, bend right=10] (C)
        (D) edge[->] (E)
        (E) edge[<-] (C)
        (C) edge[->, bend left=40] (D);

      \end{tikzpicture}\label{fig.true.dag}
    }\quad%
    \subfigure[Skeleton returned by the oracle version of Algorithm \ref{pseudo.old.pc} with any ordering, and by the sample version of Algorithm
    \ref{pseudo.old.pc} with
    $\text{order}_1(\mathbf{V})$.]{
      \begin{tikzpicture}
        \tikzstyle{every circle node}=
        [%
        draw=black!20!black,%
        minimum size=7mm,%
        circle,%
        thick%
        ]
        \tikzstyle{every rectangle node}=
        [%
        draw=black!20!black,%
        minimum size=7mm,%
        rectangle,%
        thick%
        ]

        \node[circle] (A) at (-1.4, 0) {$X_{1}$};
        \node[circle] (B) at (1.4, 0) {$X_{2}$};
        \node[circle] (C) at (0, 1.5) {$X_{3}$};
        \node[circle] (D) at (0, -1.5) {$X_{4}$};
        \node[circle] (E) at (0, 0) {$X_{5}$};

        \path [thick,shorten >=1pt,-stealth']
        (A) edge[-] (D)
        (A) edge[-] (E)
        (A) edge[-] (C)
        (B) edge[-] (E)
        (B) edge[-] (C)
        (D) edge[-] (E)
        (E) edge[-] (C);

      \end{tikzpicture}\label{pc.old.output1}
    }\quad%
    \subfigure[Skeleton returned by the sample version of Algorithm
    \ref{pseudo.old.pc} with
    $\text{order}_2(\mathbf{V})$.]{
          \begin{tikzpicture}
        \tikzstyle{every circle node}=
        [%
        draw=black!20!black,%
        minimum size=7mm,%
        circle,%
        thick%
        ]
        \tikzstyle{every rectangle node}=
        [%
        draw=black!20!black,%
        minimum size=7mm,%
        rectangle,%
        thick%
        ]

        \node[circle] (A) at (-1.4, 0) {$X_{1}$};
        \node[circle] (B) at (1.4, 0) {$X_{2}$};
        \node[circle] (C) at (0, 1.5) {$X_{3}$};
        \node[circle] (D) at (0, -1.5) {$X_{4}$};
        \node[circle] (E) at (0, 0) {$X_{5}$};

        \path [thick,shorten >=1pt,-stealth']
        (A) edge[-] (D)
        (A) edge[-] (E)
        (A) edge[-] (C)
        (B) edge[-] (E)
        (B) edge[-] (C)
        (D) edge[-] (E)
        (E) edge[-] (C)
        (B) edge[-] (D);

      \end{tikzpicture}\label{pc.old.output2}
    }
    \caption{Graphs corresponding to Examples \ref{ex.orderdep.old} and
      \ref{ex.orderindep.new}.}
    \label{ex1.graphs.old}
  \end{figure}
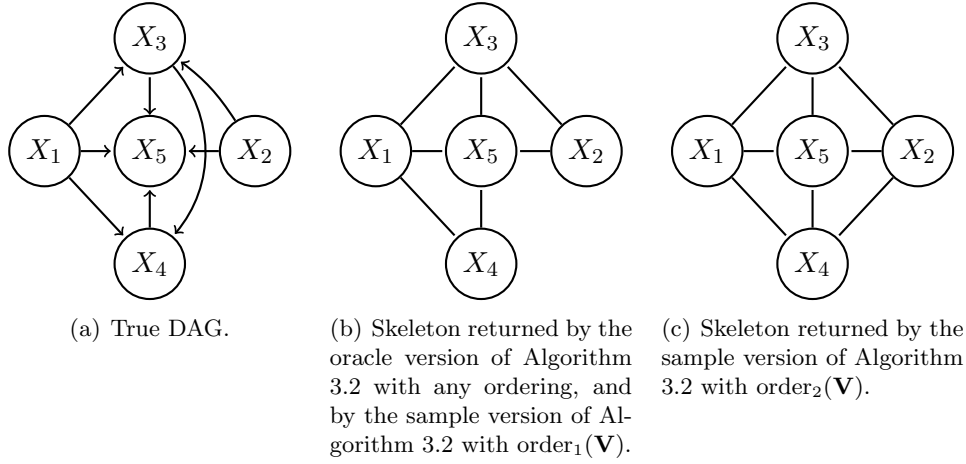
\end{example}


\begin{example}\label{ex.sepset}
  (Order-dependent separating sets and v-structures in the sample
  PC-algorithm.) Suppose that the distribution of
  $\mathbf{V}=\{X_1,X_2,X_3, X_4,X_5\}$ is faithful to the DAG in Figure
  \ref{ex.true.dag}. This DAG encodes the following conditional
  independencies with minimal separating sets: $X_1 \ci X_3| \{X_2\}$, $X_1
  \ci X_4| \{X_2\}$, $X_1 \ci X_4 | \{X_3\}$, $X_2 \ci
  X_4| \{X_3\}$, $X_2\ci X_5 | \{X_1,X_3\}$, $X_2\ci X_5 | \{X_1,X_4\}$, $X_3 \ci X_5 | \{X_1,X_4\}$ and $X_3 \ci X_5 | \{X_2,X_4\}$.

  We consider the oracle PC-algorithm with two different orderings on
  the variables: $\text{order}_3(\mathbf{V})=(X_1,X_4,X_2,X_3,X_5)$ and
  $\text{order}_4(\mathbf{V})=(X_1,X_4,X_3,X_2,X_5)$. For
  $\text{order}_3(\mathbf{V})$, we obtain sepset$(X_1,X_4) = \{X_2\}$,
  while for $\text{order}_4(\mathbf{V})$ we get
  sepset$(X_1,X_4)= \{X_3\}$. Thus, the separating sets
   are order-dependent. However, we obtain the same v-structure $X_1 \to
   X_5 \leftarrow X_4$ for both orderings, since $X_5$ is not in the
   sepset$(X_1,X_4)$, regardless of the ordering. In fact, this holds in
   general, since in the oracle version of the PC-algorithm, a vertex is
   either in all possible separating sets or in none of them (cf. \cite[Lemma 5.1.3]{SpirtesEtAl00}).

  Now suppose that we have an i.i.d.\ sample of $(X_1,X_2,X_3,X_4,X_5)$.
  Suppose that at some level $\alpha$, all true conditional independencies
  are judged to hold, and $X_1\ci X_3|\{X_4\}$ is thought to hold by mistake. We again consider two
  different orderings: $\text{order}_5(\mathbf{V}) =
  (X_1,X_3,X_4,X_2,X_5)$ and $\text{order}_6(\mathbf{V}) =
  (X_3,X_1,X_2,X_4,X_5)$. With $\text{order}_5(\mathbf{V})$ we obtain the
  incorrect $\text{sepset}(X_1,X_3)=\{X_4\}$. This also leads to an
  incorrect v-structure $X_1 \to X_2\leftarrow X_3$ in Step 2 of the
  algorithm. With $\text{order}_6(\mathbf{V})$, we obtain the correct
  $\text{sepset}(X_1,X_3)=\{X_2\}$, and hence correctly find that $X_1 -
  X_2 -X_3$ is not a v-structure in Step 2. This illustrates that
  order-dependent separating sets in Step 1 of the sample version of the
  PC-algorithm can lead to order-dependent v-structures in Step 2 of the
  algorithm.

  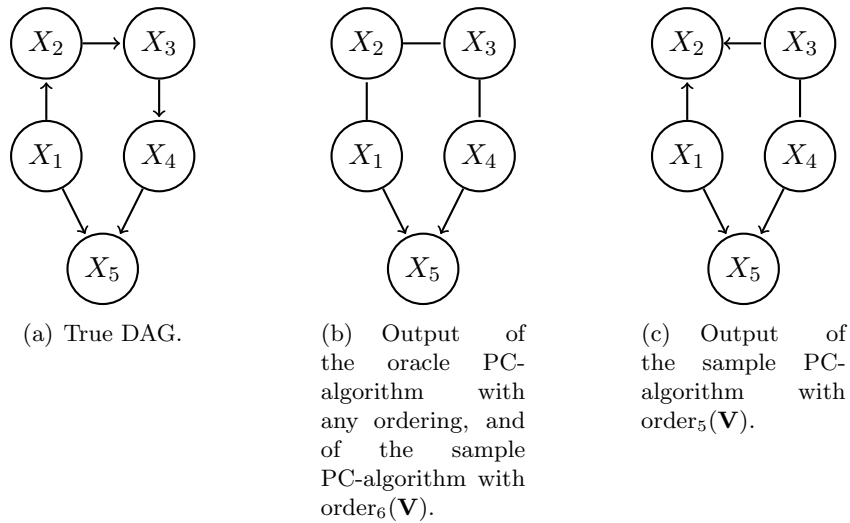
\begin{figure}[!h]\centering%
   \subfigure[True DAG.]{
     \begin{tikzpicture}
       \tikzstyle{every circle node}=
       [%
       draw=black!20!black,%
       minimum size=7mm,%
       circle,%
       thick%
       ]
       \tikzstyle{every rectangle node}=
       [%
       draw=black!20!black,%
       minimum size=6mm,%
       rectangle,%
       thick%
       ]

       \node[circle] (A) at (0, 0) {$X_{1}$};
       \node[circle] (B) at (0, 1.5) {$X_{2}$};
       \node[circle] (C) at (1.5, 1.5) {$X_{3}$};
       \node[circle] (D) at (1.5, 0) {$X_{4}$};
       \node[circle] (E) at (0.75,-1.5) {$X_{5}$};

       \path [thick,shorten >=1pt,-stealth']
       (A) edge[->] (B)
       (A) edge[->] (E)
       (B) edge[->] (C)
       (C) edge[->] (D)
       (D) edge[->] (E);

     \end{tikzpicture}\label{ex.true.dag}
   }\qquad \qquad %
   \subfigure[Output of the oracle PC-algorithm with any
   ordering, and of the sample PC-algorithm with $\text{order}_6(\mathbf{V})$.]{
     \begin{tikzpicture}
       \tikzstyle{every circle node}=
       [%
       draw=black!20!black,%
       minimum size=7mm,%
       circle,%
       thick%
       ]
       \tikzstyle{every rectangle node}=
       [%
       draw=black!20!black,%
       minimum size=7mm,%
       rectangle,%
       thick%
       ]

       \node[circle] (A) at (0, 0) {$X_{1}$};
       \node[circle] (B) at (0, 1.5) {$X_{2}$};
       \node[circle] (C) at (1.5, 1.5) {$X_{3}$};
       \node[circle] (D) at (1.5, 0) {$X_{4}$};
       \node[circle] (E) at (0.75, -1.5) {$X_{5}$};

       \path [thick,shorten >=1pt,-stealth']
       (A) edge[-] (B)
       (A) edge[->] (E)
       (B) edge[-] (C)
       (C) edge[-] (D)
       (D) edge[->](E);

     \end{tikzpicture}\label{ex.true.skelet}
   }\qquad \qquad %
    \subfigure[Output of the sample PC-algorithm with
    $\text{order}_5(\mathbf{V})$.]{
     \begin{tikzpicture}
       \tikzstyle{every circle node}=
       [%
       draw=black!20!black,%
       minimum size=7mm,%
       circle,%
       thick%
       ]
       \tikzstyle{every rectangle node}=
       [%
       draw=black!20!black,%
       minimum size=7mm,%
       rectangle,%
       thick%
       ]

       \node[circle] (A) at (0, 0) {$X_{1}$};
       \node[circle] (B) at (0, 1.5) {$X_{2}$};
       \node[circle] (C) at (1.5, 1.5) {$X_{3}$};
       \node[circle] (D) at (1.5, 0) {$X_{4}$};
       \node[circle] (E) at (0.75, -1.5) {$X_{5}$};

       \path [thick,shorten >=1pt,-stealth']
       (A) edge[->] (B)
       (A) edge[->] (E)
       (B) edge[<-] (C)
       (C) edge[-] (D)
       (D) edge[->] (E);

     \end{tikzpicture}\label{ex.wrong.output}
   }
   \caption{Graphs corresponding to Examples \ref{ex.sepset} and \ref{ex.sepset.new}.}
   \label{ex2.graphs}
 \end{figure}
\end{example}

\begin{example}\label{ex.orientation}
  (Order-dependent orientation rules in Steps 2 and 3 of the sample
  PC-algorithm.) Consider the graph in Figure \ref{ex.vstruct} with two
  unshielded triples $(X_1,X_2,X_3)$ and $(X_2,X_3,X_4)$, and assume this
  is the skeleton after Step 1 of the sample version of the
  PC-algorithm. Moreover, assume that we found $\text{sepset}(X_1,X_3) =
  \text{sepset}(X_2,X_4)=\text{sepset}(X_1,X_4)=\emptyset$. Then in Step 2
  of the algorithm, we obtain two v-structures $X_1 \to X_2 \leftarrow X_3$
  and $X_2 \to X_3 \leftarrow X_4$. Of course this means that at least one
  of the statistical tests is wrong, but this can happen in the sample
  version. We now have conflicting information about the orientation of the
  edge $X_2 - X_3$. In the current implementation of \texttt{pcalg}, where
  conflicting edges are simply overwritten, this means that the orientation
  of $X_2 - X_3$ is determined by the v-structure that is last
  considered. Thus, we obtain $X_1 \to X_2 \to X_3 \leftarrow X_4$ if
  $(X_2,X_3,X_4)$ is considered last, while we get $X_1 \to X_2 \leftarrow
  X_3 \leftarrow X_4$ if $(X_1,X_2,X_3)$ is considered last.

  Next, consider the graph in Figure \ref{ex.R1}, and assume that this is
  the output of the sample version of the PC-algorithm after Step 2. Thus,
  we have two v-structures, namely $X_1 \rightarrow X_2 \leftarrow X_3$ and
  $X_4 \rightarrow X_5 \leftarrow X_6$, and four unshielded triples, namely
  $(X_1,X_2,X_5)$, $(X_3,X_2,X_5)$, $(X_4,X_5,X_2)$, and $(X_6,X_5,X_2)$. Thus, we
  then apply the orientation rules in Step 3 of the algorithm, starting
  with rule R1. If one of the two unshielded triples $(X_1,X_2,X_5)$ or
  $(X_3,X_2,X_5)$ is considered first, we obtain $X_2 \rightarrow X_5$. On
  the other hand, if one of the unshielded triples $(X_4,X_5,X_2)$ or
  $(X_6,X_5,X_2)$ is considered first, then we obtain $X_2 \leftarrow
  X_5$. Note that we have no issues with overwriting of edges here, since as
  soon as the edge $X_2-X_5$ is oriented, all edges are oriented and no
  further orientation rules are applied.

  These examples illustrate that Steps 2 and 3 of the PC-algorithm can be
  order-dependent regardless of the output of the previous steps.
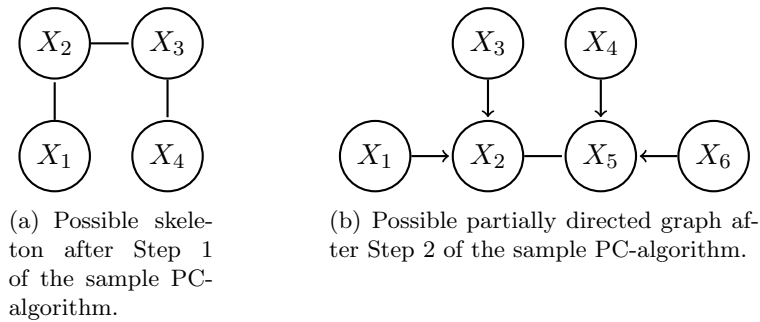
\begin{figure}[!h]\centering%
    \subfigure[Possible skeleton after Step 1 of the sample PC-algorithm.]{
      \begin{tikzpicture}
        \tikzstyle{every circle node}=
        [%
        draw=black!20!black,%
        minimum size=7mm,%
        circle,%
        thick%
        ]
        \tikzstyle{every rectangle node}=
        [%
        draw=black!20!black,%
        minimum size=6mm,%
        rectangle,%
        thick%
        ]

        \node[circle] (A) at (0, 0) {$X_{1}$};
        \node[circle] (B) at (0, 1.5) {$X_{2}$};
        \node[circle] (C) at (1.5, 1.5) {$X_{3}$};
        \node[circle] (D) at (1.5, 0) {$X_{4}$};

        \path [thick,shorten >=1pt,-stealth']
        (A) edge[-] (B)
        (B) edge[-] (C)
        (C) edge[-] (D);

      \end{tikzpicture}\label{ex.vstruct}
    }\qquad \qquad %
    \subfigure[Possible partially directed graph after Step 2 of the sample PC-algorithm.]{
      \begin{tikzpicture}
        \tikzstyle{every circle node}=
        [%
        draw=black!20!black,%
        minimum size=7mm,%
        circle,%
        thick%
        ]
        \tikzstyle{every rectangle node}=
        [%
        draw=black!20!black,%
        minimum size=7mm,%
        rectangle,%
        thick%
        ]

        \node[circle] (A) at (0, 0) {$X_{2}$};
        \node[circle] (B) at (0, 1.5) {$X_{3}$};
        \node[circle] (C) at (1.5, 1.5) {$X_{4}$};
        \node[circle] (D) at (1.5, 0) {$X_{5}$};
        \node[circle] (E) at (-1.5, 0) {$X_{1}$};
        \node[circle] (F) at (3, 0) {$X_{6}$};

        \path [thick,shorten >=1pt,-stealth']
        (A) edge[-] (D)
        (E) edge[->] (A)
        (B) edge[->] (A)
        (C) edge[->] (D)
        (F) edge[->] (D);

      \end{tikzpicture}\label{ex.R1}
    }
    \caption{Graphs corresponding to Examples \ref{ex.orientation} and
      \ref{ex.orientation.new}.
    }
    \label{ex3.graphs}
  \end{figure}
\end{example}

\section{Modified algorithms}\label{sec.modified.pc}

We now propose several modifications of the PC-algorithm (and hence also of
the related algorithms) that remove the order-dependence in the various
stages of the
algorithm. Sections \ref{sec.new.skeleton}, \ref{sec.new.vstruct},
and \ref{sec.new.orientations} discuss the skeleton,
the v-structures and the orientation rules, respectively. In each of these
sections, we first describe the oracle version of the modifications, and
then results and examples about order-dependence in the corresponding
sample version (obtained by replacing conditional independence queries by
conditional independence tests, as in
Section \ref{sec.order-dep.sample.old}). Finally, Section
\ref{sec.new.other.algos} discusses order-independent versions of related
algorithms like RFCI and FCI, and Section \ref{sec.consistency} presents
high-dimensional consistency results for the sample versions of all
modifications.

\subsection{The skeleton}\label{sec.new.skeleton}

We first consider estimation of the skeleton in Step 1 of the
PC-algorithm. The pseudocode for our modification is given in Algorithm
\ref{pseudo.new.pc}. The resulting PC-algorithm, where Step 1 in Algorithm
\ref{pseudo.pc} is replaced by Algorithm \ref{pseudo.new.pc}, is called
``PC-stable".

The main difference between Algorithms \ref{pseudo.old.pc} and
\ref{pseudo.new.pc} is given by the for-loop on lines 6-8 in the latter
one, which computes and stores the adjacency sets $a(X_i)$ of all
variables after each new size $\ell$ of the conditioning sets. These stored
adjacency sets $a(X_i)$ are used whenever we search for conditioning
sets of this given size $\ell$. Consequently, an edge deletion on line 13
no longer affects which conditional independencies are checked for other
pairs of variables at this level of $\ell$.

In other words, at each level of $\ell$, Algorithm \ref{pseudo.new.pc}
records which edges should be removed, but for the purpose of the adjacency
sets it removes these edges only when it goes to the next value of
$\ell$.
Besides resolving the order-dependence in the estimation of the skeleton,
our algorithm has the advantage that it is easily parallelizable at each
level of $\ell$.

\begin{algorithm}[h]
\caption{Step 1 of the PC-stable algorithm (oracle version)}
\label{pseudo.new.pc}
\begin{algorithmic}[1]
   \REQUIRE Conditional independence information among all variables in
   $\mathbf{V}$, and an ordering order$(\mathbf{V})$ on the variables
   \STATE Form the complete undirected graph $\mathcal{C}$ on the vertex set
   $\mathbf{V}$
   \STATE Let $\ell = -1$;
   \REPEAT
   \STATE Let $\ell = \ell + 1$;
   \FORALL{vertices $X_i$ in $\mathcal{C}$}
   \STATE Let $a(X_i) = \text{adj}(\mathcal{C},X_i)$
   \ENDFOR
   \REPEAT
   \STATE Select a (new) ordered pair of vertices $(X_i,X_j)$ that are
   adjacent in $\mathcal{C}$ and satisfy $|a(X_i) \setminus \{X_j\}|
   \geq \ell$, using order$(\mathbf{V})$;
   \REPEAT
   \STATE Choose a (new) set $\mathbf{S} \subseteq a(X_i) \setminus
   \{X_j\}$ with $|\mathbf{S}|=\ell$, using order$(\mathbf{V})$;
   \IF{$X_i$ and $X_j$ are conditionally independent given $\mathbf{S}$}
   \STATE Delete edge $X_i - X_j$ from $\mathcal{C}$;
   \STATE Let $\text{sepset}(X_i,X_j) = \text{sepset}(X_j,X_i) = \mathbf{S}$;
   \ENDIF
   \UNTIL{$X_i$ and $X_j$ are no longer adjacent in $\mathcal{C}$ or all
     $\mathbf{S} \subseteq a(X_i) \setminus \{X_j\}$
     with $|\mathbf{S}|=\ell$ have been considered}
   \UNTIL{all ordered pairs of adjacent vertices $(X_i,X_j)$ in
     $\mathcal{C}$ with $|a(X_i) \setminus \{X_j\}|
     \geq \ell$ have been considered}
   \UNTIL{all pairs of adjacent vertices $(X_i,X_j)$ in $\mathcal{C}$
     satisfy $|a(X_i) \setminus \{X_j\}| \leq \ell$}
   \RETURN{$\mathcal{C}$, sepset}.
 \end{algorithmic}
\end{algorithm}

The PC-stable algorithm is sound and complete in the oracle version (Theorem
\ref{theo.corr.pc}), and yields order-independent skeletons in the sample
version (Theorem \ref{theo.indep}).
We illustrate the algorithm in Example \ref{ex.orderindep.new}.

\begin{theorem}
\label{theo.corr.pc}
  Let the distribution of $\mathbf{V}$ be faithful to a DAG $\mathcal
  G=(\mathbf{V},\mathbf{E})$, and assume that we are given perfect
  conditional independence information about all pairs of variables
  $(X_i,X_j)$ in $\mathbf{V}$ given subsets $\mathbf{S} \subseteq
  \mathbf{V} \setminus \{X_i,X_j\}$. Then the output of the PC-stable
  algorithm is the CPDAG that represents $\mathcal G$.
\end{theorem}
\begin{proof}
  The proof of Theorem \ref{theo.corr.pc} is completely analogous to the proof of
  Theorem \ref{th.correct.pc.original} for the original PC-algorithm, as
  discussed in Section \ref{sec.oracle.orig.pc}.
\end{proof}

\begin{theorem}\label{theo.indep}
   The skeleton resulting from the sample version of the PC-stable
   algorithm is order-independent.
\end{theorem}
\begin{proof}
  We consider the removal or retention of an arbitrary edge $X_i-X_j$ at
  some level $\ell$.   The ordering of the variables determines
  the order in which the edges (line 9 of Algorithm \ref{pseudo.new.pc})
  and the subsets $\mathbf{S}$ of $a(X_i)$ and $a(X_j)$ (line 11 of
  Algorithm \ref{pseudo.new.pc}) are considered. By construction, however,
  the order in which edges are considered does not affect the sets $a(X_i)$ and
  $a(X_j)$.

  If there is at least one subset $\mathbf{S}$ of $a(X_i)$ or $a(X_j)$ such that
  $X_i\ci X_j |\mathbf{S}$, then any ordering of the variables will find
  a separating set for $X_i$ and $X_j$ (but different orderings may lead to
  different separating sets as illustrated in Example
  \ref{ex.sepset}). Conversely, if there is no subset $\mathbf{S'}$ of
  $a(X_i)$ or $a(X_j)$ such that $X_i\ci X_j |\mathbf{S'}$, then no
  ordering will find a separating set.

  Hence, any ordering of the variables
  leads to the same edge deletions, and therefore to the same
  skeleton.
\end{proof}

\begin{example}\label{ex.orderindep.new}
  (Order-independent skeletons) We go back to Example
  \ref{ex.orderdep.old}, and consider the sample version of Algorithm
  \ref{pseudo.new.pc}. The algorithm now outputs the skeleton shown in
  Figure \ref{pc.old.output1} for both orderings
  $\text{order}_1(\mathbf{V})$ and $\text{order}_2(\mathbf{V})$.

  We again go through the algorithm step by step. We start with a
  complete undirected graph on $\mathbf{V}$. The only conditional
  independence found when $\ell=0$ or $\ell=1$ is $X_1 \ci X_2$, and the
  corresponding edge is removed.
  When $\ell=2$, the algorithm first computes the new adjacency sets:
  $\text{a}(X_1)=a(X_2)=\{X_3,X_4,X_5\}$ and
  $\text{a}(X_i)=\mathbf{V}\setminus \{X_i\}$ for $i=3,4,5$. There
  are two pairs of variables that are thought to be conditionally
  independent given a subset of size 2, namely $(X_2,X_4)$ and
  $(X_3,X_4)$. Since the sets $\text{a}(X_i)$ are not updated after edge
  removals, it does not matter in which order we consider the ordered pairs
  $(X_2,X_4)$, $(X_4,X_2)$, $(X_3,X_4)$ and $(X_4,X_3)$. Any ordering leads
  to the removal of both edges, as the separating set $\{X_1,X_3\}$ for
  $(X_4,X_2)$ is contained in $\text{a}(X_4)$, and the separating set
  $\{X_1,X_5\}$ for $(X_3,X_4)$ is contained in $\text{a}(X_3)$ (and in
  $\text{a}(X_4)$).
\end{example}

\subsection{Determination of the v-structures}\label{sec.new.vstruct}

We propose two methods to resolve the order-dependence in the determination
of the v-structures, using the conservative PC-algorithm (CPC) of
\cite{RamseyZhangSpirtes06} and a variation thereof.

The CPC-algorithm works as follows. Let $\mathcal C$ be the graph resulting
from Step 1 of the PC-algorithm (Algorithm \ref{pseudo.pc}). For all
unshielded triples $(X_i,X_j,X_k)$ in $\mathcal C$, determine all subsets
$\mathbf{Y}$ of $\text{adj}(\mathcal{C},X_i)$ and of
$\text{adj}(\mathcal{C},X_k)$ that make $X_i$ and $X_k$ conditionally
independent, i.e., that satisfy $X_i \ci X_k | \mathbf{Y}$. We refer to
such sets as separating sets. The triple $(X_i,X_j,X_k)$ is labelled as
\textit{unambiguous} if at least one such separating set is
found and either $X_j$ is in all separating sets or in none of them;
otherwise it is labelled as \textit{ambiguous}. If the triple is
unambiguous, it is oriented  as v-structure if and only if $X_j$ is in none
of the separating sets. Moreover, in Step 3 of the PC-algorithm (Algorithm
\ref{pseudo.pc}), the orientation rules are adapted so that only
unambiguous triples are oriented. We refer to the combination of PC-stable
and CPC as the CPC-stable algorithm.

We found that the CPC-algorithm can be very conservative, in the sense that
very few unshielded triples are unambiguous in the sample version. We
therefore propose a minor modification of this approach, called majority
rule PC-algorithm (MPC). As in CPC, we first determine all subsets
$\mathbf{Y}$ of $\text{adj}(\mathcal{C},X_i)$ and of
$\text{adj}(\mathcal{C},X_k)$ satisfying $X_i \ci X_k | \mathbf{Y}$. We
then label the triple $(X_i,X_j,X_k)$ as \textit{unambiguous} if at least
one such separating set is found and $X_j$ is not in exactly $50\%$ of the
separating sets. Otherwise it is labelled as \textit{ambiguous}. If the
triple is unambiguous, it is oriented as v-structure if and only if $X_j$
is in less than half of the separating sets. As in CPC, the orientation
rules in Step 3 are adapted so that only unambiguous triples are
oriented. We refer to the combination of PC-stable and MPC as the
MPC-stable algorithm.

Theorem \ref{theo.corr.pc.cons.maj} states that the oracle CPC- and
MPC-stable algorithms are sound and complete. When looking at the sample
versions of the algorithms, we note that any unshielded triple that is
judged to be unambiguous in CPC-stable is also unambiguous in MPC-stable,
and any unambiguous v-structure in CPC-stable is an unambiguous v-structure
in MPC-stable. In this sense, CPC-stable is more conservative than
MPC-stable, although the difference appears to be small in simulations and
for the yeast data (see Sections \ref{sec.simulation.results}
and \ref{sec.simulation.realdata}). Both CPC-stable and MPC-stable share
the property that the determination of v-structures no longer depends on
the (order-dependent) separating sets that were found in Step 1 of the
algorithm. Therefore, both CPC-stable and MPC-stable yield
order-independent decisions about v-structures in the sample version, as
stated in Theorem \ref{theo.indep.sepsets}. Example \ref{ex.sepset.new}
illustrates both algorithms.

We note that the CPC/MPC-stable algorithms may yield a lot fewer directed
edges than PC-stable. On the other hand, we can put more trust in those
edges that were oriented.

\begin{theorem}\label{theo.corr.pc.cons.maj}
  Let the distribution of $\mathbf{V}$ be faithful to a DAG $\mathcal
  G=(\mathbf{V},\mathbf{E})$, and assume that we are given perfect
  conditional independence information about all pairs of variables
  $(X_i,X_j)$ in $\mathbf{V}$ given subsets $\mathbf{S} \subseteq
  \mathbf{V} \setminus \{X_i,X_j\}$. Then the output of the
  CPC/MPC(-stable) algorithms is the CPDAG that represents $\mathcal G$.
\end{theorem}
\begin{proof}
   The skeleton of the CPDAG is correct by Theorems
   \ref{th.correct.pc.original} and \ref{theo.corr.pc}. The
   unshielded triples are all unambiguous (in the conservative and the
   majority rule versions), since for
   any unshielded triple $(X_i,X_j,X_k)$ in a DAG, $X_j$ is either in all
   sets that d-separate $X_i$ and $X_k$ or in none of them \citep[Lemma
   5.1.3]{SpirtesEtAl00}. In particular, this also means that all
   v-structures are determined correctly. Finally, since all unshielded
   triples are unambiguous, the orientation rules are as in the original
   oracle PC-algorithm, and soundness and completeness of these rules
   follows from \cite{Meek95} and \cite{AndersonEtAll97}.
\end{proof}

\begin{theorem}\label{theo.indep.sepsets}
   The decisions about v-structures in the sample versions of the
   CPC/MPC-stable algorithms are order-independent.
\end{theorem}
\begin{proof}
  The CPC/MPC-stable algorithms have order-independent skeletons in Step 1,
  by Theorem \ref{theo.indep}. In particular, this means that their
  unshielded triples and adjacency sets are order-independent. The decision
  about whether an unshielded triple is unambiguous and/or a v-structure is
  based on the adjacency sets of nodes in the triple, which are
  order-independent.
\end{proof}

\begin{example}\label{ex.sepset.new}
  (Order-independent decisions about v-structures) We consider
  the sample versions of the CPC/MPC-stable algorithms, using the same
  input as in Example \ref{ex.sepset}. In particular, we assume that all
  conditional independencies induced by the DAG in Figure \ref{ex.true.dag} are
  judged to hold, plus the additional (erroneous) conditional independency
  $X_1\ci X_3|X_4$.

  Denote the skeleton after Step 1 by $\mathcal C$. We consider the
  unshielded triple $(X_1,X_2, X_3)$. First, we compute adj$(\mathcal
  C,X_1) = \{X_2,X_5\}$ and adj$(\mathcal C,X_3) = \{X_2,X_4\}$. We now
  consider all subsets $\mathbf{Y}$ of these adjacency sets, and check
  whether $X_1 \ci X_3 | \mathbf{Y}$. The following separating sets are
  found: $\{X_2\}$, $\{X_4\}$, and $\{X_2,X_4\}$.

  Since $X_2$ is in some but not all of these separating sets, CPC-stable
  determines that the triple is ambiguous, and no orientations are
  performed. Since $X_2$ is in more than half of the separating sets,
  MPC-stable determines that the triple is unambiguous and not a
  v-structure. The output of both algorithms is given in Figure
  \ref{ex.true.skelet}.
\end{example}

\subsection{Orientation rules}\label{sec.new.orientations}

Even when the skeleton and the determination of the v-structures are
order-independent, Example \ref{ex.orientation} showed that there might be
some order-dependence left in the sample-version. This can be resolved by
allowing bi-directed edges ($\leftrightarrow$) and working with lists
containing the candidate edges for the v-structures in Step 2 and the
orientation rules R1-R3 in Step 3.

In particular, in Step 2 we generate a list of all (unambiguous)
v-structures, and then orient all of these, creating a bi-directed edge in
case of a conflict between two v-structures. In Step 3, we first generate a list
of all edges that can be oriented by rule R1. We orient all these edges,
again creating bi-directed edges if there are conflicts. We do the same for
rules R2 and R3, and iterate this procedure until no more edges can be
oriented.

When using this procedure, we add the letter L (standing for lists), e.g.,
LCPC-stable and LMPC-stable. The LCPC-stable and
LMPC-stable algorithms are correct in the oracle version (Theorem
\ref{theo.correct.bcpc}) and fully order-independent in the sample versions
(Theorem \ref{theo.fully.orderindep}). The procedure is illustrated
in Example \ref{ex.orientation.new}.

We note that the bi-directed edges cannot be interpreted causally. They
simply indicate that there was some conflicting information in the
algorithm.

\begin{theorem}\label{theo.correct.bcpc}
  Let the distribution of $\mathbf{V}$ be faithful to a DAG $\mathcal
  G=(\mathbf{V},\mathbf{E})$, and assume that we are given perfect
  conditional independence information about all pairs of variables
  $(X_i,X_j)$ in $\mathbf{V}$ given subsets $\mathbf{S} \subseteq
  \mathbf{V} \setminus   \{X_i,X_j\}$. Then the (L)CPC(-stable) and
  (L)MPC(-stable) algorithms output the CPDAG that represents $\mathcal G$.
\end{theorem}

\begin{proof}\label{theo.sample.bcpc.orderindep}
  By Theorem \ref{theo.corr.pc.cons.maj}, we know that the CPC(-stable) and
  MPC(-stable) algorithms are correct. With perfect conditional
  independence information, there are no conflicts between v-structures in
  Step 2 of the algorithms, nor between orientation rules in Step 3 of the
  algorithms. Therefore, the (L)CPC(-stable) and (L)MPC(-stable) algorithms
  are identical to the CPC(-stable) and MPC(-stable) algorithms.
\end{proof}

\begin{theorem}\label{theo.fully.orderindep}
   The sample versions of LCPC-stable and LMPC-stable are fully
   order-indepen\-dent.
\end{theorem}
\begin{proof}
   This follows straightforwardly from Theorems \ref{theo.indep} and
   \ref{theo.indep.sepsets} and the procedure with lists and bi-directed
   edges discussed above.
\end{proof}

Table \ref{table.modifications} summarizes the three order-dependence
issues explained above and the corresponding modifications of the
PC-algorithm that removes the given order-dependence problem.

\begin{table}[h]
\centering
\begin{tabular}{lccc}
  \cline{2-4}
  & \multicolumn{1}{|c|}{skeleton} &
   \multicolumn{1}{|c|}{v-structures decisions} &
   \multicolumn{1}{|c|}{edges orientations} \\ \cline{1-4}
  \multicolumn{1}{|l|}{PC} & \multicolumn{1}{|c|}{-} &
  \multicolumn{1}{|c|}{-} &\multicolumn{1}{|c|}{-}\\ \cline{1-4}
  \multicolumn{1}{|l|}{PC-stable} & \multicolumn{1}{|c|}{$\surd$} &
  \multicolumn{1}{|c|}{-} &\multicolumn{1}{|c|}{-}\\ \cline{1-4}
  \multicolumn{1}{|l|}{CPC/MPC-stable} & \multicolumn{1}{|c|}{$\surd$} &
  \multicolumn{1}{|c|}{$\surd$} &\multicolumn{1}{|c|}{-}\\ \cline{1-4}
  \multicolumn{1}{|l|}{BCPC/BMPC-stable} & \multicolumn{1}{|c|}{$\surd$} &
  \multicolumn{1}{|c|}{$\surd$} &\multicolumn{1}{|c|}{$\surd$}\\ \cline{1-4}

\end{tabular}
\caption{Order-dependence issues and corresponding modifications of the
  PC-algorithm that remove the problem. A tick mark indicates that the
  corresponding aspect of the graph is estimated order-independently in the
  sample version. For example, with PC-stable the skeleton is estimated
  order-independently but not the v-structures and the edge orientations.}
  \label{table.modifications}
\end{table}



\begin{example}\label{ex.orientation.new}
  First, we consider the two unshielded
  triples $(X_1,X_2,X_3)$ and $(X_2,X_3,X_4)$ as shown
  in Figure \ref{ex.vstruct}. The version of the algorithm that uses lists
  for the orientation rules, orients these edges as $X_1 \to X_2
  \leftrightarrow X_3 \leftarrow X_4$, regardless of the ordering of the
  variables.

  Next, we consider the structure shown in Figure \ref{ex.R1}. As a first
  step, we construct a list containing all candidate structures eligible
  for orientation rule R1 in Step 3. The list contains the unshielded
  triples $(X_1,X_2,X_5)$, $(X_3,X_2,X_5)$, $(X_4,X_5,X_2)$, and
  $(X_6,X_5,X_2)$. Now, we go through each element in the list
  and we orient the edges accordingly, allowing bi-directed edges. This
  yields the edge orientation $X_2 \leftrightarrow X_5$, regardless of the
  ordering of the variables.
\end{example}

\subsection{Related algorithms}\label{sec.new.other.algos}


The FCI-algorithm (\cite{SpirtesEtAl00,SpirtesMeekRichardson99}) first runs
Steps 1 and 2 of the PC-algorithm (Algorithm \ref{pseudo.pc}). Based on the
resulting graph, it then computes certain sets, called ``Possible-D-SEP''
sets, and conducts more conditional independence tests given subsets of the
Possible-D-SEP sets. This can lead to additional edge removals and corresponding
separating sets. After this, the v-structures are newly
determined. Finally, there are ten orientation rules as defined by
\cite{Zhang08-orientation-rules}.

From our results, it immediately follows that FCI with any of our
modifications of the PC-algorithm is sound and complete in the oracle
version. Moreover, we can easily construct partially or fully
order-independent sample versions as follows. To solve the order-dependence
in the skeleton we can use the following three step approach. First, we use
PC-stable to find an initial order-independent
skeleton. Next, since Possible-D-SEP sets are determined from
the orientations of the v-structures, we need order-independent
v-structures. Therefore, in Step 2 we can determine the v-structures using
CPC. Finally, we compute the Possible-D-SEP sets for all pairs of nodes at
once, and do not update these after possible edge removals. The
modification that uses these three steps returns an order-independent
skeleton, and we call it FCI-stable. To assess order-independent
v-structures in the final output, one should again use an
order-independent procedure, as in CPC or MPC for the second time that
v-structures are determined. We call these modifications CFCI-stable and
MFCI-stable, respectively. Regarding the orientation rules, we have that
the FCI-algorithm does not suffer from conflicting v-structures, as shown
in Figure \ref{ex.vstruct} for the PC-algorithm, because it orients edge
\emph{marks} and because bi-directed edges are allowed. However, the ten
orientation rules still suffer from order-dependence issues as in the
PC-algorithm, as in Figure \ref{ex.R1}. To solve this problem, we can again
use lists of candidate edges for each orientation rule as explained in the
previous section about the PC-algorithm. However, since these ten
orientation rules are more involved than the three for PC, using lists can
be very slow for some rules, for example the one for discriminating
paths. We refer to these
modifications as LCFCI-stable and LMFCI-stable, and they are fully
order-independent in the sample version.

Table \ref{table.modifications.fci} summarizes the three order-dependence
issues for FCI and the corresponding modifications that remove them.

\begin{table}[h]
\centering
\begin{tabular}{lccc}
  \cline{2-4}
  & \multicolumn{1}{|c|}{skeleton} &
   \multicolumn{1}{|c|}{v-structures decisions} &
   \multicolumn{1}{|c|}{edges orientations} \\ \cline{1-4}
  \multicolumn{1}{|l|}{FCI} & \multicolumn{1}{|c|}{-} &
  \multicolumn{1}{|c|}{-} &\multicolumn{1}{|c|}{-}\\ \cline{1-4}
  \multicolumn{1}{|l|}{FCI-stable} & \multicolumn{1}{|c|}{$\surd$} &
  \multicolumn{1}{|c|}{-} &\multicolumn{1}{|c|}{-}\\ \cline{1-4}
  \multicolumn{1}{|l|}{CFCI/MFCI-stable} & \multicolumn{1}{|c|}{$\surd$} &
  \multicolumn{1}{|c|}{$\surd$} &\multicolumn{1}{|c|}{-}\\ \cline{1-4}
  \multicolumn{1}{|l|}{LCFCI/LMFCI-stable} & \multicolumn{1}{|c|}{$\surd$} &
  \multicolumn{1}{|c|}{$\surd$} &\multicolumn{1}{|c|}{$\surd$}\\ \cline{1-4}

\end{tabular}
\caption{Order-dependence issues and corresponding modifications of the
  FCI-algorithm that remove the problem. A tick mark indicates that the
  corresponding aspect of the graph is estimated order-independently in the
  sample version. For example, with FCI-stable the skeleton is estimated
  order-independently but not the v-structures and the edge orientations.}
  \label{table.modifications.fci}
\end{table}

The RFCI-algorithm \citep{CoMaKaRi2012} can be viewed as an algorithm that
is in between PC and FCI, in the sense that its
computational complexity is of the same order as PC, but its output can be
interpreted causally without assuming causal sufficiency (but is slightly less
informative than the output from FCI).

RFCI works as follows. It runs the first step of PC. It then has a more
involved Step 2 to determine the v-structures \citep[Lemma
3.1]{CoMaKaRi2012}. In particular, for any unshielded triple
$(X_i,X_j,X_k)$, it conducts additional tests to check if both $X_i$ and
$X_j$ and $X_j$ and $X_k$ are conditionally dependent given
$\text{sepset}(X_i,X_j)\setminus\{X_j\}$ found in Step 1. If a conditional
independence relationship is detected, the corresponding edge is removed
and a minimal separating set is stored. The removal of an edge can create
new unshielded triples or destroy some of them. Therefore, the algorithm
works with lists to make sure that these actions are order-independent. On
the other hand, if both conditional dependencies hold and $X_j$ is not in
the separating set for $(X_i,X_k)$, the triple is oriented as a
v-structure. Finally, in Step 3 it uses the ten orientation rules of
\cite{Zhang08-orientation-rules} with a modified orientation rule for the
discriminating paths, that also involves some additional conditional
independence tests.

From our results, it immediately follows that RFCI with any of our
modifications of the PC-algorithm is correct in the oracle
version. Because of its more involved rules for v-structures and
discriminating paths, one needs to make several adaptations to create a
fully order-independent algorithm. For example, the additional conditional
independence tests conducted for the v-structures are based on the
separating sets found in Step 1. As already mentioned before (see Example
\ref{ex.sepset}) these separating sets are order-dependent, and therefore
also the possible edge deletions based on them are order-dependent, leading
to an order-dependent skeleton. To produce an order-independent skeleton
one should use a similar approach to the conservative one for the
v-structures to make the additional edge removals
order-independent. Nevertheless, we can remove a large amount of the
order-dependence in the skeleton by using the stable version for the
skeleton as a first step. We refer to this modification as
RFCI-stable. Note that this procedure does not produce a fully
order-independent skeleton, but as shown in Section \ref{sec.sim.skeleton}, it
reduces the order-dependence considerably. Moreover, we can combine this
modification with CPC or MPC on the final skeleton to reduce the
order-dependence of the v-structures. We refer to these modifications as
CRFCI-stable and MRFCI-stable. Finally, we can again use lists for the
orientation rules as in the FCI-algorithm to reduce the order-dependence
caused by the orientation rules.


The CCD-algorithm \citep{Richardson96} can also be made order-independent
using similar approaches.

\subsection{High-dimensional consistency}\label{sec.consistency}

The original PC-algorithm has been shown to be consistent for certain
sparse high-dimensio\-nal graphs. In particular, \cite{KalischBuehlmann07a}
proved consistency for multivariate Gaussian distributions.
More recently, \cite{HarrisDrton12} showed
consistency for the broader class of Gaussian copulas when using rank
correlations, under slightly different conditions.

These high-dimensional consistency results allow the DAG $\mathcal{G}$ and
the number of observed variables $p$ in $\mathbf{V}$ to grow as a function of
the sample size, so that $p=p_n$, $\mathbf{V}=\mathbf{V}_n =
(X_{n,1},\dots,X_{n,p_n})$ and  $\mathcal{G}=\mathcal{G}_n$. The
corresponding CPDAGs that represent $\mathcal G_n$ are denoted by $\mathcal
C_n$, and the estimated CPDAGs using tuning parameter $\alpha_n$ are denoted
by $\hat{\mathcal {C}}_n(\alpha_n)$. Then the consistency results say that,
under some conditions, there exists a sequence $\alpha_n$ such that
$P(\hat{\mathcal {C}}_n(\alpha_n) = \mathcal C_n) \to 1$ as $n\to \infty$.

These consistency results rely on the fact that the PC-algorithm only performs
conditional independence tests between pairs of variables given subsets
$\mathbf{S}$ of size less than or equal to the degree of the graph (when no
errors are made). We made sure that our modifications still obey this
property, and therefore the consistency results of \cite{KalischBuehlmann07a} and
\cite{HarrisDrton12} remain valid for the (L)CPC(-stable) and
(L)MPC(-stable) algorithms, under exactly the same conditions as for the original
PC-algorithm.

Finally, also the consistency results of \cite{CoMaKaRi2012}
for the FCI- and RFCI-algorithms remain valid for the (L)CFCI(-stable),
(L)MFCI(-stable), CRFCI(-stable), and MRFCI(-stable)
algorithms, under exactly the same conditions as for the original FCI- and
RFCI-algorithms.

\section{Simulations}\label{sec.simulation.results}

We compared all algorithms using simulated data from low-dimensional and
high-dimensional systems  with and without latent variables.
In the low-dimensional setting, we compared the modifications of PC, FCI
and RFCI. All algorithms performed similarly in this setting, and the
results are presented in Appendix \ref{app.sim.results.low}.
The remainder of this section therefore focuses on the high-dimensional
setting, where we compared (L)PC(-stable), (L)CPC(-stable) and
(L)MPC(-stable) in systems without latent variables, and RFCI(-stable),
CRFCI(-stable) and MRFCI(-stable) in systems with latent variables. We
omitted the FCI-algorithm and the modifications with lists for the
orientation rules of RFCI because of their computational complexity. Our
results show that our modified algorithms perform better than the original
algorithms in the high-dimensional settings we considered.

In Section \ref{sec.sim.setup} we describe the simulation
setup. Section \ref{sec.sim.skeleton} evaluates the estimation of the
skeleton of the CPDAG or PAG (i.e., only looking at the presence or absence
of edges), and Section \ref{sec.sim.outputs} evaluates the estimation of
the CPDAG or PAG (i.e., also including the edge marks). Appendix
\ref{app.sim.time.pc} compares the computing time and the number of
conditional independence tests performed by PC and PC-stable, showing that
PC-stable generally performs more conditional independence tests, and is
slightly slower than PC. Finally, Appendix \ref{app.sim.results.equal}
compares the modifications of FCI and RFCI in two medium-dimensional
settings with latent variables, where the number of nodes in the graph is
roughly equal to the sample size and we allow somewhat denser graphs. The
results indicate that also in this setting our modified versions perform
better than the original ones.


\subsection{Simulation set-up}\label{sec.sim.setup}

We used the following procedure to generate a random weighted DAG with a
given number of vertices $p$ and an expected neighborhood size
$E(N)$. First, we generated a random adjacency matrix $A$ with independent
realizations of Bernoulli$(E(N)/(p-1))$ random variables in the lower
triangle of the matrix and zeroes in the remaining entries. Next, we
replaced the ones in $A$ by independent realizations of a
Uniform($[0.1,1]$) random variable, where a nonzero entry $A_{ij}$ can be
interpreted as an edge from $X_j$ to $X_i$ with weight $A_{ij}$. 
(We bounded the edge weights away from zero to avoid problems with near-unfaithfulness.)

We related a multivariate Gaussian distribution to each DAG by letting $X_1
= \epsilon_1$ and $X_i = \sum_{r=1}^{i-1}A_{ir}X_r + \epsilon_i$ for
$i=2,\dots,p$, where $\epsilon_1,\dots,\epsilon_{p}$ are mutually independent
$\mathcal{N}(0,1)$ random variables. The variables $X_1,\dots, X_{p}$
then have a multivariate Gaussian distribution with mean zero and
covariance matrix $\Sigma = (\mathbf{1}-A)^{-1}(\mathbf{1}-A)^{-T}$, where
$\mathbf{1}$ is the $p \times p$ identity matrix.

We generated 250 random weighted DAGs with $p=1000$ and $E(N)=2$, and for
each weighted DAG we generated an i.i.d.\ sample of size $n=50$. In the
setting without latents, we simply used all variables. In the setting with
latents, we removed half of the variables that had no parents and at least two
children, chosen at random.

We estimated each graph for 20 random orderings
of the variables, using the sample versions of (L)PC(-stable),
(L)CPC(-stable), and (L)MPC(-stable) in the setting without latents, and
using the sample versions of RFCI(-stable), CRFCI(-stable), and
MRFCI(-stable) in the setting with latents, using levels $\alpha \in
\{0.000625, 0.00125, 0.0025, 0.005, 0.01, 0.02, 0.04\}$ for the
partial correlation tests. Thus, from each randomly generated DAG, we
obtained 20 estimated CPDAGs or PAGs from each algorithm, for each value of
$\alpha$.


\subsection{Estimation of the skeleton}\label{sec.sim.skeleton}

Figure \ref{fig.sim.skelet} shows the number of edges, the number of
errors, and the true discovery rate for the estimated skeletons. The figure
only compares PC and PC-stable in the setting without latent variables, and
RFCI and RFCI-stable in the setting with latent variables, since the
modifications for the v-structures and the orientation rules do not affect
the estimation of the skeleton.

\begin{figure}[h]
\centering
\includegraphics[scale=0.80]{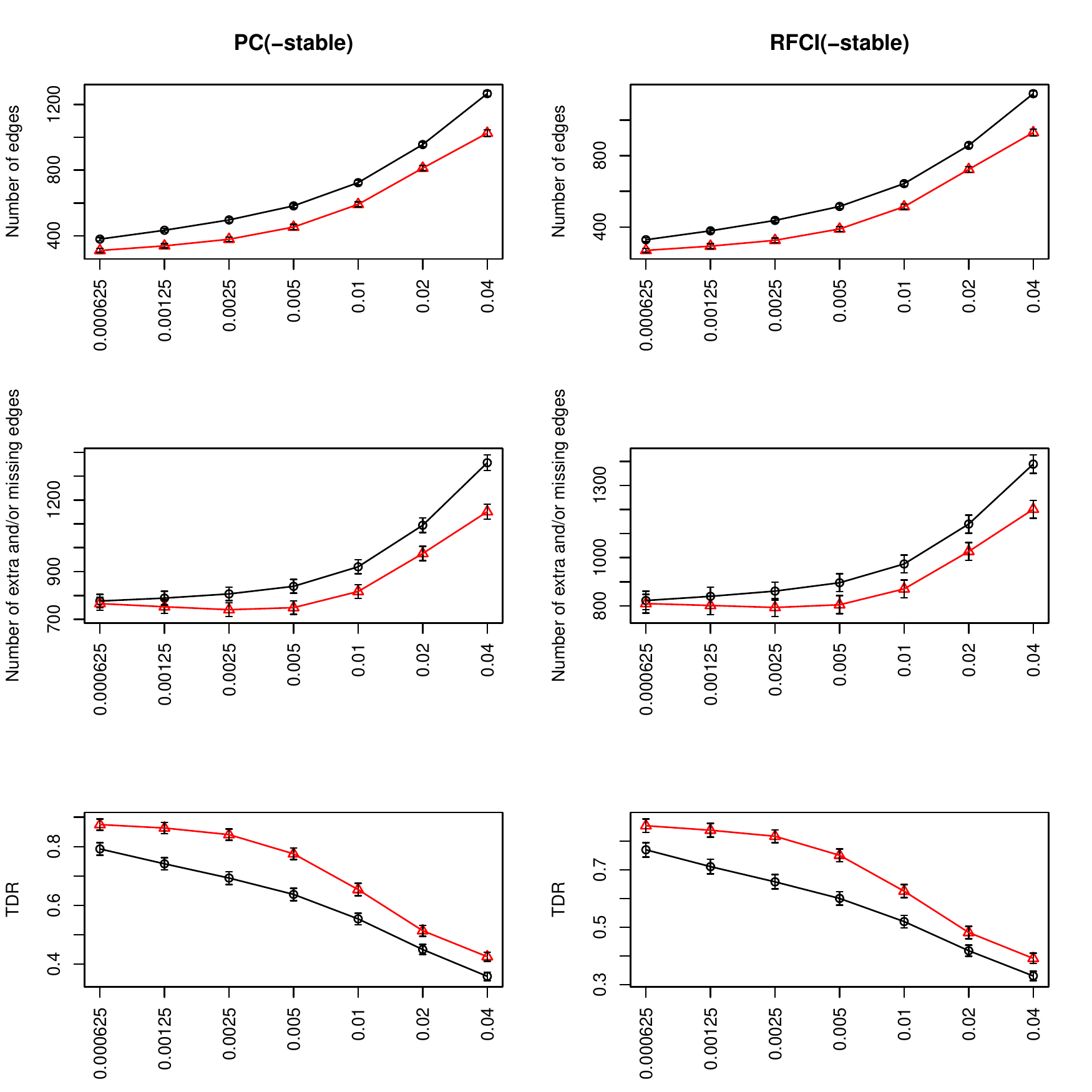}
\caption{Estimation performance of PC (circles; black line) and PC-stable
  (triangles; red line) for the skeleton of the CPDAGs (first column of
  plots), and of RFCI (circles; black line) and RFCI-stable
  (triangles; red line) for the skeleton of the PAGs (second column of plots),
  for different values of $\alpha$ ($x$-axis displayed in $\log$ scale). The
  results are shown as averages plus or minus one standard deviation,
  computed over 250 randomly generated graphs and 20 random variable
  orderings per graph.}
  \label{fig.sim.skelet}
\end{figure}

\begin{figure}[!bh]\centering%
  \subfigure[$\alpha=0.00125$]{
     \includegraphics[scale=0.3,angle=0]{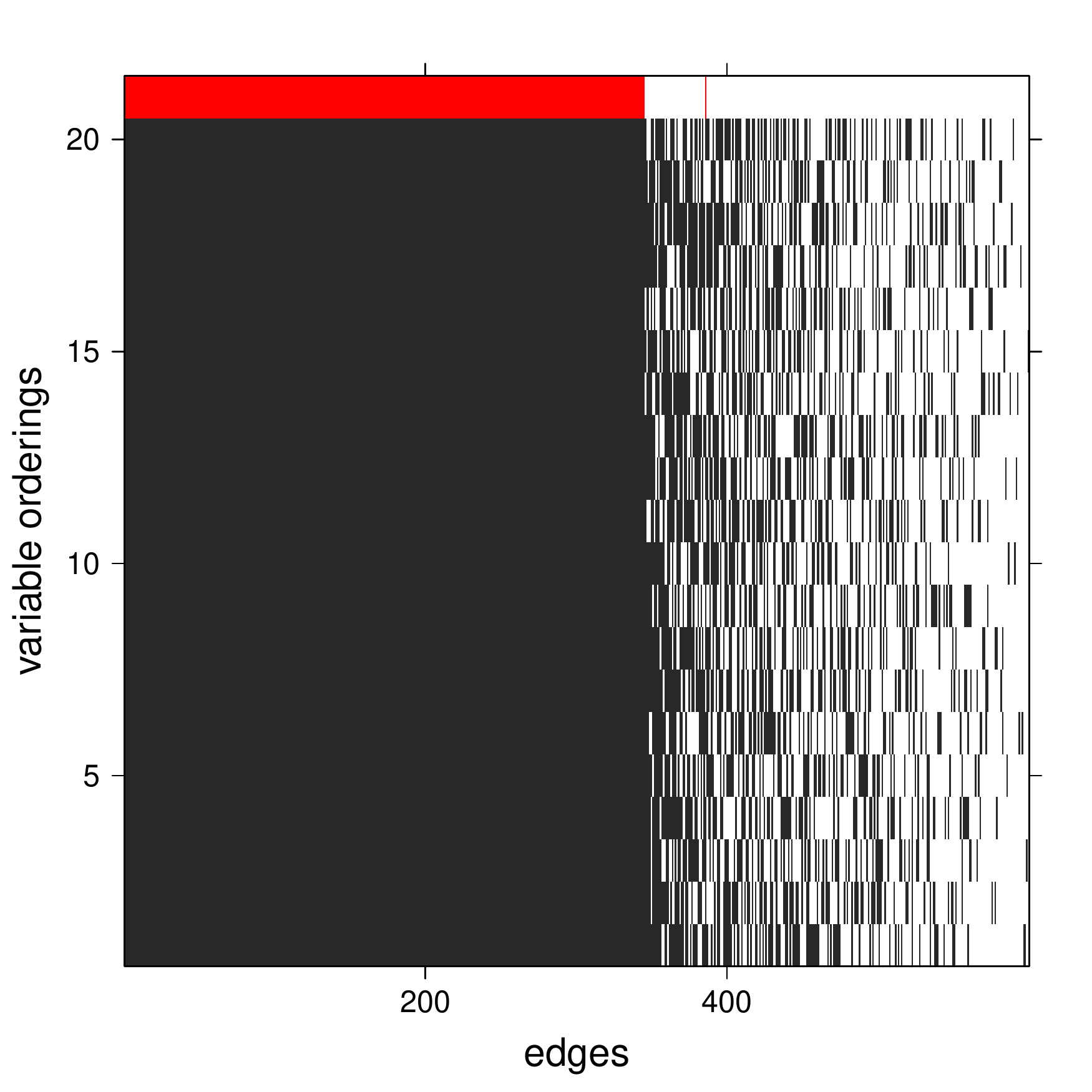}
     \label{fig.levelplot.00125}
  }\qquad
   \subfigure[$\alpha=0.005$]{
     \includegraphics[scale=0.3,angle=0]{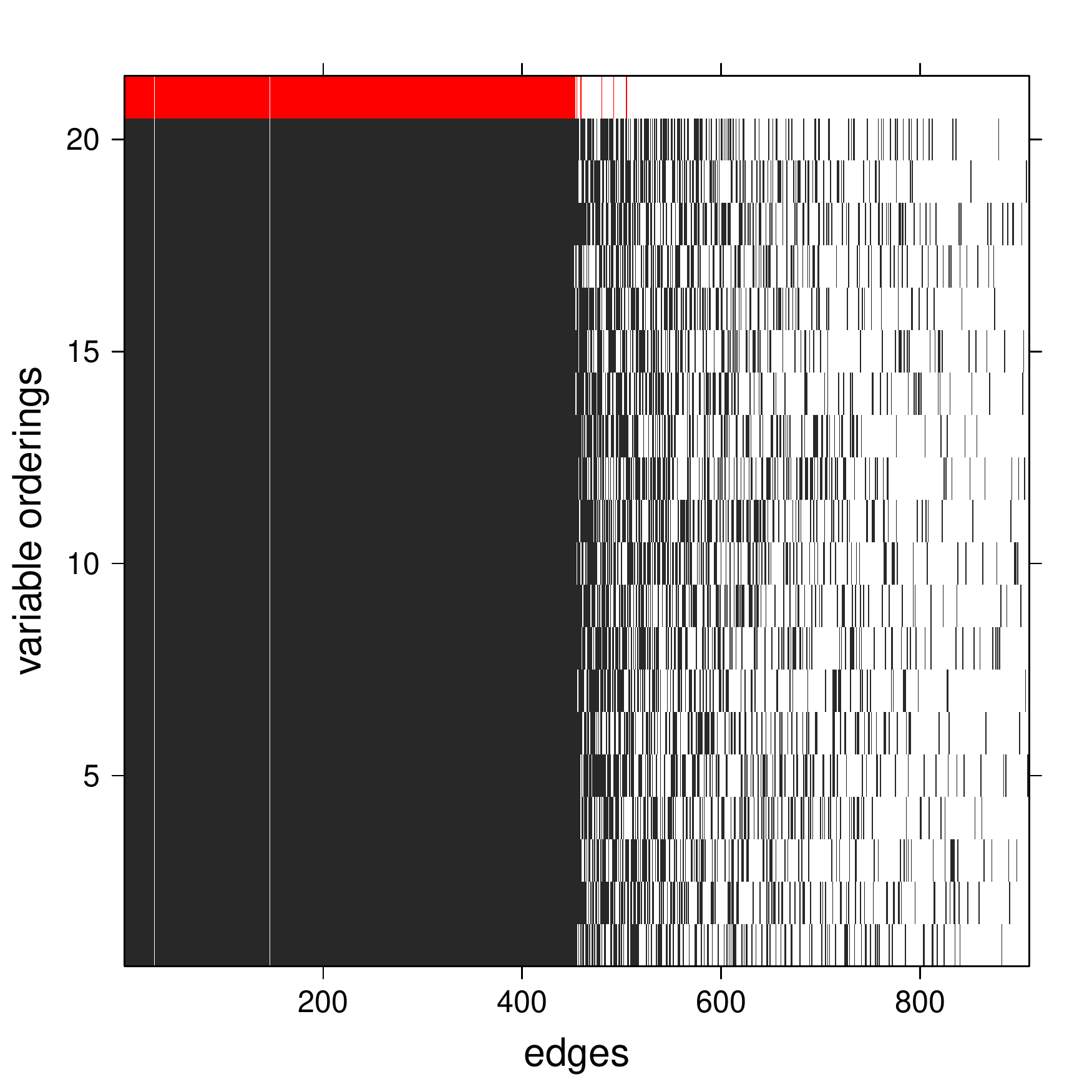}
     \label{fig.levelplot.005}
  }\\
  \subfigure[$\alpha=0.02$]{
     \includegraphics[scale=0.3,angle=0]{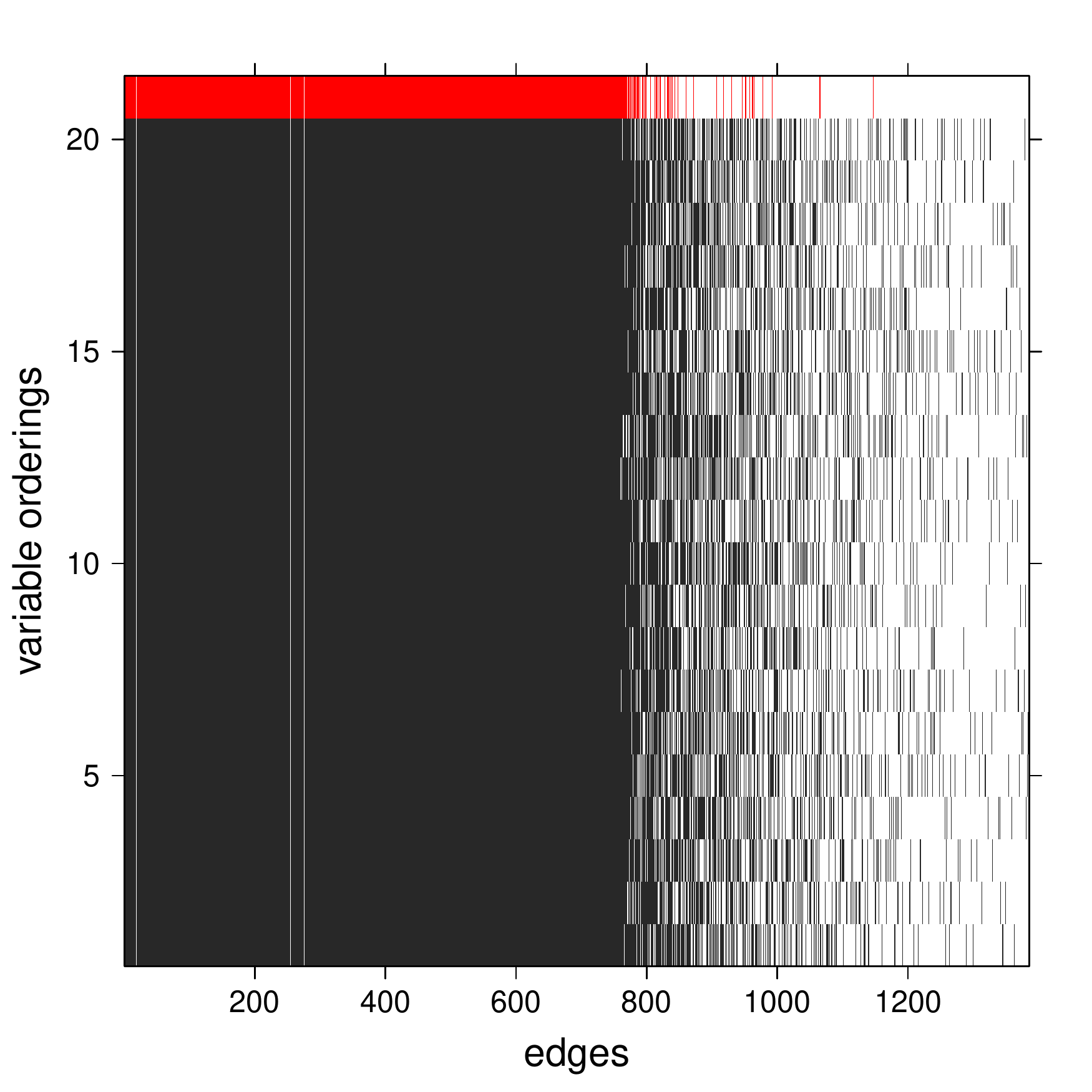}
     \label{fig.levelplot.02}
  }\qquad
   \subfigure[$\alpha=0.04$]{
     \includegraphics[scale=0.3,angle=0]{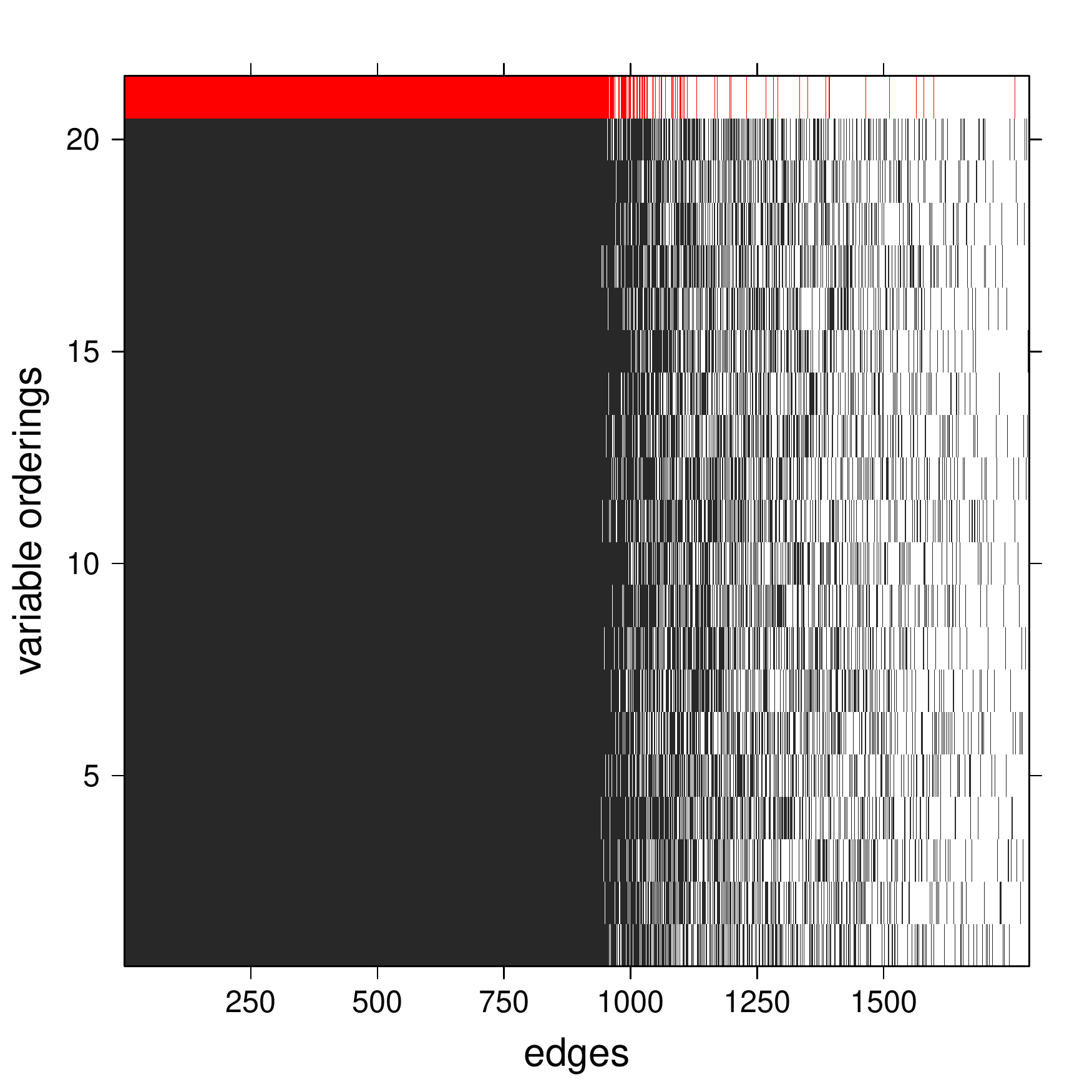}
     \label{fig.levelplot.04}
  }
  \caption{Estimated edges with the PC-algorithm (black) for 20 random
    orderings on the variables, as well as with the PC-stable algorithm
    (red, shown as variable ordering 21),
    for a random graph from the high-dimensional setting. The edges along the $x$-axes are ordered
    according to their presence in the 20 random orderings using the
    original PC-algorithm. Edges that did not occur for any of the
    orderings were omitted.}
  \label{fig.levelplot.alphas}
\end{figure}

We first consider the number of estimated errors in the skeleton, shown in
the first row of Figure \ref{fig.sim.skelet}. We see that PC-stable and
RFCI-stable return estimated skeletons with fewer edges than PC and RFCI,
for all values of $\alpha$. This can be explained by the fact that
PC-stable and RFCI-stable tend to perform more tests than PC and RFCI (see
also Appendix \ref{app.sim.time.pc}). Moreover, for both algorithms smaller
values of $\alpha$ lead to sparser outputs, as expected. When interpreting
these plots, it is useful to know that the average number of edges in the true
CPDAGs and PAGs are 1000 and 919, respectively. Thus, for both algorithms
and almost all values of $\alpha$, the estimated graphs are too sparse. 

The second row of Figure \ref{fig.sim.skelet} shows that PC-stable and
RFCI-stable make fewer errors in the estimation of the skeletons than PC
and RFCI, for all values of $\alpha$. This may be somewhat surprising given
the observations above: for most values of $\alpha$ the output of PC and
RFCI is too sparse, and the output of PC-stable and RFCI-stable is even
sparser. Thus, it must be that PC-stable and RFCI-stable yield a large
decrease in the number of false positive edges that outweighs any increase
in false negative edges. 

This conclusion is also supported by the last row of Figure
\ref{fig.sim.skelet}, which shows that PC-stable and RFCI-stable have a
better True Discovery Rate (TDR) for all values of $\alpha$, where the TDR
is defined as the proportion of edges in the estimated skeleton that are
also present in the true skeleton.



Figure \ref{fig.levelplot.alphas} shows more detailed results
for the estimated skeletons of PC and PC-stable for one of the 250 graphs
(randomly chosen), for four different values of $\alpha$. For each value of
$\alpha$ shown, PC yielded a certain number of
stable edges that were present for all 20 variable orderings, but also a
large number of extra edges that seem to pop in or out randomly for
different orderings. The PC-stable algorithm yielded far fewer edges (shown
in red), and roughly captured the edges that were stable among the
different variable orderings for PC. The results for RFCI and RFCI-stable
show an equivalent picture. 


\subsection{Estimation of the CPDAGs and PAGs}\label{sec.sim.outputs}

We now consider estimation of the CPDAG or PAG, that is, also taking into
account the edge orientations. For CPDAGs, we summarize the number of
estimation errors using the Structural Hamming Distance (SHD), which is
defined as the minimum number of edge insertions, deletions, and flips that
are needed in order to transform the estimated graph into the true one. For
PAGs, we summarize the number of estimation errors by counting the number
of errors in the edge marks, which we call ``SHD edge marks". For example,
if an edge $X_i \to X_j$ is present in the estimated PAG but the true PAG
contains $X_i \leftrightarrow X_j$, then that counts as one error, while it
counts as two errors if the true PAG contains, for example, $X_i \leftarrow
X_j$ or $X_i$ and $X_j$ are not adjacent. 

\begin{figure}[!h]\centering%
  \subfigure[Estimation performance of (L)PC(-stable), (L)CPC(-stable), and
    (L)MPC(-stable) for the CPDAGs in terms of SHD.]{
     \includegraphics[scale=0.36,angle=0]{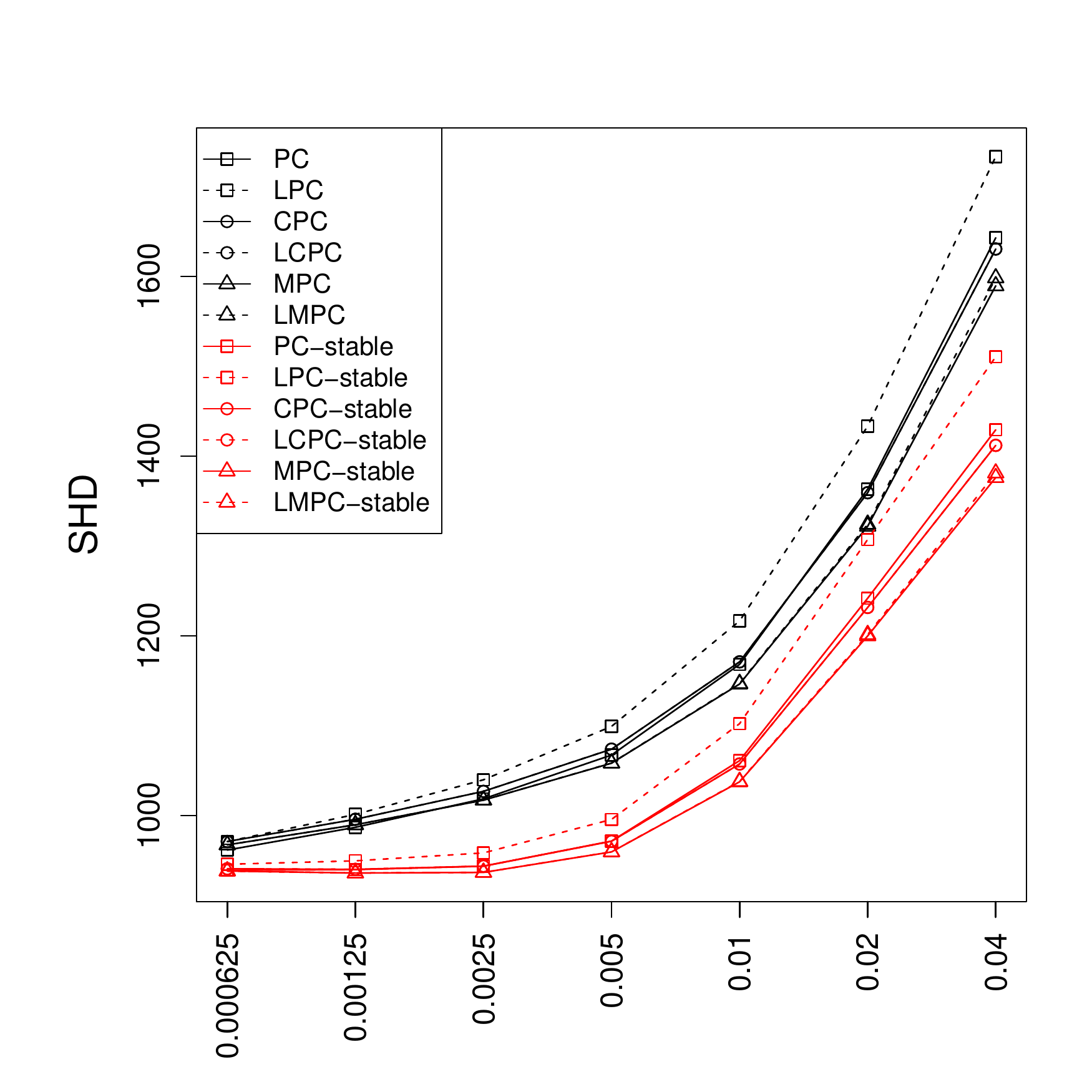}
     \label{fig.cpdags.shd}
  }\qquad
  \subfigure[Estimation performance of RFCI(-stable),
    CRFCI(-stable), and MRFCI(-stable) for the PAGs in terms of SHD edge
    marks.]{
     \includegraphics[scale=0.36,angle=0]{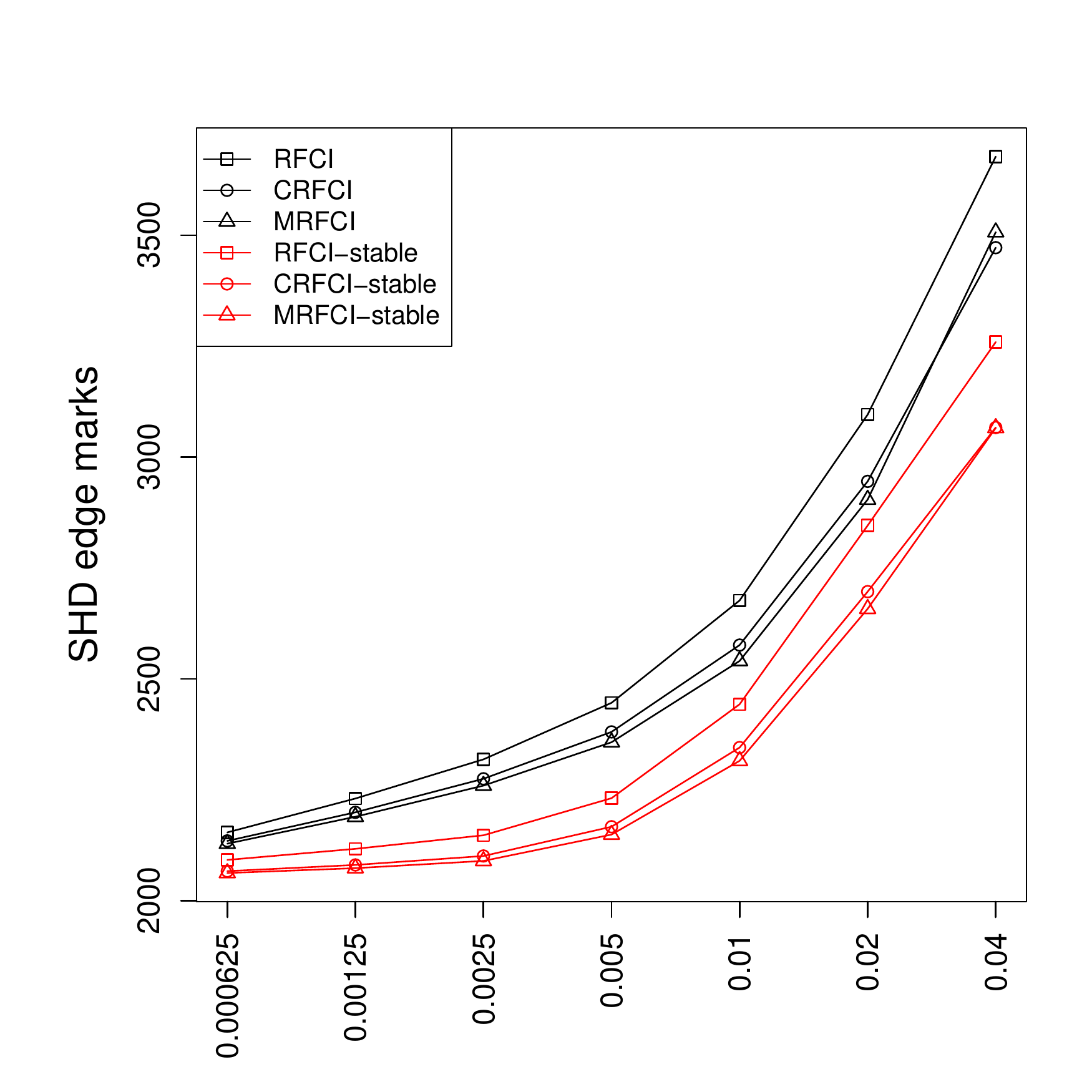}
     \label{fig.pags.shd}
  }
\caption{Estimation performance in terms of SHD for the CPDAGs and SHD edge
  marks for the PAGs, shown
  as averages over 250 randomly generated graphs and 20 random
  variable orderings per graph, for different values of
  $\alpha$ ($x$-axis displayed in $\log$ scale).} 
  \label{fig.sim.shd}
\end{figure}

Figure \ref{fig.sim.shd} shows that the PC-stable and
RFCI-stable versions have significantly better estimation performance than
the versions with the original skeleton, for all values of
$\alpha$. Moreover, MPC(-stable) and CPC(-stable) perform better than
PC(-stable), as do MRFCI(-stable) and CRFCI(-stable) with respect to
RFCI(-stable). Finally, for PC the idea to introduce bi-directed edges and
lists in LCPC(-stable) and LMPC(-stable) seems to make little difference.

Figure \ref{fig.sim.shd.variance} shows the variance in SHD for the CPDAGs,
see Figure \ref{fig.cpdags.variance}, and the variance in SHD edge marks
for the PAGs, see Figure \ref{fig.pags.variance}, both computed over
the 20 random variable orderings per graph, and then plotted as averages
over the 250 randomly generated graphs for the different values of
$\alpha$. The PC-stable and RFCI-stable versions yield significantly
smaller variances than their counterparts with unstabilized
skeletons. Moreover, the variance is further reduced for (L)CPC-stable and
(L)MPC-stable, as well as for CRFCI-stable and MRFCI-stable, as expected.

\begin{figure}[!h]\centering%
   \subfigure[Estimation performance of (L)PC(-stable), (L)CPC(-stable), and
    (L)MPC(-stable) for the CPDAGs in terms of the variance of SHD.]{
     \includegraphics[scale=0.36,angle=0]{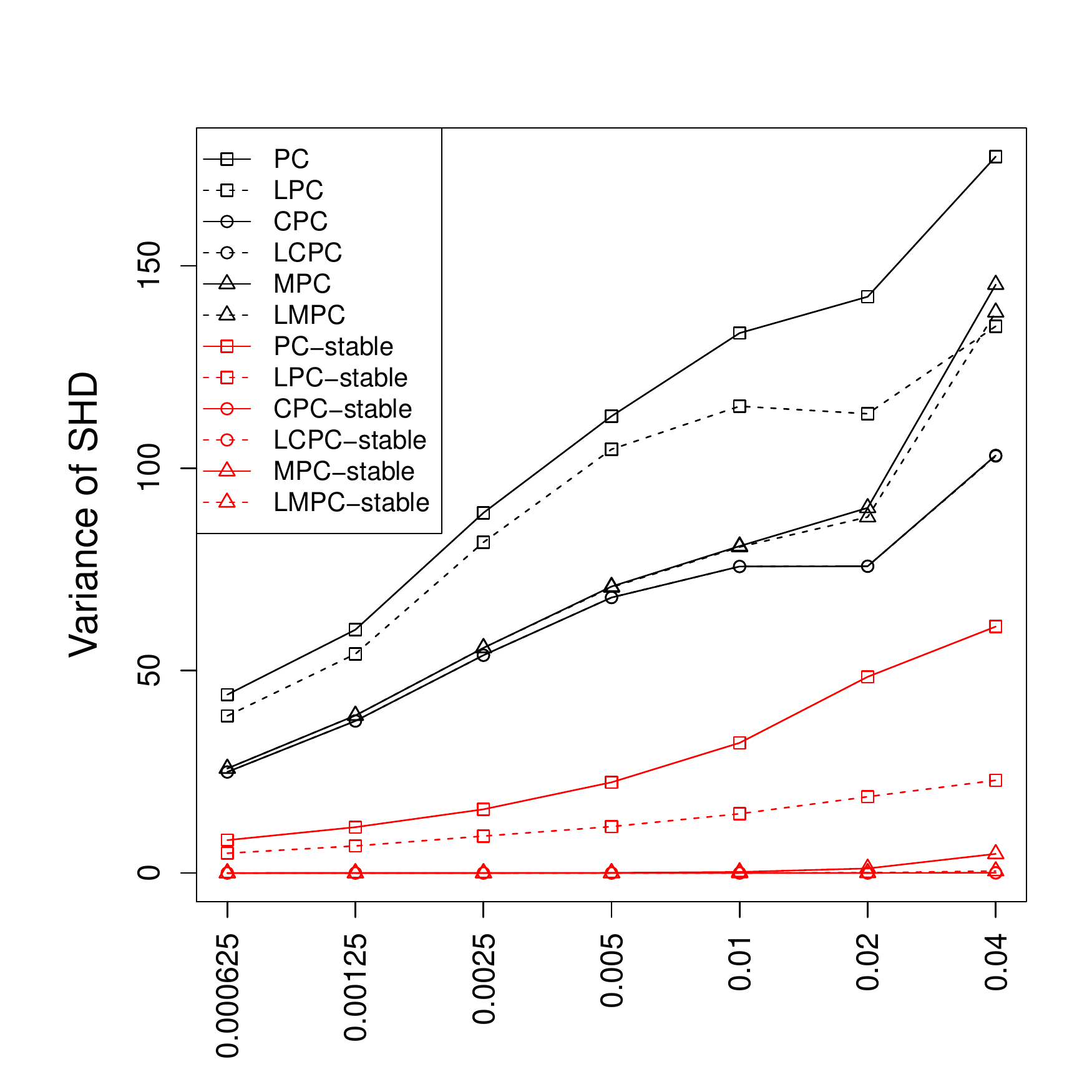}
     \label{fig.cpdags.variance}
  }\qquad
\subfigure[Estimation performance of RFCI(-stable),
    CRFCI(-stable), and MRFCI(-stable) for the PAGs in terms of the
    variance SHD edge marks.]{
     \includegraphics[scale=0.36,angle=0]{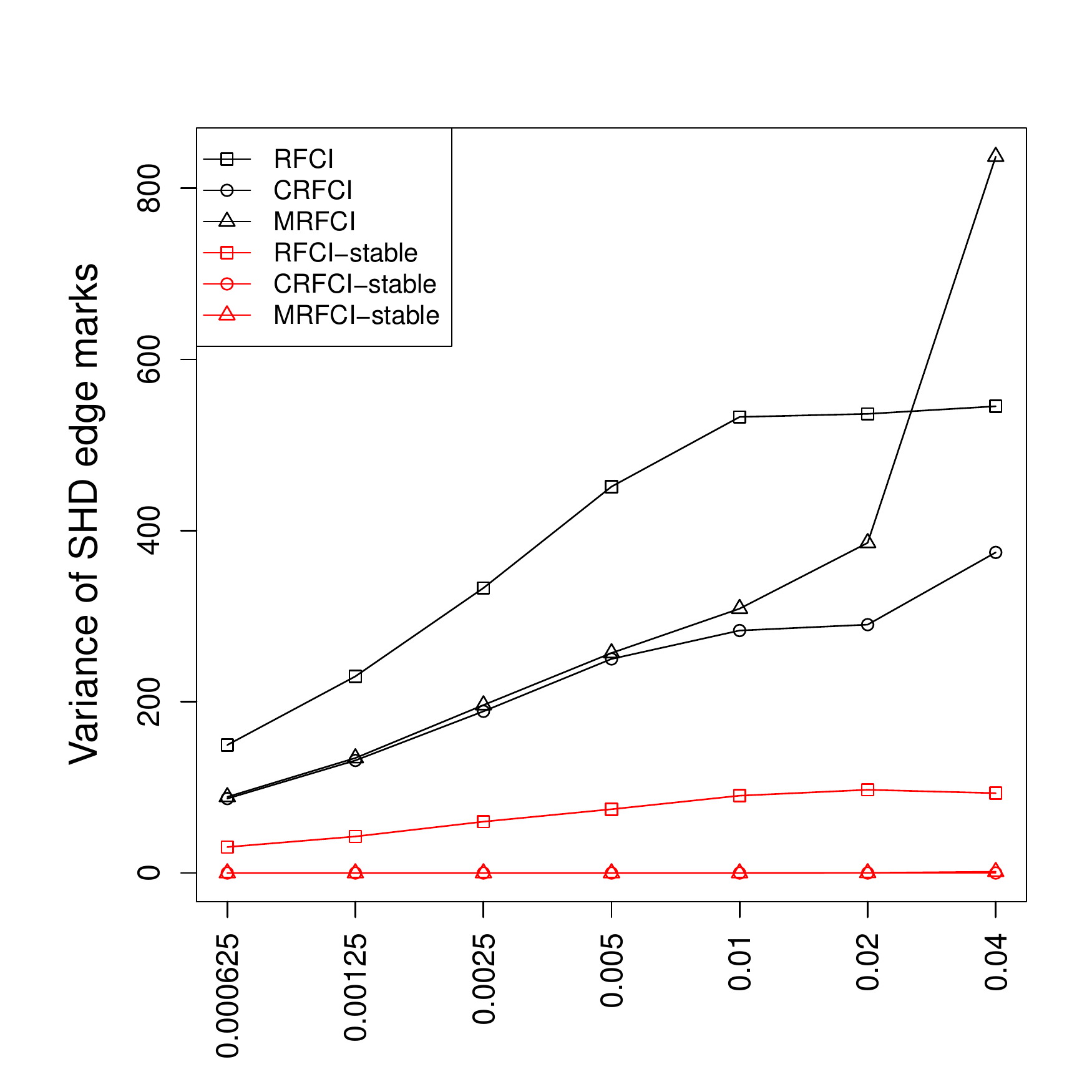}
     \label{fig.pags.variance}
  }
  \caption{Estimation performance in terms of the variance of the SHD for
    the CPDAGs and SHD edge marks for the PAGs over the 20 random variable
    orderings per graph, shown as averages over 250 randomly generated
    graphs, for different values of $\alpha$ ($x$-axis displayed in $\log$
    scale).} 
  \label{fig.sim.shd.variance}
\end{figure}

Figure \ref{fig.sim.directed} shows receiver operating characteristic
curves (ROC) for the directed edges in the estimated CPDAGs (Figure
\ref{fig.cpdags.directed}) and PAGs (Figure
\ref{fig.pags.directed}). We see that finding directed edges is much harder
in settings that allow for hidden variables, as shown by the lower true
positive rates (TPR) and higher false positive rates (FPR) in Figure
\ref{fig.pags.directed}. Within each figure, the different versions of the
algorithms perform roughly similar, and MPC-stable and MRFCI-stable yield
the best ROC-curves. 

\begin{figure}[!h]\centering%
    \subfigure[Estimation performance of (L)PC(-stable), (L)CPC(-stable), and
    (L)MPC(-stable) for the CPDAGs in terms of TPR and FPR for the
   directed edges.]{
     \includegraphics[scale=0.36,angle=0]{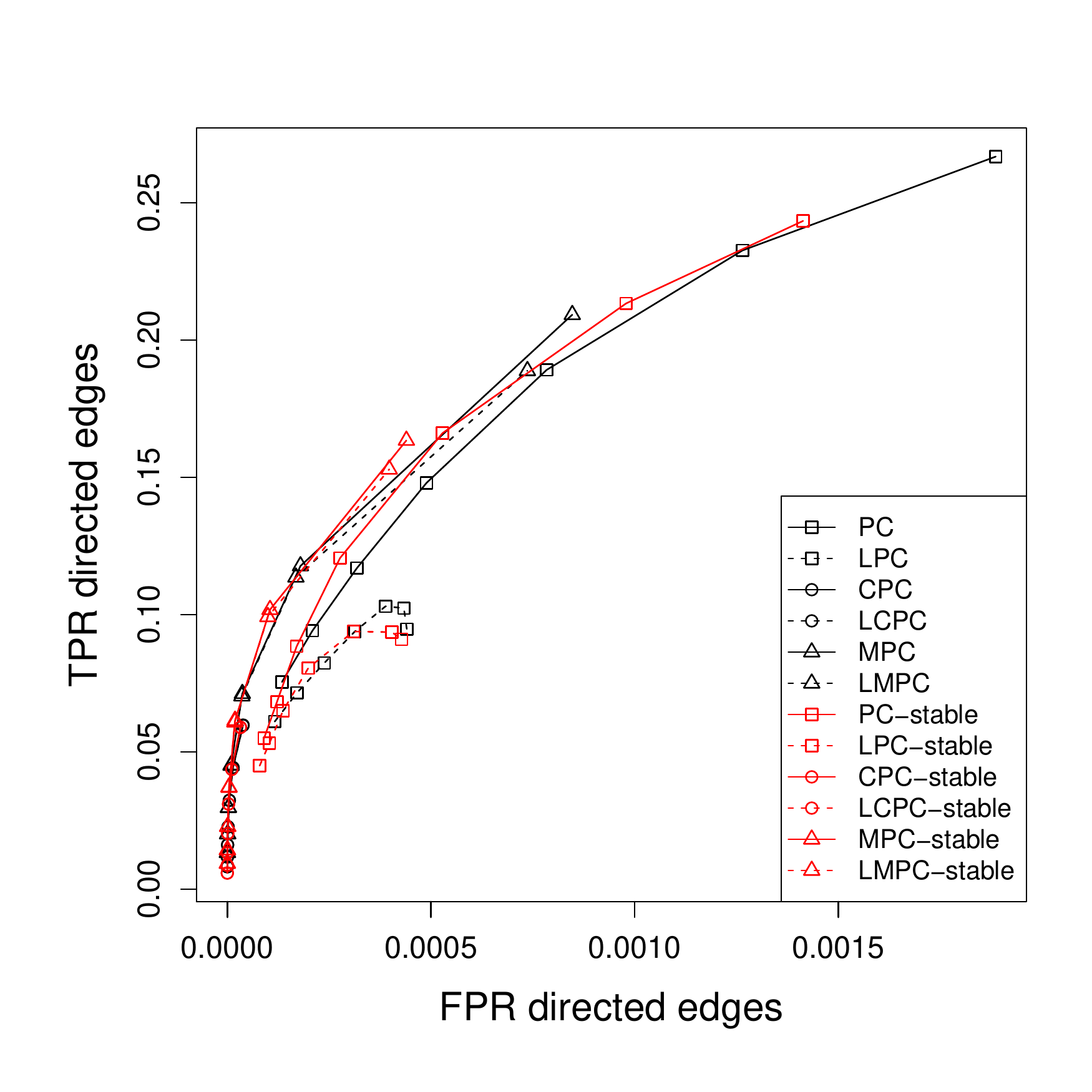}
     \label{fig.cpdags.directed}

  }\qquad
  \subfigure[Estimation performance of RFCI(-stable),
    CRFCI(-stable), and MRFCI(-stable) for the PAGs in terms of TPR and FPR
    for the directed edges.]{
     \includegraphics[scale=0.36,angle=0]{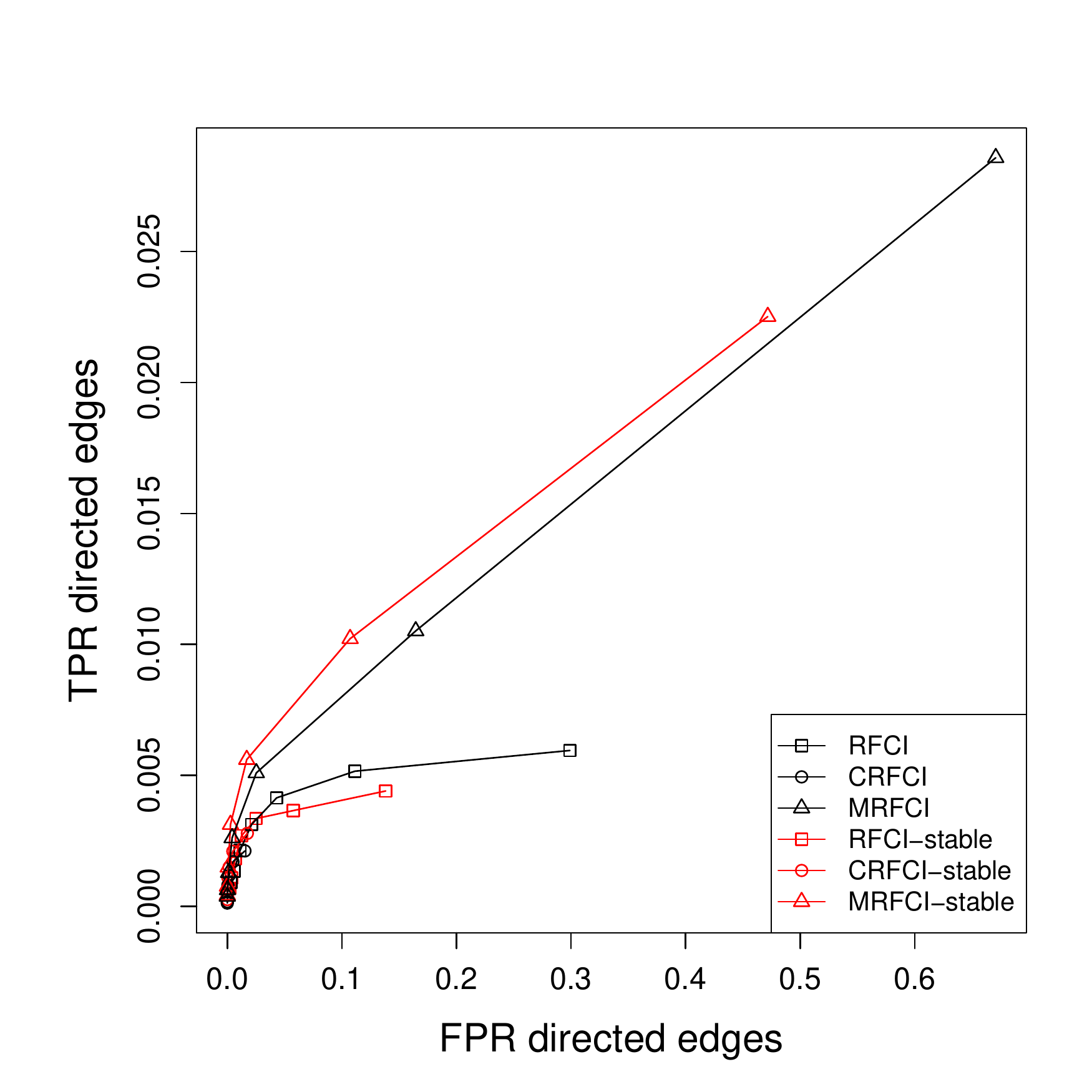}
     \label{fig.pags.directed}

  }
  \caption{Estimation performance in terms of TPR and FPR for the
   directed edges in CPDAGs and PAGs, shown as averages over 250 randomly
   generated graphs and 20 random variable orderings per graph, where every
   curve is plotted with respect to the different values of
  $\alpha$.} 
  \label{fig.sim.directed}
\end{figure}

\section{Yeast gene expression data}\label{sec.simulation.realdata}

We also compared the PC and PC-stable algorithms on
the yeast gene expression data \citep{Hughes00} that were already briefly
discussed in  Section \ref{sec.introduction}. In Section
\ref{sec.yeastdata.pc} we consider estimation of the skeleton of the CPDAG,
and in Section \ref{sec.yeastdata.ida} we consider estimation of bounds on
causal effects.

We used the same pre-processed data as in
\cite{MaathuisColomboKalischBuhlmann10}. These contain: (1) expression
measurements of 5361 genes for 63 wild-type cultures (observational data of
size $63 \times 5361$), and (2) expression measurements of the same 5361
genes for 234 single-gene deletion mutant strains (interventional data of
size $234 \times 5361$).

\subsection{Estimation of the skeleton}\label{sec.yeastdata.pc}

We applied PC and PC-stable to the observational data.
We saw in Section \ref{sec.introduction} that the PC-algorithm
yielded estimated skeletons that were highly dependent on
the variable ordering, as shown in black in Figure
\ref{fig.skeletanalysis.1} for the 26 variable orderings (the original
ordering and 25 random orderings of the variables). The PC-stable algorithm
does not suffer from this order-dependence, and consequently all these 26
random orderings over the variables produce the same skeleton which is
shown in the figure in red. We see that the PC-stable algorithm yielded a
far sparser skeleton (2086 edges for PC-stable versus 5015-5159 edges for
the PC-algorithm, depending on the variable ordering). Just as in the
simulations in Section \ref{sec.simulation.results} the order-independent
skeleton from the PC-stable algorithm roughly captured the edges that were
stable among the different order-dependent skeletons estimated from
different variable orderings for the original PC-algorithm.

\begin{figure}[!h]\centering%
  \subfigure[As Figure \ref{fig.levelplot.1}, plus the edges occurring in the
  unique estimated skeleton using the PC-stable algorithm over the same 26
  variable orderings (red, shown as variable ordering 27).]{
     \includegraphics[scale=0.32,angle=0]{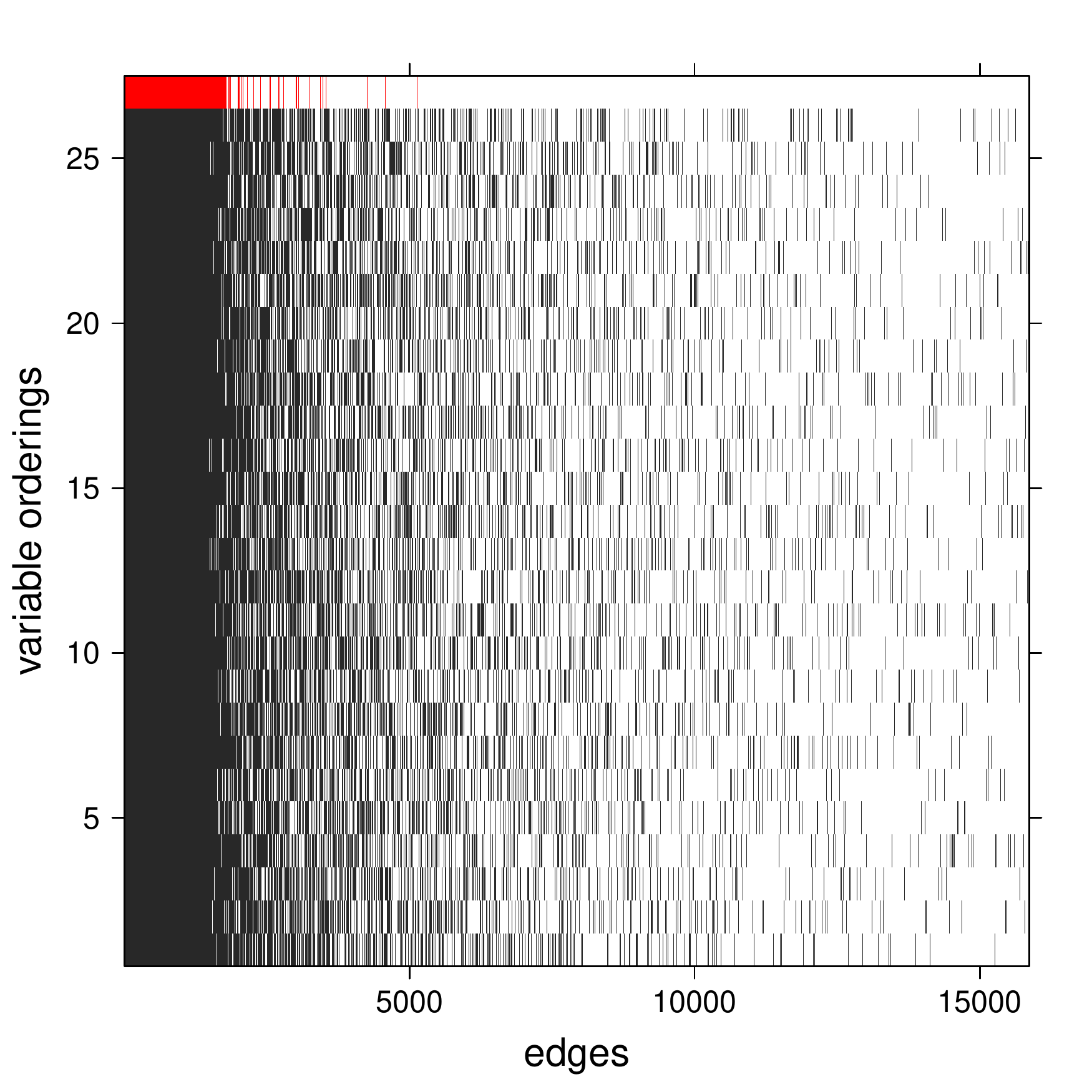}
     \label{fig.levelplot.2}
  }\qquad
   \subfigure[The step function shows the proportion of the 26 variable orderings
   in which the edges were present for the original PC-algorithm, where the
   edges are ordered as in Figure \ref{fig.levelplot.2}. The red bars show
   the edges present in the estimated skeleton using the PC-stable algorithm.]{
     \includegraphics[scale=0.34,angle=0]{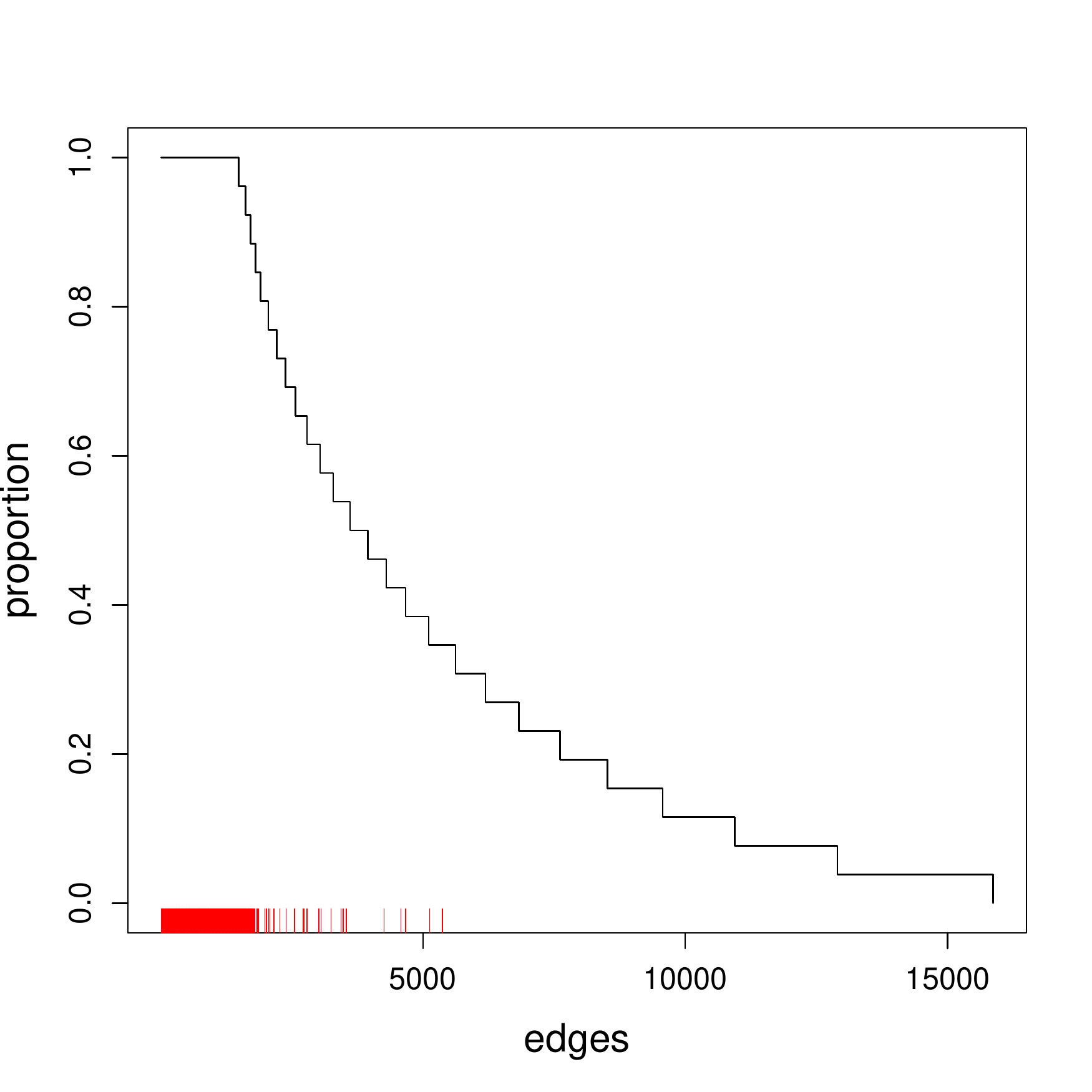}
     \label{fig.percplot.1}
  }
  \caption{Analysis of estimated skeletons of the CPDAGs for the yeast gene
    expression data \citep{Hughes00}, using the PC and PC-stable
    algorithms. The PC-stable algorithm yields an order-independent
    skeleton that roughly captures the edges that were stable among the
    different variable orderings for the original PC-algorithm.}
  \label{fig.skeletanalysis.1}
\end{figure}

To make ``captured the edges that were stable" somewhat more precise, we
defined the following two sets: Set 1 contained all edges (directed edges)
that were present for all 26 variable orderings using the original PC-algorithm,
and Set 2 contained all edges (directed edges) that were present for at least
$90\%$ of the 26 variable orderings using the original PC-algorithm. Set 1
contained 1478 edges (7 directed edges), while Set 2 contained 1700 edges
(20 directed edges).

Table \ref{table.tpfp} shows how well the PC and PC-stable algorithms
could find these stable edges in terms of number of edges in the estimated
graphs that are present in Sets 1 and 2 (IN), and the number of edges
in the estimated graphs that are not present in Sets 1 and 2 (OUT). We see that the number of estimated edges present in
Sets 1 and 2 is about the same for both algorithms, while the output of the
PC-stable algorithm has far fewer edges which are not present in the two
specified sets.

\begin{table}[h]
\centering
\begin{tabular}{cc|r|r|r|r|}
\cline{3-6}
& & \multicolumn{2}{c|}{Edges} & \multicolumn{2}{c|}{Directed edges}\\ \cline{3-6}
& & PC-algorithm & PC-stable algorithm & PC-algorithm & PC-stable algorithm \\ \cline{1-6}
\multicolumn{1}{|c}{\multirow{2}{*}{Set 1}} & \multicolumn{1}{|c|}{IN} &
1478 (0) & 1478 (0) & 7 (0) & 7 (0)     \\ \cline{2-6}
\multicolumn{1}{|c}{} & \multicolumn{1}{|c|}{OUT} & 3606 (38) & 607 (0) &
4786 (47)  & 1433 (7)   \\ \cline{1-6}
\multicolumn{1}{|c}{\multirow{2}{*}{Set 2}} &
\multicolumn{1}{|c|}{IN} & 1688 (3) & 1688 (0) & 19 (1) & 20 (0)\\ \cline{2-6}
\multicolumn{1}{|c}{} &
\multicolumn{1}{|c|}{OUT} & 3396 (39) & 397 (0) & 4774 (47) & 1420 (7) \\
\cline{1-6}
\end{tabular}
\caption{Number of edges in the estimated graphs that are present in Sets 1 and 2 (IN), and the number of edges
in the estimated graphs that are not present in Sets 1 and 2 (OUT). The results are shown as averages (standard deviations)
  over the 26 variable orderings.}
  \label{table.tpfp}
\end{table}

\subsection{Estimation of causal effects}\label{sec.yeastdata.ida}

We used the interventional data as the gold standard for
estimating the total causal effects of the 234 deleted genes on the
remaining 5361 (see \cite{MaathuisColomboKalischBuhlmann10}).
We then defined the top $10\%$ of the largest effects in absolute value
as the target set of effects, and we evaluated how well IDA
\citep{MaathuisKalischBuehlmann09,MaathuisColomboKalischBuhlmann10}
identified these effects from the observational data.

We saw in Figure \ref{fig.roc.old} that IDA with the original
PC-algorithm is highly order-dependent. Figure \ref{fig.roc.new001} shows
the same analysis with PC-stable (solid black lines). We see that using
PC-stable generally yielded better and more stable results than the
original PC-algorithm. Note that some of the curves for PC-stable are worse
than the reference curve of \cite{MaathuisColomboKalischBuhlmann10} towards
the beginning of the curves. This can be explained by the fact that the
original variable ordering seems to be especially ``lucky'' for this part
of the curve (cf. Figure \ref{fig.roc.old}). There is still variability in
the ROC curves in Figure \ref{fig.roc.new001} due to the order-dependent
v-structures (because of order-dependent separating sets) and orientations in the
PC-stable algorithm, but this variability is less prominent than in
Figure \ref{fig.roc.old}. Finally, we see that there are 3 curves that
produce a very poor fit.

Using CPC-stable and MPC-stable helps in
stabilizing the outputs, and in fact all the 25 random variable orderings
produce almost the same CPDAGs for both modifications. Unfortunately, these
estimated CPDAGs are almost entirely undirected (around 90 directed edges
among the 2086 edges) which leads to a large equivalence class and
consequently to a poor performance in IDA, see the dashed black line in Figure
\ref{fig.roc.new001} which corresponds to the 25 random variable orderings
for both CPC-stable and MPC-stable algorithms.

\begin{figure}[!h]\centering%
  \includegraphics[scale=0.50,angle=0]{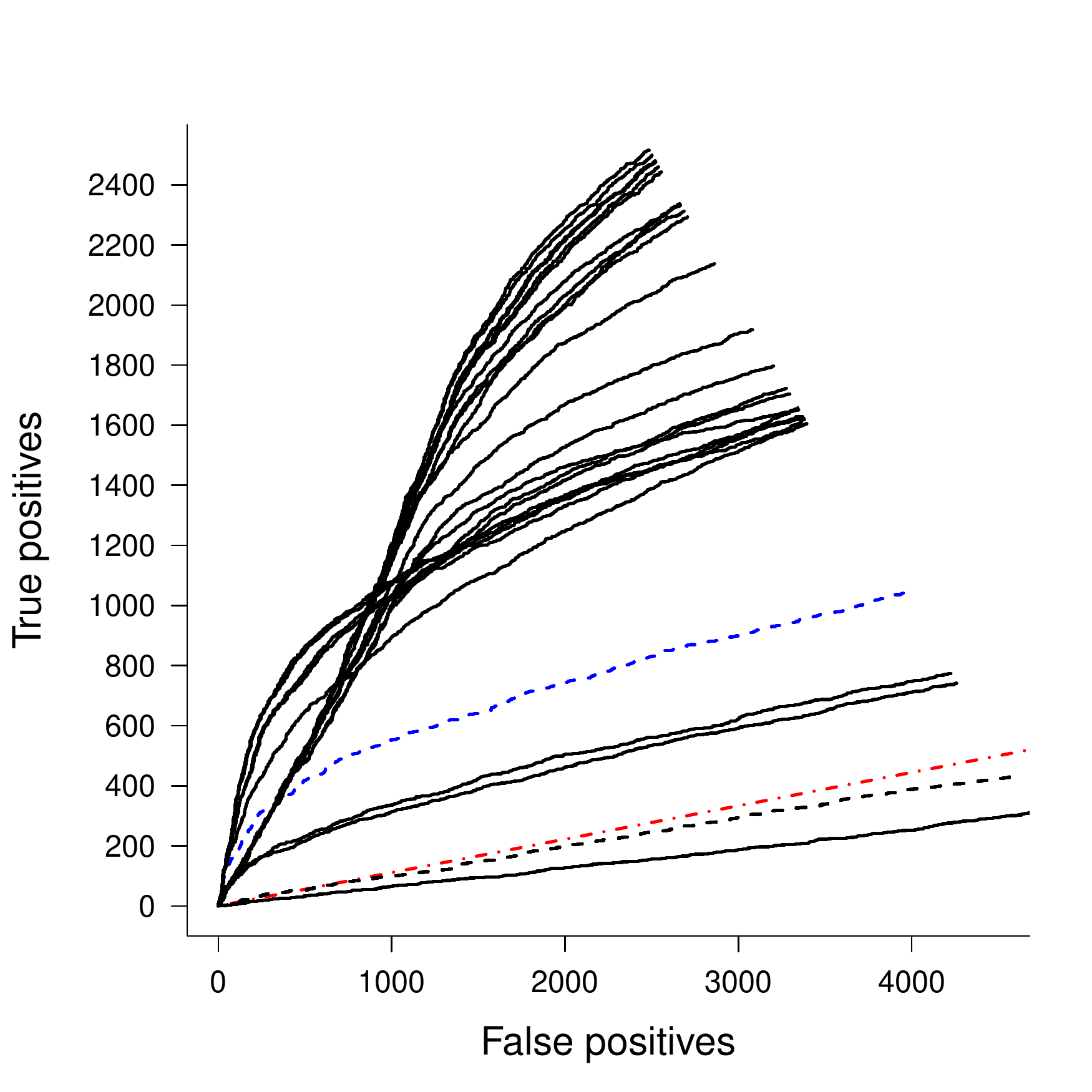}
  \caption{ROC curves corresponding to the 25 random orderings of the
  variables for the analysis of yeast gene expression data
  \citep{Hughes00}, where the curves are generated as in
  \cite{MaathuisColomboKalischBuhlmann10} but using PC-stable (solid black
  lines) and MPC-stable and CPC-stable (dashed black lines) with
  $\alpha=0.01$. The ROC curves from
  \cite{MaathuisColomboKalischBuhlmann10} (dashed blue) and the one for
  random guessing (dashed-dotted red) are shown as references. The
  resulting causal rankings are less order-dependent.}
  \label{fig.roc.new001}
\end{figure}


Another possible solution for the order-dependence orientation issues would
be to use stability selection \citep{MeinshausenBuehlmann10} to find the
most stable orientations among the runs. In fact, \cite{StekhovenEtAll12}
already proposed a combination of IDA and stability selection which led to
improved performance when compared to IDA alone, but they used the original
PC-algorithm and did not permute the variable ordering. We present here a
more extensive analysis, where we consider the PC-algorithm (black lines),
the PC-stable algorithm (red lines), and the MPC-stable algorithm (blue
lines). Moreover, for each one of these algorithms we propose three
different methods to estimate the CPDAGs and the causal effects: (1) use
the original ordering of the variables (solid lines); (2) use the same
methodology used in \cite{StekhovenEtAll12} with 100 stability selection
runs but without permuting the variable orderings (labelled as + SS; dashed
lines); and (3) use the same methodology used in \cite{StekhovenEtAll12}
with 100 stability selection runs but permuting the variable orderings in
each run (labelled as + SSP; dotted lines). The results are shown in Figure
\ref{fig.hughesetal.stabsel} where we investigate the performance for the
top 20000 effects instead of the 5000 as in Figures \ref{fig.roc.old} and
\ref{fig.roc.new001}.

\begin{figure}[!h]\centering%
     \includegraphics[scale=0.5,angle=0]{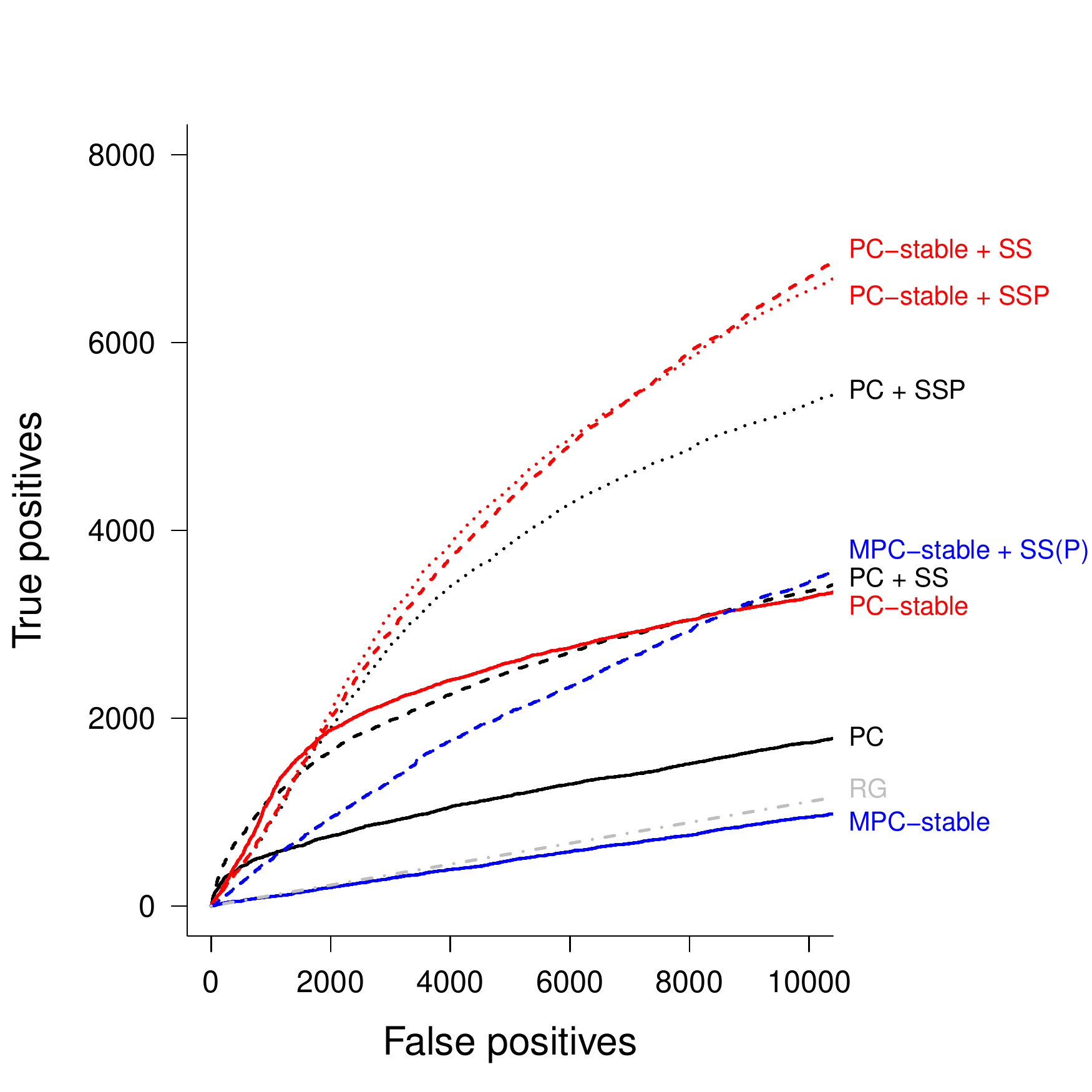}
  \caption{Analysis of the yeast gene expression data \citep{Hughes00} for PC,
    PC-stable, and MPC-stable algorithms using the original ordering over
    the variables (solid lines), using 100 runs stability selection without
    permuting the variable orderings, labelled as + SS (dashed lines), and
    using 100 runs stability selection with permuting the variable
    orderings, labelled as + SSP (dotted lines). The grey line labelled as
    RG represents the random guessing.}
  \label{fig.hughesetal.stabsel}
\end{figure}

We see that PC with stability selection and permuted variable orderings (PC
+ SSP) loses some performance at the beginning of the curve when compared
to PC with standard stability selection (PC + SS), but it has much better
performance afterwards. The PC-stable algorithm with the original variable
ordering performs very similar to PC plus stability selection (PC +
SS) along the whole curve. Moreover, PC-stable plus stability selection
(PC-stable + SS and PC-stable + SSP), loses a bit at the beginning of the
curves but picks up much more signal later on in the curve. It is
interesting to note that for PC-stable with stability selection, it makes
little difference if the variable orderings are further permuted or not,
even though PC-stable is not fully order-independent (see Figure
\ref{fig.roc.new001}). In fact, PC-stable plus stability selection (with or
without permuted variable orderings) produces the best fit over all results.

\section*{Acknowledgments}
We thank Richard Fox, Markus Kalisch, and Thomas Richardson for their valuable comments.

\section{Discussion}\label{sec.discussion}

Due to their computational efficiency, constraint-based causal structure learning algorithms are
often used in sparse high-dimensional settings. We have seen, however, that especially in these
settings the order-dependence in these algorithms is highly problematic.


In this paper, we investigated this issue systematically, and resolved the various sources of order-dependence.
There are of course many ways in which the order-dependence issues could be resolved, and we designed our modifications to be as simple as possible. Moreover, we made sure that existing high-dimensional consistency
results for PC-, FCI- and RFCI-algorithms remain valid for their
modifications under the same conditions.  We showed that our
proposed modifications yield improved and more stable estimation in sparse
high-dimensional settings for simulated data, while their performances are similar to
the performances of the original algorithms in low-dimensional settings.

Additionally to the order-dependence discussed in this
paper, there is another minor type of order-dependence in the sense that
the output of these algorithms also depends on the order in which the
final orientation rules for the edges are applied. The reason is that an edge(mark) could be eligible
for orientation by several orientation rules, and might be oriented
differently depending on which rule is applied first. In our analyses, we
have always used the original orderings in which the rules were given.

Compared to the adaptation of \cite{CanoEtAl08}, the modifications we
propose are much simpler and we made sure that they preserve existing
soundness, completeness, and high-dimensional consistency
results. Finally, our modifications can be
easily used together with other adaptations of constraint-based algorithms,
for example hybrid versions of PC with
score-based methods \cite{SinghValtorta93,SpirtesMeek95,vanDijketAl03} or
the PC$^*$ algorithm \cite[Section 5.4.2.3]{SpirtesEtAl00}.

All software is implemented in the R-package \texttt{pcalg}
\citep{KalischEtAl12}.

\appendix

\section{Additional simulation results}\label{app.sim.results}

We now present additional simulation results for low-dimensional settings (Appendix \ref{app.sim.results.low}), high-dimensional settings (Appendix \ref{app.sim.time.pc}) and medium-dimensional settings (Appendix \ref{app.sim.results.equal}).

\subsection{Estimation performance in low-dimensional
  settings}\label{app.sim.results.low}

We considered the estimation performance in low-dimensional settings
with less sparse graphs.

For the scenario without latent variables, we generated 250 random weighted DAGs
with $p=50$ and $E(N)=\{2,4\}$, as described in Section \ref{sec.sim.setup}.
For each weighted DAG we generated an i.i.d.\
sample of size $n=1000$. We then estimated each graph for 50 random
orderings of the variables, using the sample versions of (L)PC(-stable),
(L)CPC(-stable), and (L)MPC(-stable) at levels $\alpha \in \{0.000625,
0.00125, 0.0025, 0.005, 0.01, 0.02, 0.04\}$ for $E(N)=2$ and
$\alpha \in \{0.005, 0.01, 0.02, 0.04, 0.08, 0.16, 0.32\}$ for $E(N)=4$ for
the partial correlation tests. Thus, for each randomly generated graph, we
obtained 50 estimated CPDAGs from each algorithm, for each value of
$\alpha$. Figure \ref{fig.sim.pc.skelet.add1} shows the estimation
performance of PC (circle; black line) and PC-stable (triangles; red line)
for the skeleton. Figure \ref{fig.sim.pc.cpdags.add1} shows the
estimation performance of all modifications of PC and PC-stable with
respect to the CPDAGs in terms of SHD, and in terms
of the variance of the SHD over the 50 random variable orderings per graph.

\begin{figure}[h]
\centering
\includegraphics[scale=0.80]{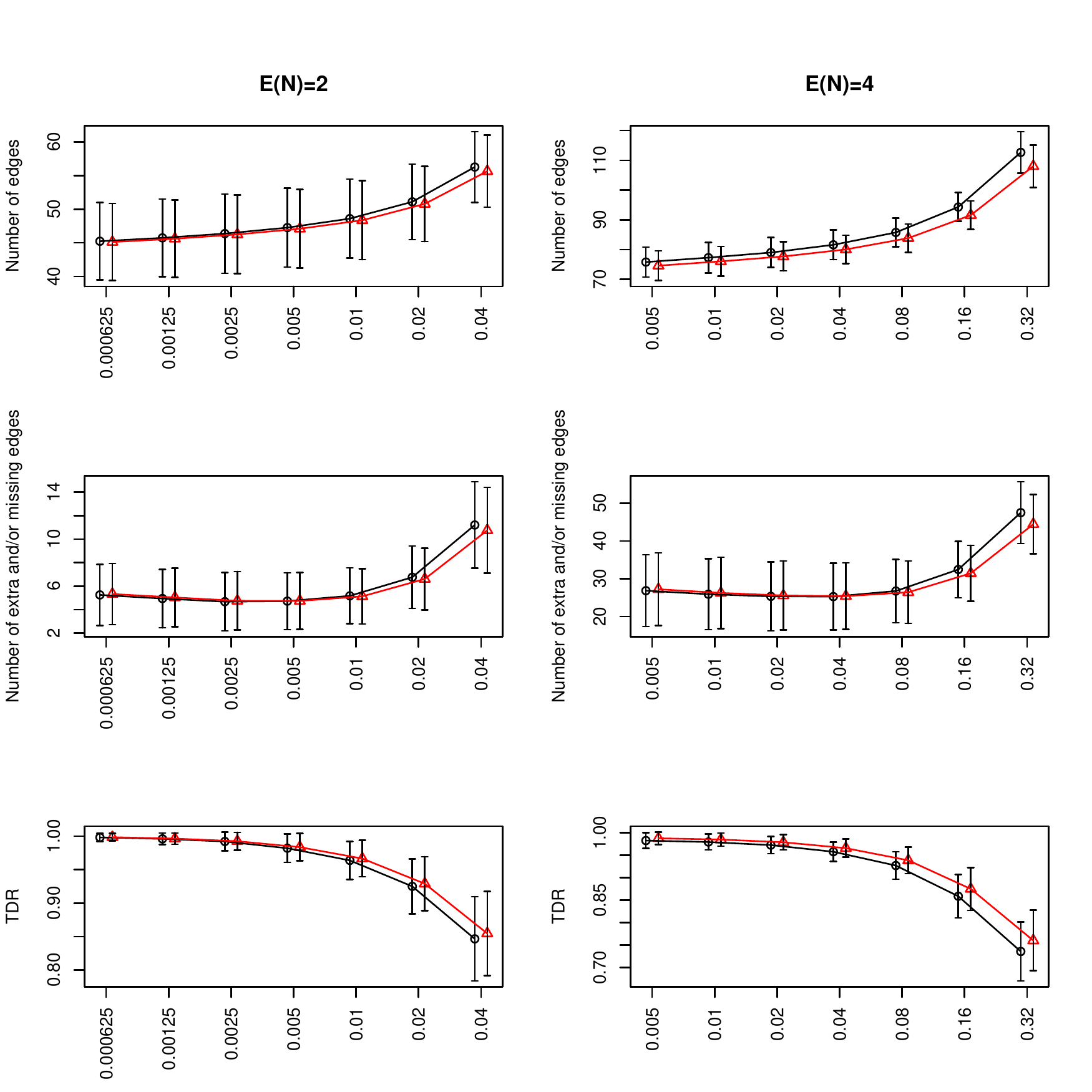}
\caption{Estimation performance of PC (circles; black line) and PC-stable
  (triangles; red line) for the skeleton of the CPDAGs, for different
  values of $\alpha$ ($x$-axis displayed in $\log$ scale) in both
  low-dimensional settings. The results are shown as averages plus or minus
  one standard deviation, computed over 250 randomly generated graphs and
  50 random variable orderings per graph, and slightly shifted up and down from
  the real values of $\alpha$ for a better visualization.}
  \label{fig.sim.pc.skelet.add1}
\end{figure}


\begin{figure}[!h]\centering%
     \includegraphics[scale=0.36,angle=0]{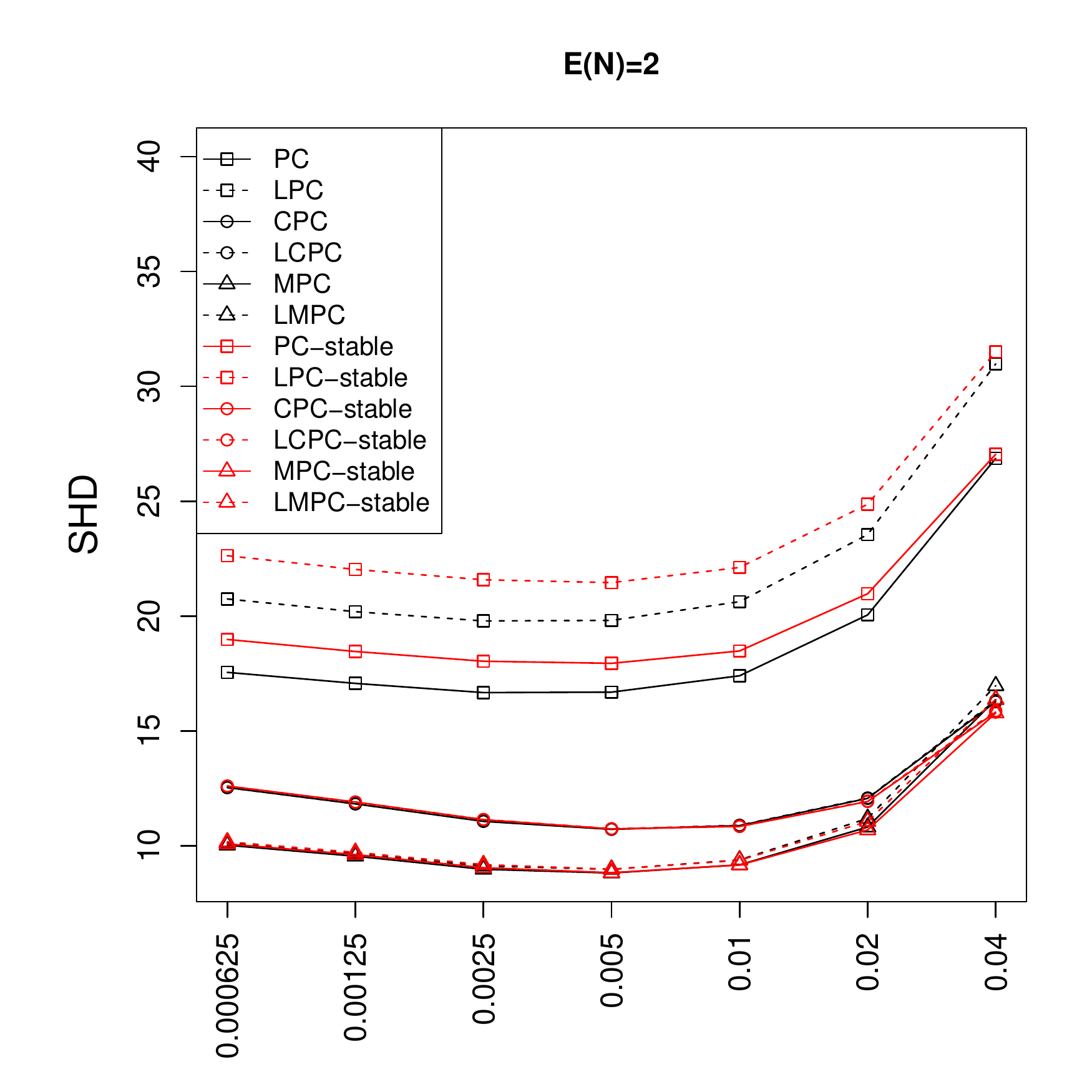}\qquad
     \includegraphics[scale=0.36,angle=0]{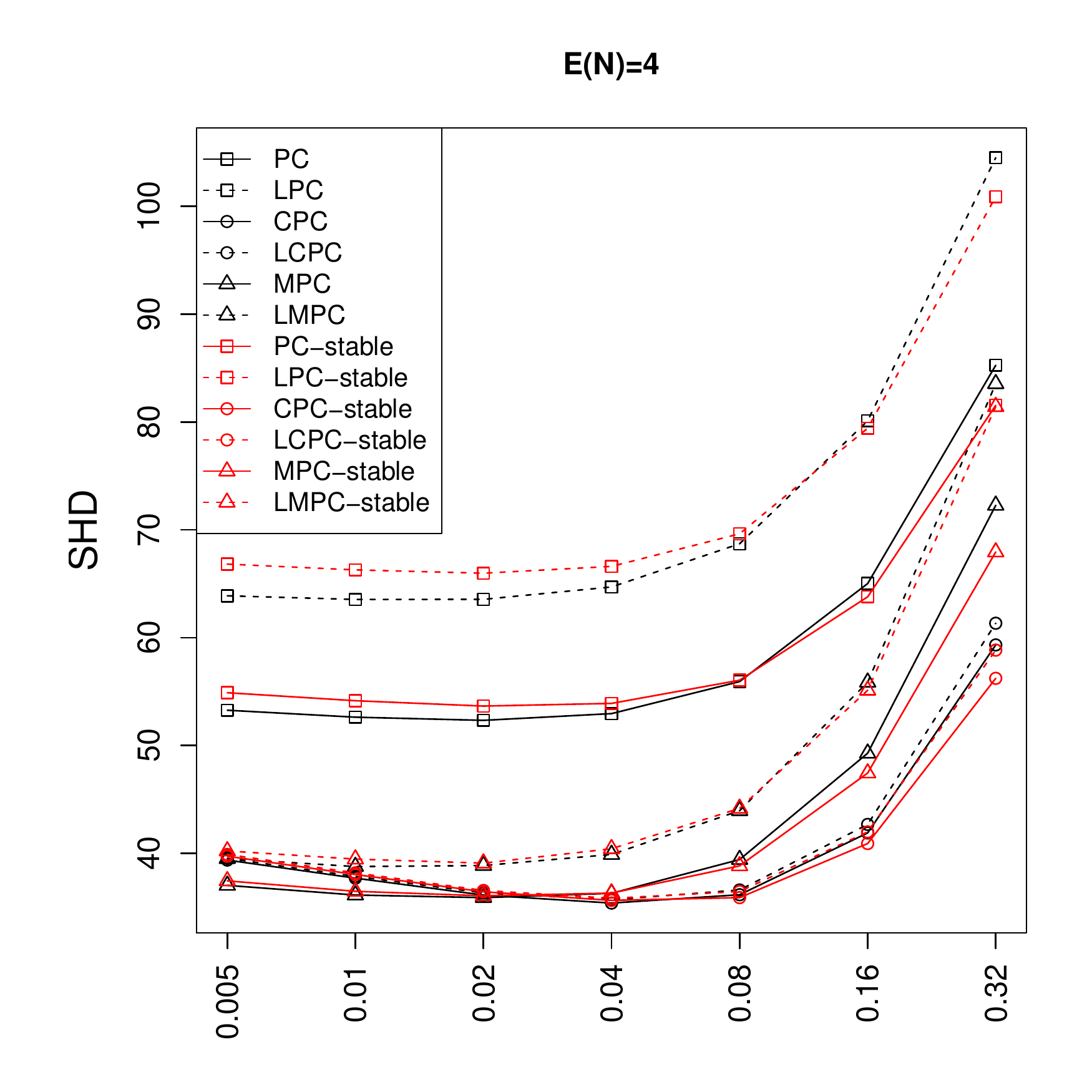}\\
     \includegraphics[scale=0.36,angle=0]{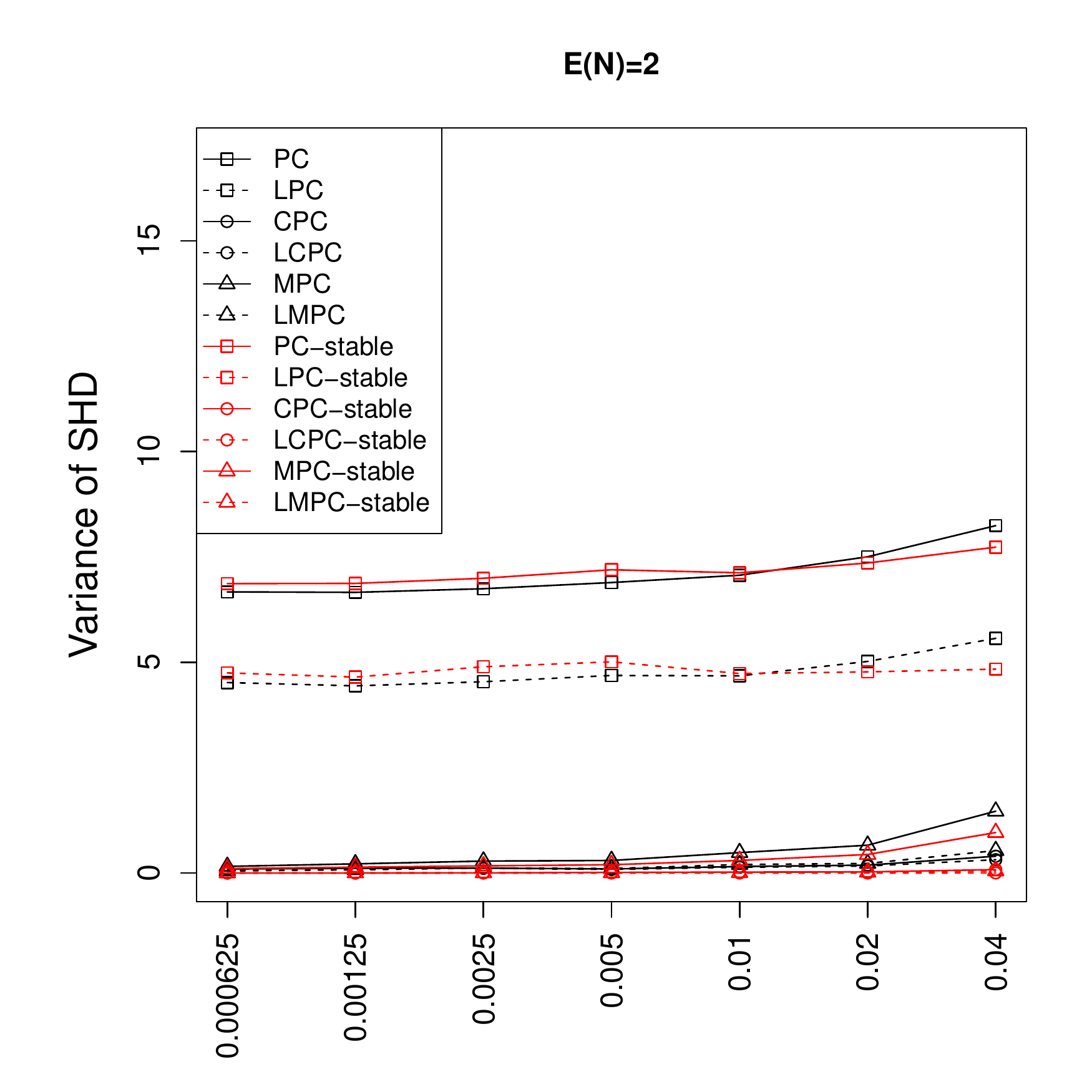}\qquad
     \includegraphics[scale=0.36,angle=0]{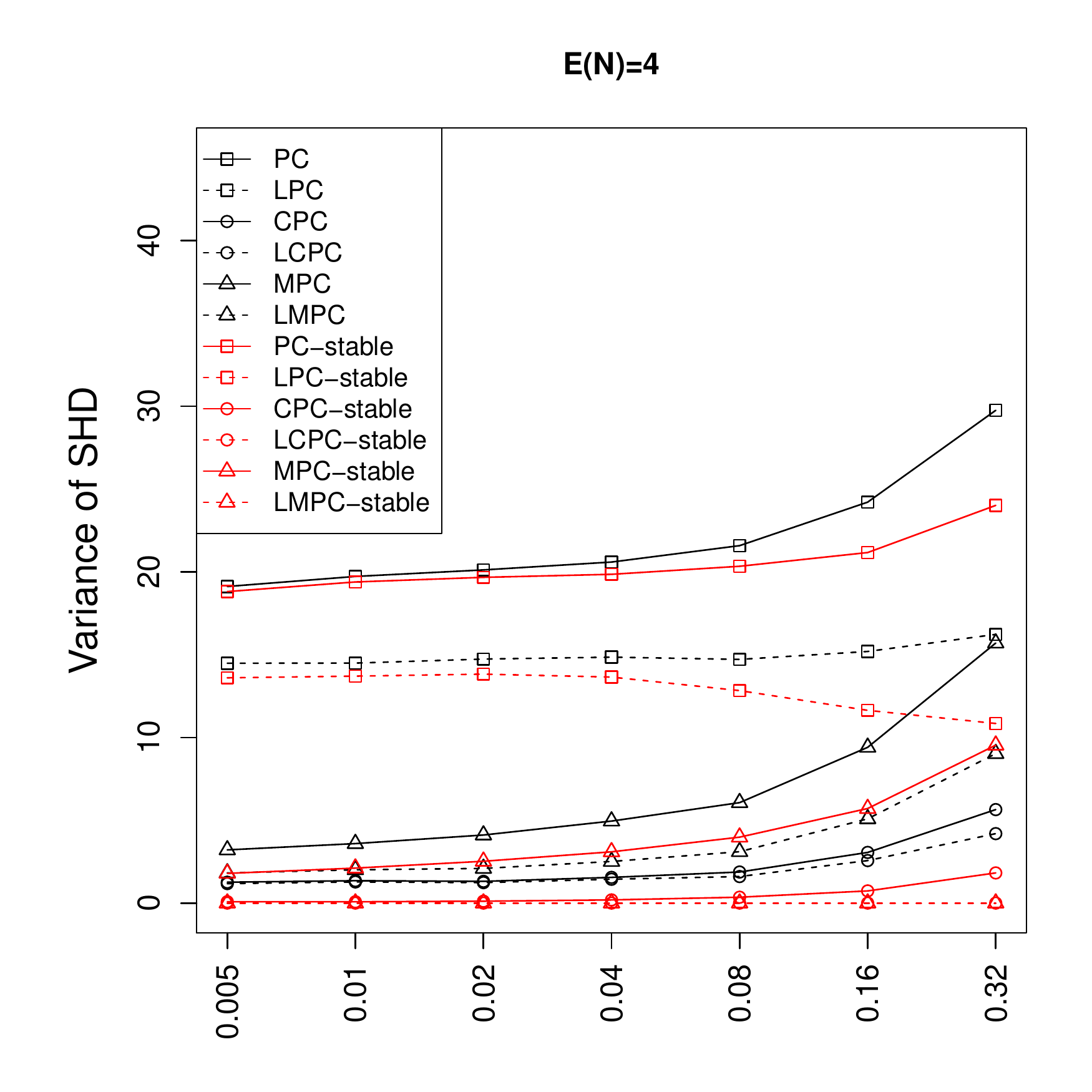}
  \caption{Estimation performance of (L)PC(-stable), (L)CPC(-stable), and
    (L)MPC(-stable) for the CPDAGs in the low-dimensional settings, for
    different values of $\alpha$. The first row of plots shows the
    performance in terms of SHD, shown as averages over 250 randomly
    generated graphs and 50 random variable orderings per graph. The second
    row of plots shows the performance in terms of the variance of the SHD
  over the 50 random variable orderings per graph, shown as averages over
  250 randomly generated graphs.}
  \label{fig.sim.pc.cpdags.add1}
\end{figure}


For the scenario with latent variables, we generated 120 random weighted DAGs with
$p=50$ and $E(N)=2$, as described in Section \ref{sec.sim.setup}.
For each DAG we generated an i.i.d.\ sample size of
$n=1000$. To assess the impact of latent variables,
we randomly defined in each DAG half of the
variables that have no parents and at least two children to be latent. We
then estimated each graph for 20 random orderings of the observed
variables, using the sample versions of FCI(-stable), CFCI(-stable),
MFCI(-stable), RFCI(-stable), CRFCI(-stable), and MRFCI(-stable) at levels
$\alpha \in \{0.0025, 0.005, 0.01, 0.02, 0.04, 0.08\}$ for the partial
correlation tests. Thus, for each randomly generated graph, we obtained 20
estimated PAGs from each algorithm, for each value of $\alpha$. Figure
\ref{fig.sim.fci.skelet.add1} shows the estimation performance of FCI
(circles; black dashed line), FCI-stable (triangles; red dashed line), RFCI
(circles; black solid line), and RFCI-stable (triangles; red solid line)
for the skeleton. Figure \ref{fig.sim.fci.pags.add1} shows the
estimation performance for the PAGs in terms of SHD edge marks, and in
terms of the variance of the SHD edge marks over the
20 random variable orderings per graph.

\begin{figure}[h]
\centering
\includegraphics[scale=0.80]{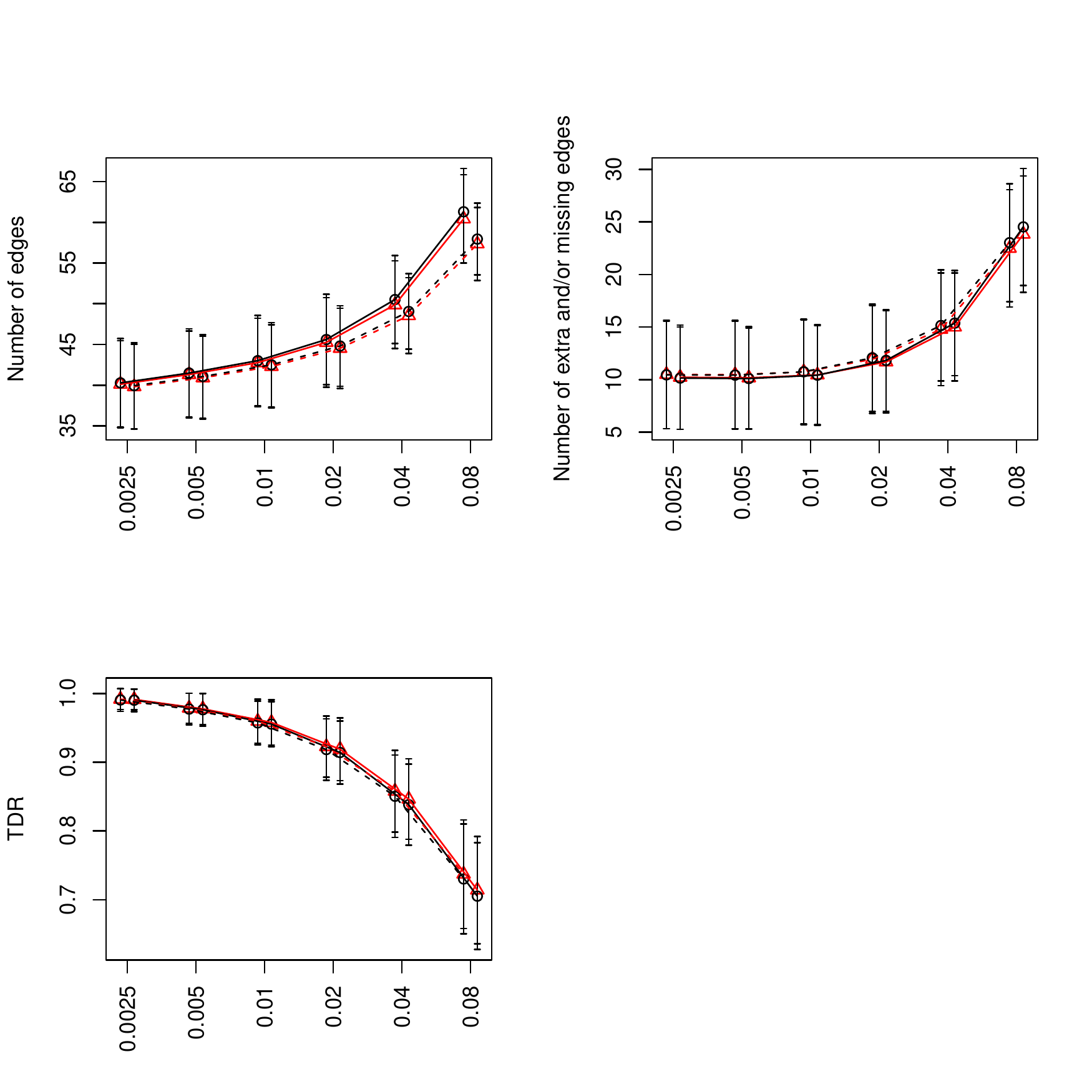}
\caption{Estimation performance of FCI (circles; black dashed line),
  FCI-stable (triangles; red dashed line), RFCI (circles; black solid
  line), and RFCI-stable (triangles; red solid line), for the skeleton of
  the PAGs for different values of $\alpha$ ($x$-axis displayed in $\log$
  scale) in the low-dimensional setting. The results are shown as averages
  plus or minus one standard deviation, computed over 120 randomly
  generated graphs and 20 random variable orderings per graph, and slightly
  shifted up and down from the real values of $\alpha$ for a better
  visualization.}
  \label{fig.sim.fci.skelet.add1}
\end{figure}

\begin{figure}[!h]\centering%
     \includegraphics[scale=0.36,angle=0]{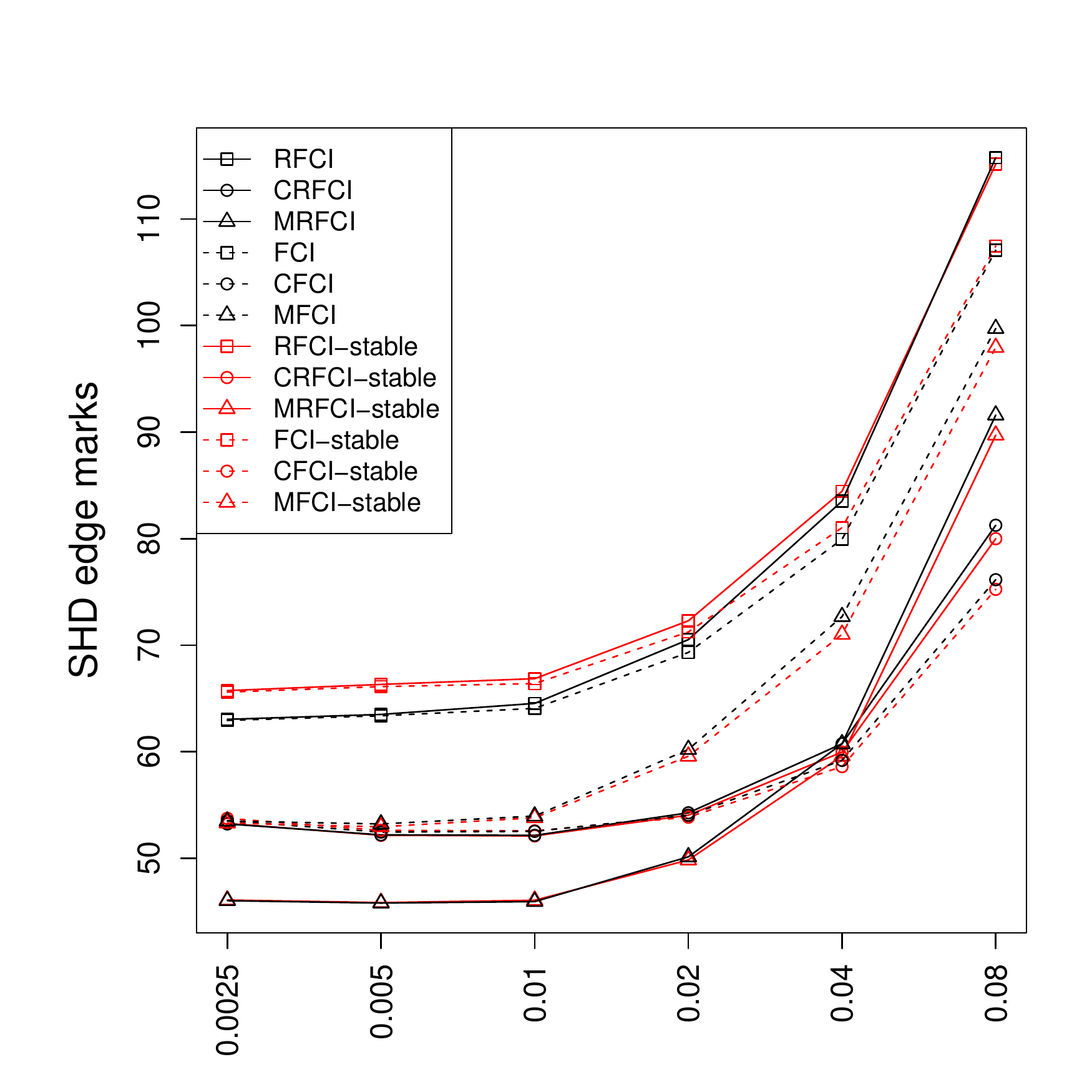}\qquad
     \includegraphics[scale=0.36,angle=0]{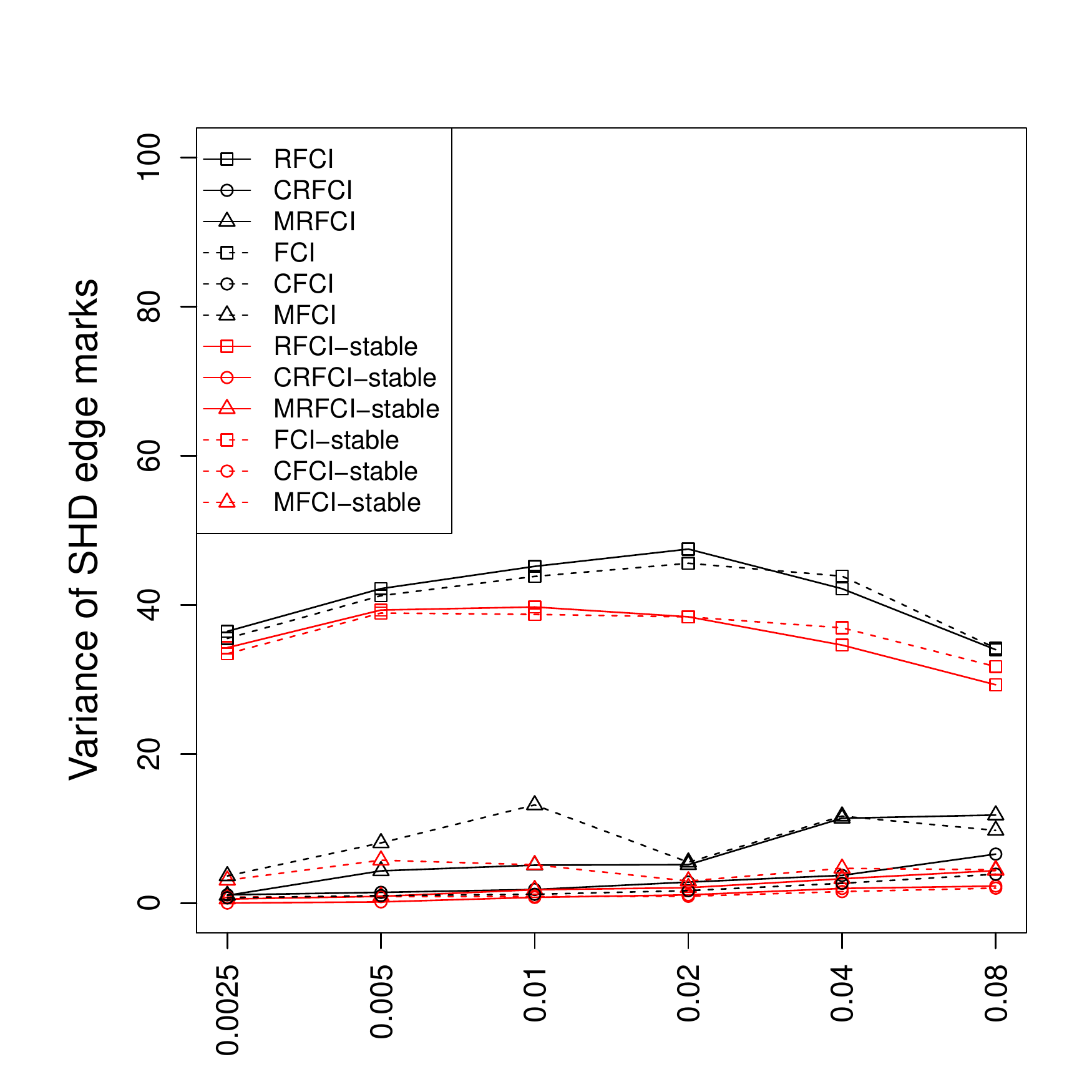}
  \caption{Estimation performance of the modifications of FCI(-stable) and
    RFCI(-stable) for the PAGs in the low-dimensional setting, for
    different values of $\alpha$. The left panel shows the performance in
    terms of SHD edge marks, shown as averages over 120 randomly generated
    graphs and 20 random variable orderings per graph. The right panel
    shows the performance in terms of the variance of the SHD
  edge marks over the 20 random variable orderings per graph, shown as
  averages over 120 randomly generated graphs.}
  \label{fig.sim.fci.pags.add1}
\end{figure}

Regarding the skeletons of the CPDAGs and PAGs, the estimation performances
between PC and PC-stable, as well as between (R)FCI and (R)FCI-stable are
basically indistinguishable for all values of $\alpha$. However, Figure
\ref{fig.sim.fci.skelet.add1} shows that FCI(-stable) returns graphs with
slightly fewer edges than RFCI(-stable), for all values of $\alpha$. This
is related to the fact that FCI(-stable) tends to perform more tests than
RFCI(-stable).

Regarding the CPDAGs and PAGs, the performance of the modifications of PC
and (R)FCI (black lines) are very close to the performance of PC-stable and
(R)FCI-stable (red lines). Moreover, CPC(-stable) and MPC(-stable) as well
as C(R)FCI(-stable) and M(R)FCI(-stable) perform better in particular in
reducing the variance of the SHD and SHD edge marks, respectively. This
indicates that most of the order-dependence in the low-dimensional setting
is in the orientation of the edges.

We also note that in all proposed measures there are only small differences
between modifications of FCI and of RFCI.

\subsection{Number of tests and computing time}\label{app.sim.time.pc}

We consider the number of tests and the computing time of PC and PC-stable
in the high-dimensional setting described in Section \ref{sec.sim.setup}.

One can easily deduce that Step 1 of the PC- and PC-stable algorithms
perform the same number of tests for $\ell=0$, because the adjacency sets
do not play a role at this stage. Moreover, for $\ell=1$ PC-stable performs
at least as many tests as PC, since the
adjacency sets $a(X_i)$ (see Algorithm \ref{pseudo.new.pc}) are
always supersets of the adjacency sets $\text{adj}(X_i)$ (see Algorithm
\ref{pseudo.old.pc}). For larger values of $\ell$, however, it is
difficult to analyze the number of tests analytically.

Table \ref{table.nrtests} therefore shows the average number of tests that
were performed by Step 1 of the two algorithms, separated by size of the
conditioning set, where we considered the high-dimensional setting with
$\alpha=0.04$ (see Section \ref{sec.sim.skeleton}) since this was most
computationally intensive. As expected the number of marginal correlation
tests was identical for both algorithms. For $\ell=1$,
PC-stable performed slightly more than twice as many tests as PC, amounting
to about $1.36 \times 10^5$ additional tests. For $\ell=2$, PC-stable
performed more tests than PC, amounting to $3.4 \times 10^3$. For
larger values of $\ell$, PC-stable performed fewer tests than PC, since the
additional tests for $\ell=1$ and $\ell=2$ lead to a sparser
skeleton. However, since PC also performed relatively few tests for larger
values of $\ell$, the absolute difference in the number of tests for large
$\ell$ is rather small. In total, PC-stable performed about $1.39 \times
10^5$ more tests than PC.

\begin{table}[h!]
\centering
\begin{tabular}{crr}
  \cline{2-3}
  & \multicolumn{1}{|c|}{PC-algorithm} &
   \multicolumn{1}{|c|}{PC-stable algorithm}\\ \cline{1-3}
  \multicolumn{1}{|c|}{$\ell=0$} & \multicolumn{1}{|r|}{$5.21 \times 10^5$
  ($1.95 \times 10^2$)} & \multicolumn{1}{|r|}{$5.21 \times 10^5$ ($1.95
  \times 10^2$)}\\ \cline{1-3}
  \multicolumn{1}{|c|}{$\ell=1$} & \multicolumn{1}{|r|}{$1.29 \times 10^5$
    ($2.19 \times 10^3$)} & \multicolumn{1}{|r|}{$2.65 \times 10^5$ ($4.68
    \times 10^3$)}\\
  \cline{1-3}
  \multicolumn{1}{|c|}{$\ell=2$} & \multicolumn{1}{|r|}{$1.10 \times 10^4$
    ($5.93 \times 10^2$)} & \multicolumn{1}{|r|}{$1.44 \times 10^4$
    ($8.90 \times 10^2$)}\\ \cline{1-3}
  \multicolumn{1}{|c|}{$\ell=3$} & \multicolumn{1}{|r|}{$1.12 \times 10^3$
    ($1.21 \times 10^2$)} & \multicolumn{1}{|r|}{$5.05 \times 10^2$
    ($8.54 \times 10^1$)}\\ \cline{1-3}
  \multicolumn{1}{|c|}{$\ell=4$} & \multicolumn{1}{|r|}{$9.38 \times 10^1$
    ($2.86 \times 10^1$)} & \multicolumn{1}{|r|}{$3.08 \times 10^1$ ($1.78
    \times 10^1$)}\\ \cline{1-3}
  \multicolumn{1}{|c|}{$\ell=5$} &  \multicolumn{1}{|r|}{$2.78 \times 10^0$
    ($4.53 \times 10^0$)} &
  \multicolumn{1}{|r|}{$0.65 \times 10^0$ ($1.94 \times 10^0$)}\\ \cline{1-3}
  \multicolumn{1}{|c|}{$\ell=6$} & \multicolumn{1}{|r|}{$0.02 \times 10^0$
    ($0.38 \times 10^0$)} &
  \multicolumn{1}{|c|}{-}\\ \hline \hline
  \multicolumn{1}{|c|}{Total} & \multicolumn{1}{|r|}{$6.62 \times 10^5$
  ($2.78 \times 10^3$)} & \multicolumn{1}{|r|}{$8.01 \times 10^5$ ($5.47
  \times 10^3$)}\\ \cline{1-3}
\end{tabular}
\caption{Number of tests performed by Step 1 of the PC and PC-stable algorithms
  for each size of the conditioning sets $\ell$, in the
  high-dimensional setting with $p=1000$, $n=50$ and $\alpha=0.04$. The
  results are shown as averages (standard deviations) over 250 random
  graphs and 20 random variable orderings per graph.}
  \label{table.nrtests}
\end{table}

Table \ref{table.runtime} shows the average runtime of the PC- and PC-stable
algorithms. We see that PC-stable is somewhat slower than PC for all values
of $\alpha$, which can be explained by the fact that PC-stable tends to
perform a larger number of tests (cf. Table \ref{table.nrtests}).

\begin{table}[h!]
\centering
\begin{tabular}{lrrr}
  \cline{2-4}
  & \multicolumn{1}{|c|}{PC-algorithm} &
  \multicolumn{1}{|c|}{PC-stable algorithm} &
  \multicolumn{1}{|c|}{PC / PC-stable}\\ \cline{1-4}
  \multicolumn{1}{|l|}{$\alpha=0.000625$} & \multicolumn{1}{|r|}{111.79 (7.54)} &
  \multicolumn{1}{|r|}{115.46 (7.47)}  & \multicolumn{1}{|c|}{0.97}\\ \cline{1-4}
  \multicolumn{1}{|l|}{$\alpha=0.00125$} & \multicolumn{1}{|r|}{110.13 (6.91)} & \multicolumn{1}{|r|}{113.77 (7.07)}& \multicolumn{1}{|c|}{0.97}\\ \cline{1-4}
  \multicolumn{1}{|l|}{$\alpha=0.025$} & \multicolumn{1}{|r|}{115.90 (12.18)} & \multicolumn{1}{|r|}{119.67 (12.03)}& \multicolumn{1}{|c|}{0.97}\\
  \cline{1-4}
  \multicolumn{1}{|l|}{$\alpha=0.05$} & \multicolumn{1}{|r|}{116.14 (9.50)}
  & \multicolumn{1}{|r|}{119.91 (9.57)}& \multicolumn{1}{|c|}{0.97}\\
  \cline{1-4}
  \multicolumn{1}{|l|}{$\alpha=0.01$} & \multicolumn{1}{|r|}{121.02
    (8.61)} & \multicolumn{1}{|r|}{125.81 (8.94)}& \multicolumn{1}{|c|}{0.96}\\ \cline{1-4}
  \multicolumn{1}{|l|}{$\alpha=0.02$} & \multicolumn{1}{|r|}{131.42 (13.98)} & \multicolumn{1}{|r|}{139.54 (14.72)} & \multicolumn{1}{|c|}{0.94} \\ \cline{1-4}
  \multicolumn{1}{|l|}{$\alpha=0.04$} &  \multicolumn{1}{|r|}{148.72 (14.98)} &
  \multicolumn{1}{|r|}{170.49 (16.31)} & \multicolumn{1}{|c|}{0.87} \\ \cline{1-4}
\end{tabular}
\caption{Run time in seconds (computed on an AMD Opteron(tm) Processor 6174
  using R 2.15.1.) of PC and PC-stable for the high-dimensional setting
  with $p=1000$ and $n=50$. The results are shown as averages (standard
  deviations) over 250 random graphs and 20 random variable orderings per graph.}
  \label{table.runtime}
\end{table}

\subsection{Estimation performance in settings where $p=n$}\label{app.sim.results.equal}

Finally, we consider two settings for the scenario with latent
variables, where we generated 250 random weighted DAGs
with $p=50$ and $E(N)=\{2,4\}$, as described in Section \ref{sec.sim.setup}.
 For each DAG we generated an i.i.d.\
sample size of $n=50$. We again randomly defined in each DAG half of the
variables that have no parents and at least two children to be latent. We
then estimated each graph for 50 random orderings of the observed
variables, using the sample versions of FCI(-stable), CFCI(-stable),
MFCI(-stable), RFCI(-stable), CRFCI(-stable), and MRFCI(-stable) at levels
$\alpha \in \{0.0025, 0.005, 0.01, 0.02, 0.04, 0.08, 0.16\}$ for $E(N)=2$
and $\alpha \in \{0.005, 0.01, 0.02, 0.04, 0.08, 0.16, 0.32\}$ for
$E(N)=4$. Thus, for each randomly generated graph, we obtained 50 estimated
PAGs from each algorithm, for each value of $\alpha$.

Figure \ref{fig.sim.fci.skelet.add2} shows the estimation performance for the
skeleton. The (R)FCI-stable versions (red lines)
lead to slightly sparser graphs and slightly better performance in TDR than
(R)FCI versions (black lines) in both settings.

\begin{figure}[h]
\centering
\includegraphics[scale=0.80]{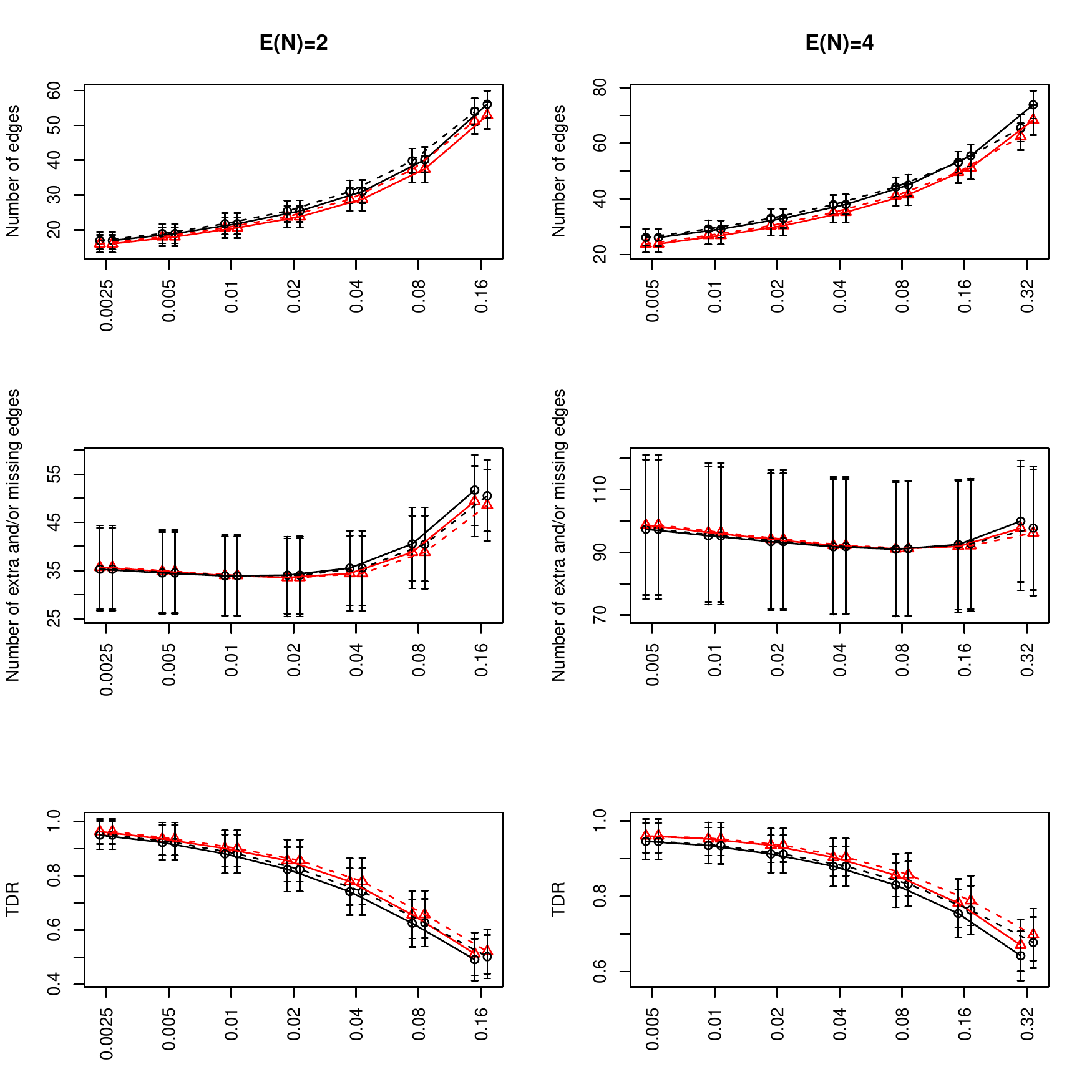}
\caption{Estimation performance of FCI (circles; black dashed line),
  FCI-stable (triangles; red dashed line), RFCI (circles; black solid
  line), and RFCI-stable (triangles; red solid line), for the skeleton of
  the PAGs for different values of $\alpha$ ($x$-axis displayed in $\log$
  scale) in two settings where $p=n$. The results are shown as averages
  plus or minus one standard deviation, computed over 250 randomly
  generated graphs and 50 random variable orderings per graph, and slightly
  shifted up and down from the real values of $\alpha$ for a better
  visualization.}
  \label{fig.sim.fci.skelet.add2}
\end{figure}

Figures \ref{fig.sim.fci.pags.add2} shows
the estimation performance of all modifications of (R)FCI with
respect to the PAGs in terms of SHD edge marks, and
in terms of the variance of the SHD edge marks over the 50 random variable
orderings per graph. The (R)FCI-stable versions produce a better fit than
the (R)FCI versions. Moreover, C(R)FCI(-stable) and M(R)FCI(-stable)
perform similarly for sparse graphs and they improve the fit, while in
denser graphs M(R)FCI(-stable) still improves the fit and it performs much
better than C(R)FCI(-stable) for the SHD edge marks. Again we see little
difference between modifications of RFCI and FCI with respect to all measures.


\begin{figure}[!h]\centering%
     \includegraphics[scale=0.36,angle=0]{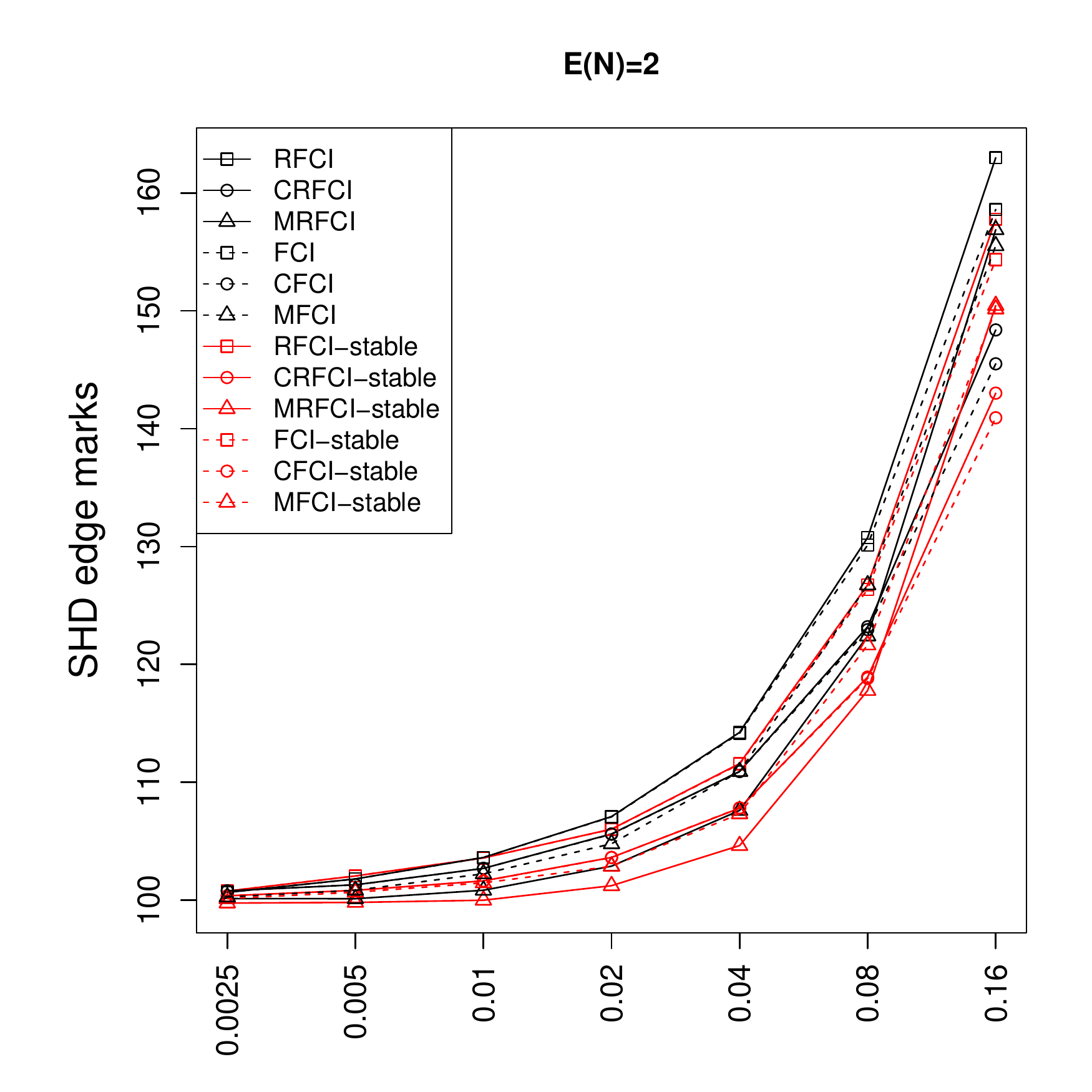}\qquad
     \includegraphics[scale=0.36,angle=0]{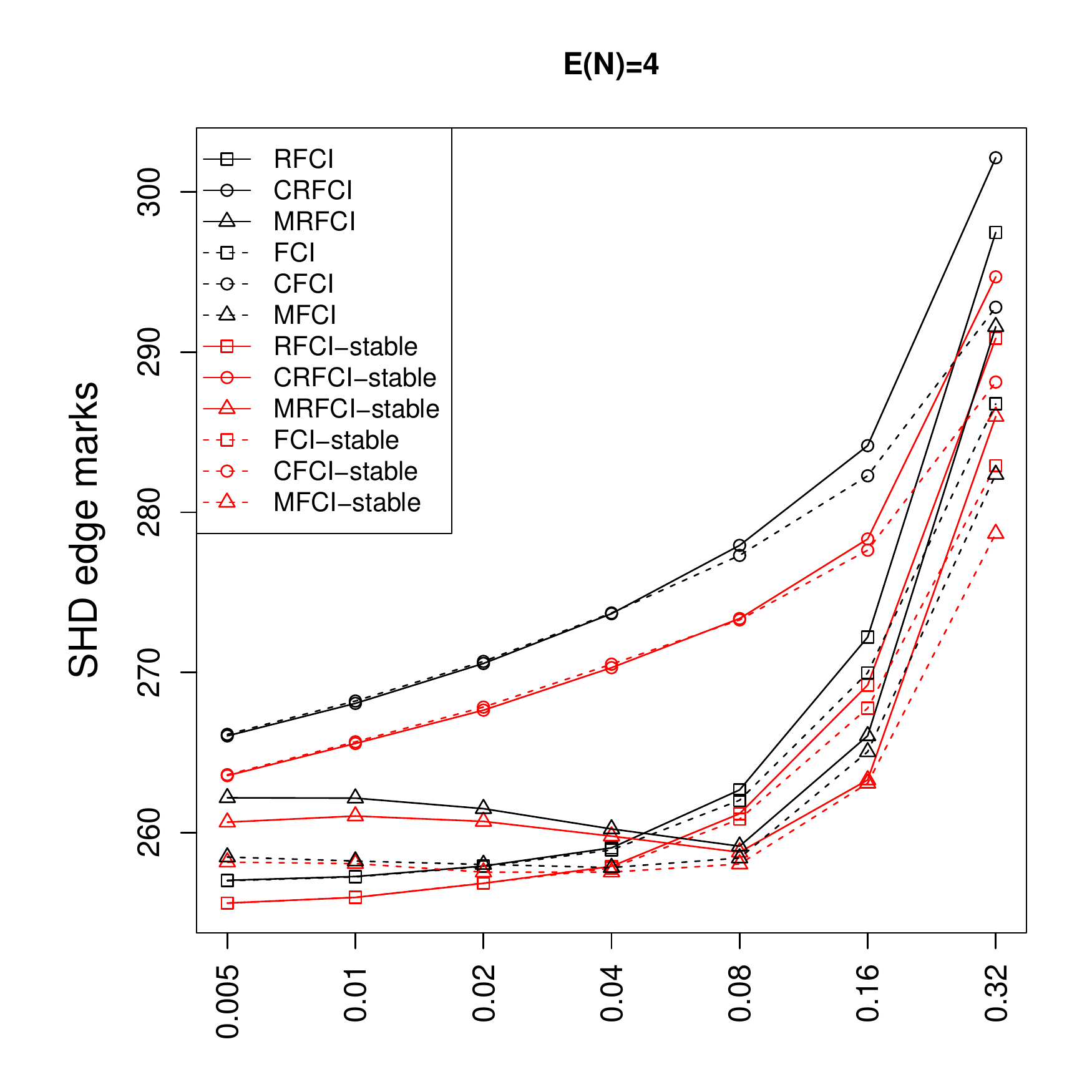}\\
     \includegraphics[scale=0.36,angle=0]{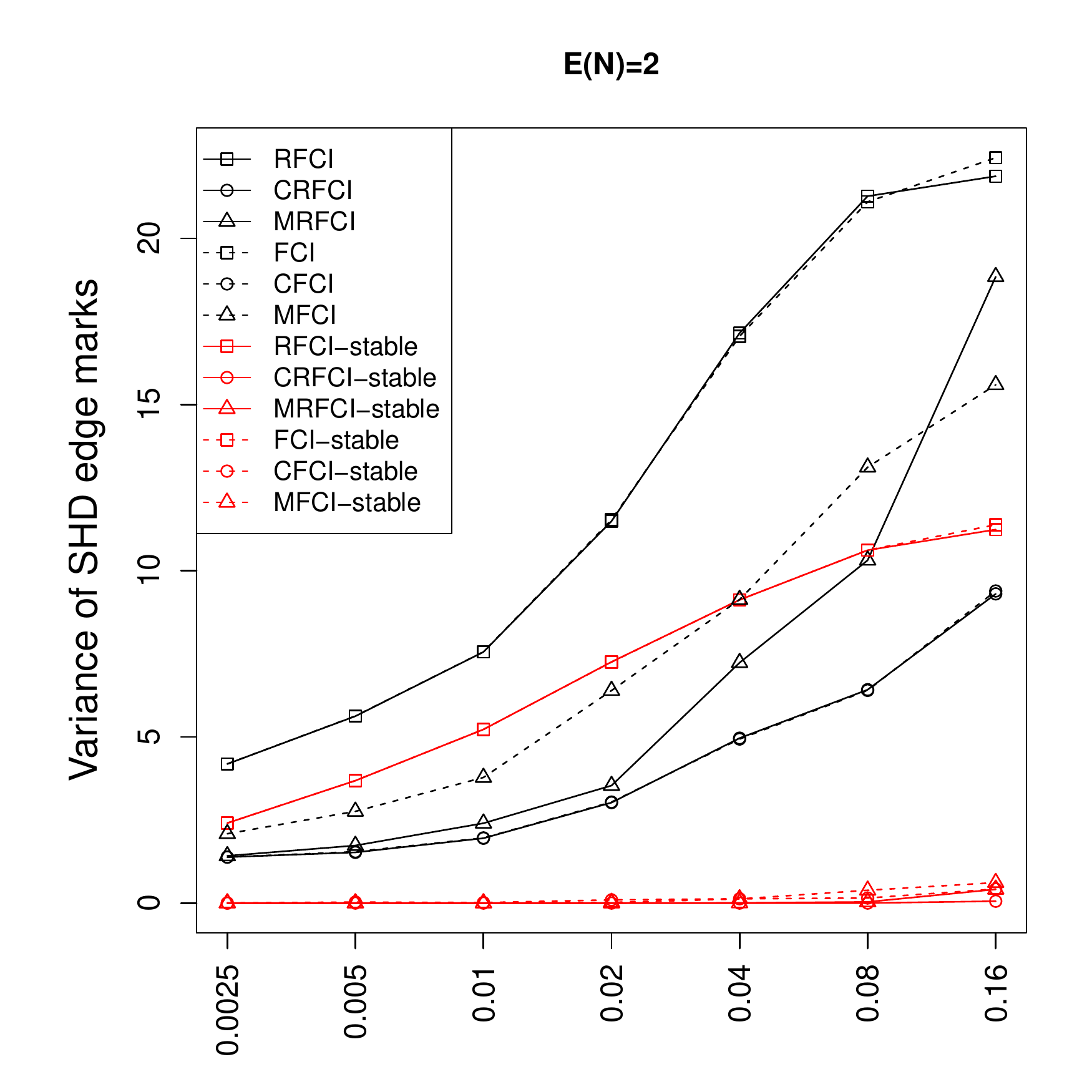}\qquad
\includegraphics[scale=0.36,angle=0]{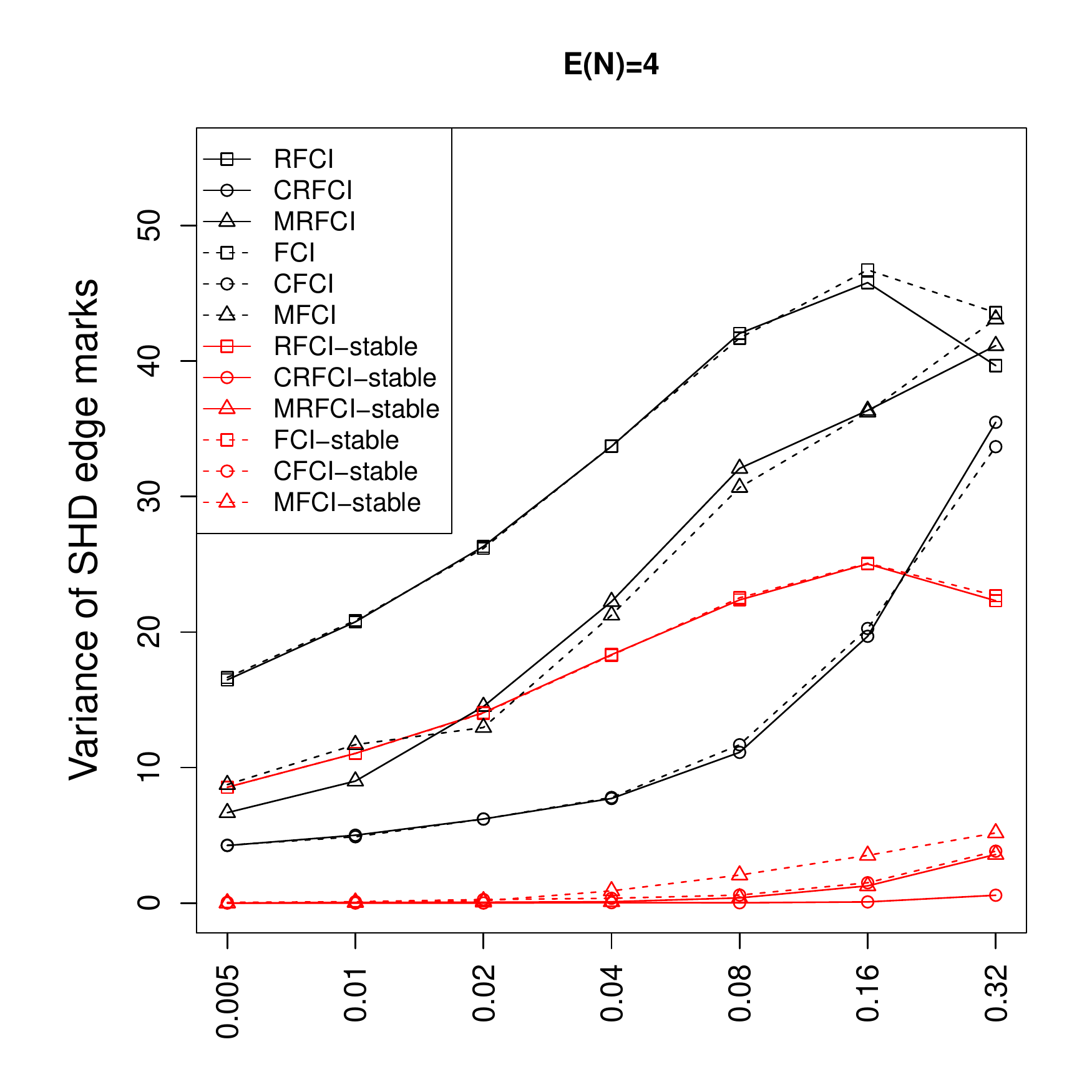}

  \caption{Estimation performance of the modifications of FCI(-stable) and
    RFCI(-stable) for the PAGs in settings where $p=n$, for different values of
  $\alpha$. The first row of plots shows the performance in terms of SHD
  edge marks, shown as averages over 250 randomly generated graphs and 50 random
  variable orderings per graph. The second row of plots shows the
  performance in terms of the variance of the SHD edge marks over the 50
  random variable orderings per graph, shown as averages over 250 randomly
  generated graphs.}
  \label{fig.sim.fci.pags.add2}
\end{figure}

\clearpage
\newpage

\bibliography{Mybib}

\begin{thebibliography}{31}
\providecommand{\natexlab}[1]{#1}
\providecommand{\url}[1]{\texttt{#1}}
\expandafter\ifx\csname urlstyle\endcsname\relax
  \providecommand{\doi}[1]{doi: #1}\else
  \providecommand{\doi}{doi: \begingroup \urlstyle{rm}\Url}\fi

\bibitem[Andersson et~al.(1997)Andersson, Madigan, and
  Perlman]{AndersonEtAll97}
S.~A. Andersson, D.~Madigan, and M.~D. Perlman.
\newblock A characterization of {M}arkov equivalence classes for acyclic
  digraphs.
\newblock \emph{Ann. Statist.}, 25\penalty0 (2):\penalty0 505--541, 1997.

\bibitem[Cano et~al.(2008)Cano, G\'omez-Olmedo, and Moral]{CanoEtAl08}
A.~Cano, M.~G\'omez-Olmedo, and S.~Moral.
\newblock A score based ranking of the edges for the {PC} algorithm.
\newblock In \emph{Proceedings of the Fourth European Workshop on Probabilistic
  Graphical Models}, pages 41--48, 2008.

\bibitem[Chickering(2002)]{Chickering02}
D.M. Chickering.
\newblock Learning equivalence classes of {B}ayesian-network structures.
\newblock \emph{J. Mach. Learn. Res.}, 2:\penalty0 445--498, 2002.

\bibitem[Claassen et~al.(2013)Claassen, Mooij, and Heskes]{ClaassenEtAl13}
Tom Claassen, Joris Mooij, and Tom Heskes.
\newblock Learning sparse causal models is not {NP}-hard.
\newblock In \emph{Proceedings of the 29th Conference on Uncertainty in
  Artificial Intelligence}, 2013.
\newblock To appear.

\bibitem[Colombo et~al.(2012)Colombo, Maathuis, Kalisch, and
  Richardson]{CoMaKaRi2012}
D.~Colombo, M.H. Maathuis, M.~Kalisch, and T.S. Richardson.
\newblock Learning high-dimensional directed acyclic graphs with latent and
  selection variables.
\newblock \emph{Ann. Statist.}, 40\penalty0 (1):\penalty0 294--321, 2012.
\newblock ISSN 0090-5364.

\bibitem[Dash and Druzdzel(1999)]{DashDruzdzel99}
D.~Dash and M.J. Druzdzel.
\newblock A hybrid anytime algorithm for the construction of causal models from
  sparse data.
\newblock In \emph{Proceedings of the Fifteenth Conference on Uncertainty on
  Artificial Intelligence (UAI-99)}, pages 142--149. Morgan Kaufmann
  Publishers, Inc., 1999.

\bibitem[Dawid(1980)]{Dawid80}
A.~P. Dawid.
\newblock Conditional independence for statistical operations.
\newblock \emph{Ann. Statist.}, 8:\penalty0 598--617, 1980.

\bibitem[Harris and Drton(2012)]{HarrisDrton12}
N.~Harris and M.~Drton.
\newblock {PC} algorithm for gaussian copula graphical models.
\newblock \emph{arXiv preprint arXiv:1207.0242}, 2012.

\bibitem[Hughes et~al.(2000)Hughes, Marton, Jones, Roberts, Stoughton, Armour,
  Bennett, Coffey, Dai, He, and et~al]{Hughes00}
T.R. Hughes, M.J. Marton, A.R. Jones, C.J. Roberts, R.~Stoughton, C.D. Armour,
  H.A. Bennett, E.~Coffey, H.~Dai, Y.D. He, and et~al.
\newblock Functional discovery via a compendium of expression profiles.
\newblock \emph{Cell}, 102(1):\penalty0 109--126, 2000.

\bibitem[Kalisch and B\"uhlmann(2007)]{KalischBuehlmann07a}
M.~Kalisch and P.~B\"uhlmann.
\newblock Estimating high-dimensional directed acyclic graphs with the
  {PC}-algorithm.
\newblock \emph{J. Mach. Learn. Res.}, 8:\penalty0 613--636, 2007.

\bibitem[Kalisch et~al.(2010)Kalisch, Fellinghauer, Grill, Maathuis, Mansmann,
  B{\"u}hlmann, and Stucki]{KalischEtAl10}
M.~Kalisch, B.A.G. Fellinghauer, E.~Grill, M.H. Maathuis, U.~Mansmann,
  P.~B{\"u}hlmann, and G.~Stucki.
\newblock Understanding human functioning using graphical models.
\newblock \emph{BMC Medical Research Methodology}, 10\penalty0 (1):\penalty0
  14, 2010.

\bibitem[Kalisch et~al.(2012)Kalisch, M\"achler, Colombo, Maathuis, and
  B\"uhlmann]{KalischEtAl12}
M.~Kalisch, M.~M\"achler, D.~Colombo, M.H. Maathuis, and P.~B\"uhlmann.
\newblock Causal inference using graphical models with the {R} package pcalg.
\newblock \emph{Journal of Statistical Software}, 47\penalty0 (11):\penalty0
  1--26, 2012.

\bibitem[Maathuis et~al.(2009)Maathuis, Kalisch, and
  B\"uhlmann]{MaathuisKalischBuehlmann09}
M.H. Maathuis, M.~Kalisch, and P.~B\"uhlmann.
\newblock Estimating high-dimensional intervention effects from observational
  data.
\newblock \emph{Ann. Statist.}, 37\penalty0 (6A):\penalty0 3133--3164, 2009.

\bibitem[Maathuis et~al.(2010)Maathuis, Colombo, Kalisch, and
  B\"uhlmann]{MaathuisColomboKalischBuhlmann10}
M.H. Maathuis, D.~Colombo, M.~Kalisch, and P.~B\"uhlmann.
\newblock Predicting causal effects in large-scale systems from observational
  data.
\newblock \emph{Nature Methods}, 7\penalty0 (4):\penalty0 247--248, 2010.

\bibitem[Meek(1995)]{Meek95}
C.~Meek.
\newblock Causal inference and causal explanation with background knowledge.
\newblock In \emph{Proceedings of the Eleventh Conference on Uncertainty in
  Artificial Intelligence (UAI-95)}, pages 403--411. Morgan Kaufmann
  Publishers, Inc., 1995.

\bibitem[Meinshausen and B{\"u}hlmann(2010)]{MeinshausenBuehlmann10}
N.~Meinshausen and P.~B{\"u}hlmann.
\newblock Stability selection.
\newblock \emph{Journal of the Royal Statistical Society: Series B (Statistical
  Methodology)}, 72\penalty0 (4):\penalty0 417--473, 2010.

\bibitem[Nagarajan et~al.(2010)Nagarajan, Datta, Scutari, Beggs, Nolen, and
  Peterson]{NagarajanEtAl10}
R.~Nagarajan, S.~Datta, M.~Scutari, M.~Beggs, G.~Nolen, and C.~Peterson.
\newblock Functional relationships between genes associated with
  differentiation potential of aged myogenic progenitors.
\newblock \emph{Frontiers in Physiology}, 1\penalty0 (160), 2010.

\bibitem[Pearl(2000)]{Pearl00}
J.~Pearl.
\newblock \emph{Causality. Models, reasoning, and inference}.
\newblock Cambridge University Press, Cambridge, 2000.

\bibitem[Pearl(2009)]{Pearl09}
J.~Pearl.
\newblock Causal inference in statistics: An overview.
\newblock \emph{Statistics Surveys}, 3:\penalty0 96--146, 2009.

\bibitem[Ramsey et~al.(2006)Ramsey, Zhang, and Spirtes]{RamseyZhangSpirtes06}
Joseph Ramsey, Jiji Zhang, and Peter Spirtes.
\newblock Adjacency-faithfulness and conservative causal inference.
\newblock In \emph{Proceedings of the 22nd Annual Conference on Uncertainty in
  Artificial Intelligence}, Arlington, VA, 2006. AUAI Press.

\bibitem[Richardson(1996)]{Richardson96}
Thomas~S. Richardson.
\newblock A discovery algorithm for directed cyclic graphs.
\newblock In \emph{Proceedings of the Twelfth international conference on
  Uncertainty in artificial intelligence}, pages 454--461. Morgan Kaufmann
  Publishers Inc., 1996.

\bibitem[Singh and Valtorta(1993)]{SinghValtorta93}
M.~Singh and M.~Valtorta.
\newblock An algorithm for the construction of bayesian network structures from
  data.
\newblock In \emph{Proceedings of the Ninth International Conference on
  Uncertainty in Artificial Intelligence (UAI-93)}, pages 259--265. Morgan
  Kaufmann Publishers, Inc., 1993.

\bibitem[Spirtes and Meek(1995)]{SpirtesMeek95}
P.~Spirtes and C.~Meek.
\newblock Learning bayesian networks with discrete variables from data.
\newblock In \emph{Proceeding of the First International Conference on
  Knowledge Discovery and Data Mining}, pages 294--299, Menlo Park, CA: AAAI,
  1995.

\bibitem[Spirtes et~al.(1999)Spirtes, Meek, and
  Richardson]{SpirtesMeekRichardson99}
P.~Spirtes, C.~Meek, and Thomas~S. Richardson.
\newblock An algorithm for causal inference in the presence of latent variables
  and selection bias.
\newblock In \emph{Computation, Causation and Discovery}, pages 211--252. MIT
  Press, 1999.

\bibitem[Spirtes et~al.(2000)Spirtes, Glymour, and Scheines]{SpirtesEtAl00}
P.~Spirtes, C.~Glymour, and R.~Scheines.
\newblock \emph{Causation, Prediction, and Search}.
\newblock MIT Press, Cambridge, second edition, 2000.

\bibitem[Spirtes(2001)]{Spirtes01-anytime}
Peter Spirtes.
\newblock An anytime algorithm for causal inference.
\newblock In \emph{Proc. of the Eighth International Workshop on Artificial
  Intelligence and Statistics}, pages 213--221, San Francisco, 2001. Morgan
  Kaufmann.

\bibitem[Spirtes et~al.(1993)Spirtes, Glymour, and Scheines]{SpirtesEtAl93}
Peter Spirtes, Clark Glymour, and Richard Scheines.
\newblock \emph{Causation, prediction, and search}, volume~81 of \emph{Lecture
  Notes in Statistics}.
\newblock Springer-Verlag, New York, 1993.
\newblock \doi{10.1007/978-1-4612-2748-9}.

\bibitem[Stekhoven et~al.(2012)Stekhoven, Moraes, Sveinbj\"ornsson, Hennig,
  Maathuis, and B\"uhlmann]{StekhovenEtAll12}
D.J. Stekhoven, I.~Moraes, G.~Sveinbj\"ornsson, L.~Hennig, M.H. Maathuis, and
  P.~B\"uhlmann.
\newblock Causal stability ranking.
\newblock \emph{Bioinformatics}, 2012.

\bibitem[van Dijk et~al.(2003)van Dijk, can~der Gaag, and
  Thierens]{vanDijketAl03}
S.~van Dijk, L.C. can~der Gaag, and D.~Thierens.
\newblock A skeleton-based approach to learning bayesian networks from data.
\newblock In \emph{Proceedings of the Seventh Conference on Principles and
  Practice of Knowledge Discovery in Databases (PKDD 2003)}, pages 132--143.
  Springer Verlag, 2003.

\bibitem[Zhang(2008)]{Zhang08-orientation-rules}
Jiji Zhang.
\newblock On the completeness of orientation rules for causal discovery in the
  presence of latent confounders and selection bias.
\newblock \emph{Artificial Intelligence}, 172:\penalty0 1873--1896, 2008.

\bibitem[Zhang et~al.(2011)Zhang, Zhao, He, Lu, Cao, Liu, Hao, Liu, and
  Chen]{ZhangEtAl11}
X.~Zhang, X.M. Zhao, K.~He, L.~Lu, Y.~Cao, J.~Liu, J.K. Hao, Z.P. Liu, and
  L.~Chen.
\newblock Inferring gene regulatory networks from gene expression data by
  pc-algorithm based on conditional mutual information.
\newblock \emph{Bioinformatics}, 2011.

\end{thebibliography}
\end{document}